\newif\ifanon\anonfalse
\PassOptionsToPackage{dvipsnames}{xcolor}

\pdfoutput=1

\documentclass[11pt]{article}
\usepackage[margin=1in]{geometry}

\usepackage{bbm}
\usepackage{nicefrac} 
\usepackage{amsmath,amsfonts}
\usepackage{bm}
\usepackage{hyperref} 
\usepackage[capitalize,nameinlink]{cleveref}

\usepackage[framemethod=TikZ]{mdframed}
\mdfsetup{%
backgroundcolor=white, %
roundcorner=4pt,
linewidth=1pt}

\hypersetup{
      colorlinks = true, 
      pdfpagemode=UseNone,
      urlcolor = {blueGrotto},
      linkcolor = {royalBlue},
      citecolor = {navyBlue},
      pdfstartview=FitW
    }
    \usepackage[
      backref=true,
      backend=biber,
      natbib=true,
      style=alphabetic,
      sorting=alphabeticlabel,
      sortcites=true,
      minbibnames=3,
      maxbibnames=999,
      mincitenames=1,
      maxcitenames=4,
      minalphanames=1,
      maxalphanames=6
    ]{biblatex}
    
    \DeclareSortingTemplate{alphabeticlabel}{
      \sort[final]{%
        \field{labelalpha} %
      }
      \sort{%
        \field{year}   %
      }
      \sort{%
        \field{title}  %
      }
    }
    
    \AtBeginRefsection{\GenRefcontextData{sorting=ynt}}
    \AtEveryCite{\localrefcontext[sorting=ynt]}
\newcommand{\E}{\mathbb{E}}

\newcommand{\R}{\mathbb{R}}

\newcommand{\Law}{\mathrm{Law}}
\newcommand{\Ent}{\mathrm{Ent}}
\newcommand{\Var}{\mathrm{Var}}

\DeclareMathOperator*{\argmin}{arg\,min}

\DeclareMathOperator{\Tr}{Tr}

\renewcommand{\Pr}{\mathbb{P}}

\newcommand{\poly}{\textnormal{poly}}
\newcommand{\polylog}{\textnormal{polylog}}

\newcommand{\ints}{\mathbb Z}

\newcommand{\reals}{\mathbb R}

\newcommand{\eps}{\epsilon}

\newcommand{\calM}{\customcal{M}}

\def\<{\langle}
\def\>{\rangle}

\newcommand{\Z}{\mathbb Z}

\def\wt{\widetilde}
\def\wh{\widehat}

\newcommand{\yesnum}{\addtocounter{equation}{1}\tag{\theequation}} 
\makeatletter
\newcommand{\tagnum}[2]{%
    \refstepcounter{equation}%
    \tag{#1) \ (\theequation}%
    \protected@write \@auxout {}{%
        \string \newlabel {#2}{{\theequation}{\thepage}{}{equation.\theequation}{}}%
    }%
}
\makeatother

\makeatletter
\renewcommand{\eqref}[1]{\textup{\eqrefform@{\ref{#1}}}}
\let\eqrefform@\tagform@

\newcommand{\changetag}[1]{%
  \renewcommand\tagform@[1]{\maketag@@@{(\ignorespaces#1\unskip\@@italiccorr)}}%
}
\makeatother

\newcommand{\Stackrel}[2]{\stackrel{\mathmakebox[\widthof{\ensuremath{#2}}]{#1}}{#2}}

\newcommand{\quadtext}[1]{\quad\text{#1}\quad}
\newcommand{\qquadtext}[1]{\qquad\text{#1}\qquad}
 
\newcommand{\quadand}{\quadtext{and}}
\newcommand{\qquadand}{\qquadtext{and}}

\newcommand{\quadwhere}{\quadtext{where}}

\def\abs#1{\left| #1 \right|}

\newcommand{\binparen}[1]{\big(#1\big)}

\newcommand{\inbrace}[1]{{\left\{#1\right\}}}

\newcommand{\inparen}[1]{{\left(#1\right)}}

\newcommand{\inangle}[1]{\left\langle#1\right\rangle}

\let\norm\relax
\newcommand{\norm}[1]{\ensuremath{\left\lVert #1 \right\rVert}}

\newcommand{\tv}[2]{\operatorname{d}_{\mathsf{TV}}(#1,#2)}

\DeclareMathOperator{\KL}{\mathsf{KL}}
\newcommand{\mmid}{\;\|\;}

\newcommand{\sfrac}[2]{{#1/#2}} 
\newcommand{\nfrac}[2]{\nicefrac{#1}{#2}}

\newcommand{\iid}{i.i.d.}
\newcommand{\unif}{{\rm Unif}}

\renewcommand{\epsilon}{\varepsilon}
\makeatletter
\newcommand*{\tran}{{\mathpalette\@tran{}}}
\newcommand*{\@tran}[2]{\raisebox{\depth}{$\m@th#1\intercal$}}
\makeatother

\mathchardef\NABLA"272
\newcommand*{\Nabla}{\boldsymbol\NABLA}
\let\nabla\Nabla

\renewcommand{\hat}{\widehat}

\renewcommand{\tilde}{\widetilde}

\newcommand{\customcal}[1]{\euscr{#1}} %

\newcommand{\cC}{\customcal{C}}

\newcommand{\cF}{\customcal{F}}

\newcommand{\cM}{\customcal{M}}
\newcommand{\cN}{\customcal{N}}

\newcommand{\cR}{\customcal{R}}

\newcommand{\cX}{\customcal{X}}

\DeclareMathAlphabet{\mathdutchcal}{U}{dutchcal}{m}{n}
\SetMathAlphabet{\mathdutchcal}{bold}{U}{dutchcal}{b}{n}
\DeclareMathAlphabet{\mathdutchbcal}{U}{dutchcal}{b}{n}

\DeclareMathAlphabet\urwscr{U}{urwchancal}{b}{n}%
\DeclareMathAlphabet\rsfscr{U}{rsfso}{m}{n}
\DeclareMathAlphabet\euscr{U}{eus}{m}{n}
\DeclareFontEncoding{LS2}{}{}
\DeclareFontSubstitution{LS2}{stix}{m}{n}
\DeclareMathAlphabet\stixcal{LS2}{stixcal}{m} {n}

\renewcommand{\paragraph}[1]{\medskip \noindent\textbf{#1}~}

\newcommand{\ie}{\textit{i.e.}}
\newcommand{\eg}{\textit{e.g.}}

\newcommand{\hypo}[1]{\mathdutchcal{#1}}

\newcommand{\hyH}{\hypo{H}}

\newcommand{\hyP}{\hypo{P}}

\renewcommand{\d}{{\rm d}}

\newcommand\indep{\mathrel{\text{\scalebox{1.07}{$\perp\mkern-10mu\perp$}}}}

\usepackage{array}
\newcolumntype{L}[1]{>{\raggedright\let\newline\\\arraybackslash\hspace{0pt}}m{#1}}
\newcolumntype{C}[1]{>{\centering\let\newline\\\arraybackslash\hspace{0pt}}m{#1}}
\newcolumntype{R}[1]{>{\raggedleft\let\newline\\\arraybackslash\hspace{0pt}}m{#1}}

\newcommand{\deq}{\coloneqq}
\newcommand{\normal}[2]{\cN(#1,#2)}

\DeclareMathOperator{\dive}{div}

\renewcommand{\d}{\mathrm{d}}

\newcommand{\bbP}{\mathbb{P}}
\newcommand{\DDPM}{\hat\theta_n^{\,\rm DDPM}}
\newcommand{\err}{\mathsf{err}}
\newcommand{\MLE}{\hat\theta_n^{\,\rm MLE}}
\newcommand\simiid{\overset{\text{i.i.d.}}{\sim}}
\newcommand\todist[1]{\xrightarrow[#1]{\mathsf{d}}}

\usepackage{url}
\usepackage{enumerate}
\usepackage{enumitem}
\usepackage{mathtools}
\usepackage{mathrsfs}
\usepackage{mathrsfs}
\usepackage{amssymb,amsthm}
\usepackage{microtype}
\renewcommand{\cref}{\Cref}

\usepackage[dvipsnames]{xcolor}
\usepackage{color}
\definecolor{niceRed}{RGB}{190,38,38}
\definecolor{niceYellow}{HTML}{f5b400}
\definecolor{blueGrotto}{HTML}{059DC0}
\definecolor{royalBlue}{HTML}{057DCD}
\definecolor{navyBlue}{HTML}{0B579C}
\definecolor{limeGreen}{HTML}{81B622}
\definecolor{nicePurple}{HTML}{9c27b0}
\definecolor{lightRoyalBlue}{HTML}{def2ff}  
\definecolor{gold}{HTML}{ffa300}

\theoremstyle{plain} %
\newtheorem{theorem}{Theorem}[section]
\newtheorem{corollary}{Corollary}
\newtheorem{proposition}{Proposition}
\newtheorem{lemma}{Lemma}
\newtheorem{fact}[theorem]{Fact}

\newtheorem{claim}{Claim}

\newtheorem{assumption}{Assumption}

\newtheorem{inftheorem}{Informal Theorem}

\newtheorem{definition}{Definition}

\newtheorem*{definition*}{Definition}

\theoremstyle{definition} %

\theoremstyle{remark} %

\newtheorem{remark}{Remark}

\AfterEndEnvironment{definition}{\noindent\ignorespaces}
\AfterEndEnvironment{infdefinition}{\noindent\ignorespaces}
\AfterEndEnvironment{example}{\noindent\ignorespaces}
\AfterEndEnvironment{assumption}{\noindent\ignorespaces}
\AfterEndEnvironment{lemma}{\noindent\ignorespaces}
\AfterEndEnvironment{theorem}{\noindent\ignorespaces}
\AfterEndEnvironment{proposition}{\noindent\ignorespaces}
\AfterEndEnvironment{fact}{\noindent\ignorespaces}
\AfterEndEnvironment{question}{\noindent\ignorespaces}
\AfterEndEnvironment{corollary}{\noindent\ignorespaces}
\AfterEndEnvironment{claim}{\noindent\ignorespaces}
\AfterEndEnvironment{model}{\noindent\ignorespaces}
\AfterEndEnvironment{remark}{\noindent\ignorespaces}
\AfterEndEnvironment{proof}{\noindent\ignorespaces}
\AfterEndEnvironment{fact}{\noindent\ignorespaces}
\AfterEndEnvironment{minftheorem}{\noindent\ignorespaces}
\AfterEndEnvironment{inftheorem}{\noindent\ignorespaces}
\AfterEndEnvironment{maintheorem}{\noindent\ignorespaces}
\AfterEndEnvironment{restatable}{\noindent\ignorespaces}
\AfterEndEnvironment{observation}{\noindent\ignorespaces}
\AfterEndEnvironment{restatable}{\noindent\ignorespaces}

\crefname{section}{Section}{Sections}
\crefname{theorem}{Theorem}{Theorems}
\crefname{theorem*}{Theorem}{Theorems}
\crefname{inftheorem}{Informal Theorem}{Informal Theorems}
\crefname{assumption}{Assumption}{Assumptions}
\crefname{definition}{Definition}{Definitions}
\crefname{infdefinition}{Informal Definition}{Informal Definitions}
\crefname{conjecture}{Conjecture}{Conjectures}
\crefname{corollary}{Corollary}{Corollaries}
\crefname{construction}{Construction}{Constructions}
\crefname{conjecture}{Conjecture}{Conjectures}
\crefname{claim}{Claim}{Claims}
\crefname{observation}{Observation}{Observations}
\crefname{proposition}{Proposition}{Propositions}
\crefname{fact}{Fact}{Facts}
\crefname{question}{Question}{Questions}
\crefname{problem}{Problem}{Problems}
\crefname{remark}{Remark}{Remarks}
\crefname{example}{Example}{Examples}
\crefname{equation}{Equation}{Equations}
\crefname{appendix}{Appendix}{Appendices}
\crefname{algorithm}{Algorithm}{Algorithms}
\crefname{model}{Model}{Models}
\crefname{figure}{Figure}{Figures}
\crefname{condition}{Condition}{Conditions}
\crefname{lemma}{Lemma}{Lemmas}

\usepackage{tcolorbox}

\newlist{asmpenum}{enumerate}{1} %
\setlist[asmpenum]{label={\arabic*.},ref=\theassumption.{\arabic*}}
\crefname{asmpenumi}{Assumption}{Assumptions}
 
\newlist{infasmpenum}{enumerate}{1} %
\setlist[infasmpenum]{label={\arabic*.},ref=\theassumption.{\arabic*},leftmargin=20pt}
\crefname{infasmpenumi}{Informal Assumption}{Informal Assumptions}

\usepackage[suppress]{color-edits}
\addauthor[Sinho]{sc}{niceRed}
\addauthor[Alkis]{ak}{blue}
\addauthor[Anay]{am}{gold}
\addauthor[Omar]{om}{teal}

\usepackage{thm-restate}

\newcommand{\thetaStar}{\ensuremath{\theta^\star}}

\renewcommand{\d}{\mathrm{d}}

\newcommand{\DGS}{\ensuremath{\mathsf{DGS}}}
\newcommand{\GapSVP}{\ensuremath{\mathsf{GapSVP}}}
\newcommand{\LWE}{\ensuremath{\mathsf{LWE}}}
\newcommand{\CLWE}{\ensuremath{\mathsf{CLWE}}}
\newcommand{\hCLWE}{\ensuremath{\mathsf{hCLWE}}}

\allowdisplaybreaks{}

\AfterEndEnvironment{definition}{\noindent\ignorespaces}
\AfterEndEnvironment{assumption}{\noindent\ignorespaces}
\AfterEndEnvironment{lemma}{\noindent\ignorespaces}
\AfterEndEnvironment{theorem}{\noindent\ignorespaces}
\AfterEndEnvironment{proposition}{\noindent\ignorespaces}
\AfterEndEnvironment{fact}{\noindent\ignorespaces}
\AfterEndEnvironment{question}{\noindent\ignorespaces}
\AfterEndEnvironment{corollary}{\noindent\ignorespaces}
\AfterEndEnvironment{model}{\noindent\ignorespaces}
\AfterEndEnvironment{remark}{\noindent\ignorespaces}
\AfterEndEnvironment{proof}{\noindent\ignorespaces}
\AfterEndEnvironment{fact}{\noindent\ignorespaces}
\AfterEndEnvironment{minftheorem}{\noindent\ignorespaces}
\AfterEndEnvironment{inftheorem}{\noindent\ignorespaces}
\AfterEndEnvironment{maintheorem}{\noindent\ignorespaces}
\AfterEndEnvironment{restatable}{\noindent\ignorespaces}
\AfterEndEnvironment{infassumption}{\noindent\ignorespaces}
\AfterEndEnvironment{conjecture}{\noindent\ignorespaces}
\AfterEndEnvironment{claim}{\noindent\ignorespaces}

\newcommand{\negSepEnumerate}{\vspace{0mm}}

\title{  
    {DDPM Score Matching and Distribution Learning}
}

\ifanon
\date{}
\author{Anonymous Author(s)}
\else
\date{Yale University}
\author{Sinho Chewi \and Alkis Kalavasis \and Anay Mehrotra \and Omar Montasser}
\fi

\newcommand{\msf}[1]{\mathsf{#1}}

\begin{document}

\maketitle 
    
\begin{abstract}
    Score estimation is the backbone of score-based generative models (SGMs), and particularly denoising diffusion probabilistic models (DDPMs). A fundamental theoretical result in this area is that, given access to accurate score estimates, SGMs can efficiently generate from any realistic data distribution (Chen, Chewi, Li, Li, Salim, and Zhang, ICLR'23; Lee, Lu, and Tan, ALT'23).  This can be viewed as a result on distribution learning, where the learned distribution is implicit as the law of the output of a sampler. However, it is unclear how score estimation relates to more classical forms of distribution learning, such as parameter estimation and density estimation.

    We present a framework reducing the other two forms of distribution learning to score estimation, which has various implications in statistical and computational learning theory:
    \begin{enumerate}
        \item[$\blacktriangleright$] (Parameter Estimation)~~ Recent work has shown that for parametric estimation, a variant of score matching known as implicit score matching is provably statistically inefficient for multimodal densities, common in practice (Koehler, Heckett, and Risteski, ICLR'23). In contrast, under mild conditions, we show that denoising score matching in DDPMs is asymptotically efficient, \ie{}, the DDPM estimator is asymptotically normal with a covariance matrix given by the inverse Fisher information. 
        \item[$\blacktriangleright$] (Density Estimation)~~ Given the reduction from generation to score estimation, there is a large volume of work providing statistical and computational guarantees for learning the score of a distribution. Using our framework, we can lift the estimated scores to a $(\epsilon,\delta)$-PAC density estimator, \ie{}, a function that $\epsilon$-approximates the target log-density in all but a $\delta$-fraction of the space.  To illustrate our framework, we provide two results: (i) minimax rates for density estimation over H\"older classes of densities in the standard $L^1$ risk and (ii) a quasi-polynomial PAC density estimation algorithm for the classical Gaussian location mixture model. For the latter result, our result builds on and answers an open problem in the recent work of Gatmiry, Kelner, and Lee (COLT'25).

        \item[$\blacktriangleright$] (Lower Bounds for Score Estimation)~~ Using the power of PAC density estimation, we can prove computational lower bounds for score estimation for general distribution families. As an application, we prove cryptographic lower bounds for score estimation of general Gaussian mixture models, conceptually recovering the results of Song (NeurIPS'24) and making progress on Song's key open problem. 
    \end{enumerate}
\end{abstract}
 
\thispagestyle{empty}
\clearpage

\thispagestyle{empty}

\tableofcontents
\thispagestyle{empty}

\clearpage
\pagenumbering{arabic}
    
\section{Introduction}

    Score-based generative models (SGMs), also known as diffusion models, have emerged as a popular approach to generate samples from complex data distributions.
    These models leverage learned score functions---that is, the logarithmic gradients of the probability density---to progressively transform white noise into samples from the target data distribution by following a stochastic differential equation (SDE).
    The remarkable empirical success of SGMs has not only led to impressive practical applications but has also spurred significant interest within the {theoretical computer science~\citep[\eg{},][] {pabbaraju2023provableBenefits,shah2023learning,chen2024learninggeneralgaussianmixtures,gatmiry2024learning} and the {machine learning and} statistics communities~\citep[\eg{},][] {dou2024optimalscorematchingoptimal,chen2023sampling,wibisono2024optimalscoreestimationempirical,oko2023minimalDiffusion,chen2023ode,qin2024fit,koehler2023statistical, KoeLeeVuo25DataInit} toward establishing rigorous theoretical foundations for SGMs.

    More specifically, SGMs evolve the data distribution along a noising process, producing a family of distributions $(P_t)_{t\in [0,T]}$.
    In this work, for concreteness, we focus exclusively on the Ornstein--Uhlenbeck (OU) process, for which
    \begin{align*}
        P_t \deq \Law(X_t) \deq \Law(e^{-t}\,X_0 + \sqrt{1-e^{-2t}}\,Z_t)\,, \qquad X_0 \sim P\,, \; Z_t \sim \normal{0}{{\mathrm{Id}}}\,, \; X_0 \indep Z_t\,.
    \end{align*}
    Then, to generate a sample from $P = P_0$, SGMs numerically discretize a certain SDE (obtained as the time reversal of the noising process), the implementation of which crucially requires estimation of the \emph{score functions} $\{\nabla \log P_t\}_{t\in [0,T]}$.
    We refer the reader to~\Cref{appendix:ddpm} for background.
    
    A central component underlying this procedure is \emph{score estimation}~\citep{hyvarinen2005scoreMatching}, which transforms the problem of learning the score function into a regression objective amenable to first-order optimization.
    It relies on the following key identity, valid for any vector field $s_t\colon \R^d\to\R^d$:
    \begin{align}\label{eq:sm_identity}
        \E\bigl[\|s_t(X_t) - \nabla \log P_t(X_t)\|^2 \bigm\vert X_0\bigr]
        &= \E\Bigl[\|s_t(X_t)\|^2 + \frac{2\, \langle s_t(X_t), Z_t\rangle}{\sqrt{1-e^{-2t}}}\, \Bigm\vert X_0\Bigr] + C(P, X_0)\,,
    \end{align}
    where $C(P,X_0)$ is a constant that does not depend on $s_t$.\footnote{The expectation above is often replaced with $\E[\|s_t(X_t) + {\nfrac{Z_t}{\sqrt{1-\exp(-2t)}}}\|^2 \mid X_0]$, which is formally equivalent by completing the square.
    However, we prefer to write the identity as above since $\int_0^T \E[\|\nfrac{Z_t}{\sqrt{1-\exp(-2t)}}\|^2]\,\d t = \infty$.}
    Since we can freely generate $Z_t$ (and thus $X_t$) given $X_0$, the right-hand side of the identity above is readily turned into an empirical loss that can be minimized over the choice of $s_t$.

    Despite a number of recent works investigating its efficacy {\cite{qin2024fit, koehler2024sampling, KoeLeeVuo25DataInit}}, a complete statistical understanding of score matching remains to be developed.
    In this work, we aim to clarify the relationships between this problem and the well-studied problem of distribution learning \cite{haussler1992decision,kearns1994learnability}.
    Concretely, we study the following general question.
    \begin{mdframed}
    \textbf{Main Question:} Given a family $\hyP$ of probability distributions and samples from $P \in \hyP$, how does the complexity (both computational and statistical) of learning the score functions $\{\nabla \log P_t\}_{t\in [0,T]}$ relate to the complexity of learning the distribution $P$?
    \end{mdframed}
    \noindent In order to make this question precise, we must specify the sense in which we learn the distribution.
    \begin{itemize}[itemsep=0pt]
        \item[$\blacktriangleright$] \textbf{Parameter estimation:} If $\hyP = \{P_\theta : \theta \in\Theta\}$ is indexed by a finite-dimensional parameter space $\Theta \subseteq \R^p$, then we can ask to output an estimate of the true value of the parameter.
        \item[$\blacktriangleright$] \textbf{Density estimation:} We can also ask to output an \emph{evaluator}, that is, a function (or algorithm) $\hat P$ which, given an input $x$, yields an estimate $\hat P(x)$ of the density of $P$ evaluated at $x$.
        When computational considerations are at play, we further ask that $\hat P$ runs in polynomial time (with respect to various problem parameters).
        This setting is especially well-established in the statistics literature \citep{Tsy09Nonpar}.
        \item[$\blacktriangleright$] \textbf{Learning a sampler:} A more recent paradigm\footnote{We note, however, that the definition goes back at least to \cite{kearns1994learnability}.} instead asks to learn a \emph{generator}, that is, a function (or algorithm) which, given a random seed, outputs a sample $\hat X$ such that $\Law(\hat X) \approx P$.
        Again,  we sometimes ask that the generator runs in polynomial time.\footnote{If we ignore computational aspects, then evaluators and generators are equivalent objects{: given a generator specified by, \eg{}, a circuit, one can estimate the relative frequency of each point in its support to sufficient accuracy to estimate the density at each point; and given an evaluator, one can generate a sample via rejection sampling}.}
    \end{itemize}
    SGMs most naturally correspond to the last paradigm since they are typically viewed as samplers.
    The prior works~\citet{chen2023sampling, LeeLuTan23} made this connection rigorous by showing that, provided the score functions for $\hyP$ are Lipschitz, \textbf{score estimation implies that one can learn a sampler}.
    More precisely, if $\nabla \log P_t$ is $L$-Lipschitz for all $t\ge 0$ and $P \in \hyP$, and we have score estimates $\{\hat s_t\}_{t\ge 0}$ satisfying the guarantee $\sup_{t\ge 0} \|\hat s_t - \nabla \log P_t\|_{L^2(P_t)}^2 \le \varepsilon^2$, then SGMs can learn a sampler up to error $\tilde O(\varepsilon)$ in total variation distance using a number of score evaluations and additional computation time that scales polynomially in $L$, \mbox{the dimension $d$, and {the desired accuracy} $\nfrac{1}{\eps}$.}\footnote{{See~\Cref{ssec:related} for further developments.}}

    \subsection{Framework and first main tool}

    Informally, our main message is that the other two main settings for distribution learning, namely parameter and density estimation, can both be reduced to score estimation as well (see~\Cref{fig:main}).

   \begin{figure}[!ht]
       \centering
       \includegraphics[width=0.75\linewidth,clip,trim={1.5cm 0.5cm 0.75cm 0.75cm}]{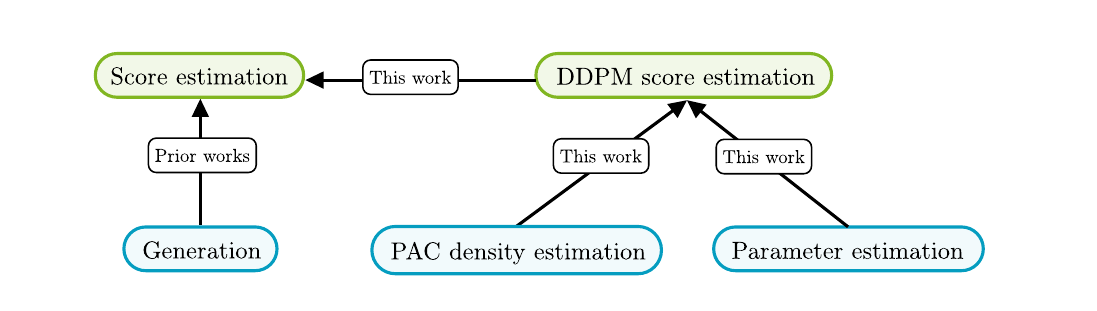}
       \caption{Landscape of reductions between score estimation and distribution learning (see \cref{sec:intro:pacDensity} for a definition of PAC density estimation). Prior to our work, the only known reduction was from {generation to score estimation}. Our work shows that score estimation has strong implications for parameter and density estimation. We introduce the notion of DDPM score estimation in \Cref{def:ddpmMain} and provide relevant background in \Cref{appendix:ddpm}.}
       \label{fig:main}
   \end{figure}
   
    \noindent Given the above landscape of computationally efficient reductions from density estimation and parameter estimation to score estimation, we can obtain several new results in statistical and computational learning theory, which we review in the upcoming sections. Before that, we present the following tool, which is {a core part of} all of our results.

     \subsubsection{Likelihood identity}

        The starting point is the following identity for the likelihood.
This identity relates the log-likelihood of point $x_0$ under the target density $P$ and {a certain integrated} score matching formula {corresponding to DDPM score matching where the} OU process is run until time $T$ starting from $x_0.$
Here, we {use} $Q_{t|0}$ {to denote} the transition semigroup of the OU process, \ie{}, $Q_{t|0}(\cdot\mid x_0)
\coloneqq \cN(e^{-t},x_0,; (1-e^{-2t}),\mathrm{Id})$.

    \begin{lemma}[Likelihood identity]
    Let $P$ be a continuous density over $\R^d$ with finite second moment.
    Then, for all $x_0\in\R^d$,
    {
    \setlength{\abovedisplayskip}{2pt}
    \setlength{\belowdisplayskip}{2pt}
    \begin{align*}
        &\textstyle \int \log P_{T}\,\d Q_{T|0}(\cdot \mid x_0) - \underbrace{~\log P(x_0)~}_{\textnormal{log-density at $x_0$}} \\
        &\textstyle \quad ~~=~~
        \underbrace{\textstyle \int_0^T \int \bigl\{\|\nabla \log P_{t}\|^2 - 2\,\langle \nabla \log P_{t},\nabla \log Q_{t|0}(\cdot \mid x_0)\rangle\bigr\}\,\d Q_{t|0}(\cdot\mid x_0) \,\d t}_{\textnormal{integrated DDPM score matching objective at $x_0$}} 
        ~+~
        \underbrace{~~d \cdot T \textstyle ~~}_{\textnormal{known constant}}\,.
    \end{align*}
    }
    \label{lem:identity}
\end{lemma}
    Variants of this identity appear in the literature, \eg{}, in~\citet{DBLP:conf/iclr/ChenLT22, li2024score, KarHeiGar25Cartoon}.
    A similar formula already appears in~\citet{song2021MLE} as a variational lower bound on the log-likelihood, and likely traces back to ideas by  \citet{jarzynski1997nonequilibrium}.
    Similar, although perhaps not identical, formulas have also appeared in the information-theory literature; see \citet{chen2012mismatched,guo2009relativeEntropy,weissman2010relationship,verdu2010mismatched} and the references therein.
    For completeness, we prove~\Cref{lem:identity} in~\Cref{sec:identity_pf}.
    In any case, while we do not claim novelty for~\Cref{lem:identity}, our contribution is to thoroughly explore its consequences for distribution learning.

    The power of the identity is that, by the convergence of the OU process, the first term on the left-hand side converges exponentially fast, as $T\to\infty$, to a known constant (see~\Cref{lem:err_term} below).
    Moreover, the integral on the right-hand side exactly corresponds to the score-matching loss from~\eqref{eq:sm_identity} {which is the} (single-sample version of the) DDPM objective evaluated at $x_0$ (for details, we also refer to \Cref{def:ddpmMain}). 
    We call this quantity the \emph{integrated DDPM score matching objective}.
    Therefore,~\Cref{lem:identity} shows that the negative log-likelihood $-\log P$ is precisely related, up to a known constant and a vanishing error, to the integrated score-matching objective. 

    In upcoming sections, we discuss applications of \Cref{lem:identity} to parameter and density estimation.

\subsection{Applications to parameter estimation}\label{sec:applications-parameter}
    
    In this section, we assume that our class of distributions is a parametric family  $\hyP = \{P_\theta : \theta \in\Theta\}$  with {parameter space} $\Theta \subseteq \R^p$. 
    In this setting, the prior work of~\citet{koehler2023statistical} investigated {the performance of} a variant of score matching known as \emph{implicit score matching} (ISM):
    Given \iid{} samples $X^{(1)},\dotsc,X^{(n)}$  from $P_{\theta^\star}$, {for some parameter} $\theta^\star \in \Theta$, the ISM estimator is 
    \begin{align*}
        \textstyle \hat\theta_n^{\,\rm ISM} \deq \argmin_{\theta\in\Theta} \frac{1}{n} \sum\nolimits_{i=1}^n \bigl\{
        \|
        \nabla \log P_\theta(X^{(i)}) \|^2 
        + 
        2\,\Delta \log P_\theta(X^{(i)})
        \bigr\}
        \,.
    \end{align*} 
    \noindent Under appropriate regularity conditions, they proved $\wh{\theta}_n^{\,\rm ISM}$ is asymptotically normal, \ie, 
    $
        \sqrt n\,(\hat\theta_n^{\,\rm ISM} - \theta^\star)
        ~\todist{}~
        \cN(0, \Sigma^{\rm ISM}(\theta^\star)).
    $
    Moreover, they bounded the operator norm of $\Sigma^{\rm ISM}(\theta^\star)$ in terms of the asymptotic covariance of the maximum likelihood estimator (MLE)---\ie{}, the inverse Fisher information matrix---and the so-called restricted Poincaré constant of the model.
    This shows that the above parameter estimator $\hat\theta_n^{\,\rm ISM} $ can achieve the asymptotic efficiency of MLE up to a constant factor for distributions whose restricted Poincaré constant is $O(1)$.
    While their result is insightful, the dependency on the restricted Poincaré constant is not harmless: it can be arbitrarily large for multimodal distributions, frequently arising in practice.
    Unfortunately, this dependence is also unavoidable, since~\citet{koehler2023statistical} also give examples in which \textbf{implicit score matching is provably inefficient compared to MLE}; see also~\citet{Diao23ScoreMatching}. 
    
    \citet{qin2024fit} generalized this asymptotic efficiency result to generalized (implicit) score matching estimators \citep{lyu2012interpretation} by linking the \emph{mixing time} of broad classes of Markov processes, and the statistical efficiency of an appropriately chosen \emph{generalized score} matching loss (GISM). 
    Under this framework, for Gaussian mixtures in $d$ dimensions, they show that the generalized score estimator is asymptotically normal with covariance matrix $\Sigma^{\mathrm{GISM}}(\theta^\star)$ whose operator norm is (roughly) at most $\poly(d)$ times the (squared) operator norm of the inverse Fisher information, bypassing the lower bounds of \citet{koehler2023statistical}.

    In short, both works \citet{koehler2023statistical,qin2024fit} indicated strong statistical properties of (generalized) ISM{. That said, they} still cannot match the performance of MLE or come within a constant factor of it {for general families $\hyP$, and they left open whether some diffusion-based estimator can achieve the statistical efficiency of MLE under mild assumptions on $\hyP$.}
    
    Here, we consider denoising diffusion probabilistic models (DDPMs)---arguably the most popular variant used in practice---which interestingly do not rely on {(generalized)} implicit score matching.
    Instead, DDPMs employ an alternative known as \emph{denoising score matching}~\citep{hyvarinen2008scoreMatching, pascal2011dsm} and extend the idea by applying score matching at many different noise levels~\citep{song2019generative, jonathan2020ddpm,ling2024SurveyDiffusion}, leading to the following risk:
        \[
            \int_0^T\E\Bigl[
                 \bigl\{\|s_t(X_t)\|^2 + \frac{2}{\sqrt{1-e^{-2t}}}\,\langle s_t(X_t), Z_t\rangle\bigr\}\, 
            \Bigm\vert X_0\Bigr] \d t\,.
            \yesnum\label{eq:ddpm_loss}
        \]

\subsubsection{DDPM is an asymptotically efficient parameter estimator}\label{sec:intro:efficiency}

    We consider the following DDPM estimator, which minimizes the DDPM risk in \cref{eq:ddpm_loss} over samples from $P$.
    Below $P_{\theta, t}$ denotes the law of $X_t \deq {e^{-t}}\,X_0 + \sqrt{1-{e^{-2t}}}\,Z_t$, where $X_0 \sim P_\theta$ and $Z_t \sim \cN(0, \mathrm{Id})$ are independent.
    We provide background on the DDPM objective in \Cref{appendix:ddpm}.

    \begin{restatable}[]{definition}{ddpmEstimator}\label{def:ddpmMain}
    Fix a terminal time $T>0$. Given samples $X_0^{(1)},\dotsc,X_0^{(n)}$ and a family $\hyP=\{P_\theta: \theta \in \Theta\subseteq \R^p\}$, the DDPM estimator is $\DDPM \coloneqq \argmin_{\theta\in \Theta} \hat{\cR}_n^{\rm DDPM}(\theta)$, where
    \[
        \hat{\cR}_n^{\rm DDPM}(\theta) \deq \frac{1}{n}\sum_{i=1}^n
        \int_{0}^T 
        \E\Bigl[\|\nabla \log P_{\theta,t}(X_t^{(i)})\|^2 
        +
        \bigl\langle\nabla \log P_{\theta,t}(X_t^{(i)}),\frac{2Z_t^{(i)}}{\sqrt{1-{e^{-2t}}}}\bigr\rangle \Bigm\vert X_0^{(i)}\Bigr] \, \d t
    \]
    and for each $i \in [n]$ and $t\in [0,T]$, we draw $Z_t^{(i)} \sim \cN(0, \mathrm{Id})$
    independently from $X_0^{(i)}$ and define the noised sample $X_t^{(i)} \deq {e^{-t}}\,X_0^{(i)} + \sqrt{1-{e^{-2t}}}\,Z_t^{(i)}$.
    \end{restatable}
    Our main result {for this application} is that, under mild regularity assumptions on the distribution family $\hyP$ (essentially the same conditions needed for the asymptotic normality of the MLE, see~\Cref{ass:mle}) and by choosing the terminal time $T = T_n$ to grow sufficiently rapidly with the number of samples $n$ (namely, $T_n - \frac{1}{2}\log n \to \infty$), 
    the DDPM estimator $\DDPM$ converges in distribution to a Gaussian centered at $\theta^\star$ with covariance \emph{exactly} equal to the inverse Fisher information.
    Recall that the inverse Fisher information is also the asymptotic covariance of the MLE and is the best possible for any unbiased estimator (by the Cram\'er--Rao or information inequality) {\cite{Vaart98Asymptotic}}, so this statement can be interpreted as a form of asymptotic optimality; furthermore, by comparison of experiments, the MLE can be shown to be locally asymptotically minimax \cite{hajek1972local,Vaart98Asymptotic}.
    To state the result formally, 
     let $\DDPM$ denote the DDPM estimator as defined in~\Cref{def:ddpmMain} on $n$ \iid{} samples $X_0^{(1)},\dotsc,X_0^{(n)}\sim P_{\theta^\star}$.
     
        \begin{restatable}[DDPM is asymptotically efficient; see~\Cref{thm:efficiency}]{inftheorem}{thmEfficiency}\label{thm:normalityEfficiency}
            Under standard assumptions, $\sqrt{n}\,\bigl(
                    \DDPM - \thetaStar
                \bigr)
                \todist{}
                \cN(0, I(\thetaStar)^{-1})
                $ as $
                n\to\infty,$
            where $I(\thetaStar)$ denotes the Fisher information matrix at $\thetaStar$.
        \end{restatable}
    \Cref{thm:normalityEfficiency} provides an explanation for the statistical power of the DDPM estimator in the asymptotic regime and has immediate implications for parameter recovery; see \Cref{sec:efficiency} for more.

    \subsection{Applications to density estimation}\label{sec:applications-density}

     A long line of works showed that, under minimal regularity assumptions, access to a score estimation oracle for $P$ is sufficient for learning to generate from the target density; see~\Cref{ssec:related} for a discussion of this literature.
    However, the precise connection between score estimation and density estimation remains elusive. One of the main results of our framework is an efficient reduction from a particular notion of density estimation to score estimation, which we define below.

\subsubsection{PAC density estimation}
    \label{sec:intro:pacDensity}

To formally define our notion of density estimation, we need the following evaluation oracle:

\begin{definition}[Evaluation oracle]
    \label{def:evalOracle}
    An evaluation oracle for a function $f\colon \R^d\to {\R}$ is a primitive that, given a point $x\in \R^d$, outputs {$f(x)$}.
\end{definition}
Density estimation is extremely well-studied in statistics and computer science. 
Prior works on density estimation~\citep[see, \eg{},][] {kearns1994learnability,diakonikolas2016learning,barron1991minimum,feldman2006pac,kalai2010efficiently,moitra2010settling,suresh2014near,daskalakis2014faster,diakonikolas2019robust,li2017robust,ashtiani2018nearly,diakonikolas2020small,bakshi2022robustly} focus on finding a model $\wh P$ (\ie{}, in the form of an evaluation oracle for the target density) such that, with high probability over the training set, 
it estimates the target density $P$ with high accuracy.
There are many variants of density estimation {which differ in the specific} evaluation metric {used}; \eg{}, it can be {the} total variation {distance} or KL divergence \cite{kearns1994learnability, diakonikolas2016learning, diakonikolas2020small}.

We are now ready to define the notion of probably approximately correct (PAC) density estimation, which is a slight relaxation of the above requirement.
    The goal in PAC density estimation is to output a model $\wh P$ that fits the target density $P$ everywhere except for {a} $\delta$-fraction of points {(according to the probability mass of the target distribution $P$)}, with high probability.

\begin{definition}
[PAC density estimation algorithm]\label{def:PACdensityEstimation}
    Let $\hyP$ be a class of distributions over $\R^d$. 
    An $(\eps,\delta)$-PAC density estimation algorithm for $\hyP$ is an algorithm which, for any $P \in \hyP,$ given $\eps,\delta>0$ and $\poly(d,\nfrac{1}{\eps},\nfrac{1}{\delta})$ i.i.d.\ samples drawn from  distribution $P$, outputs a representation (in the form of an evaluation oracle) of a possibly randomized function $\wh P$
    such that with probability at least $\nfrac{9}{10}$ over the randomness of the samples and the algorithm,
    \[
        \E~ P\{x\in\R^d : e^{-\eps}\,P(x) \le \wh P(x) \le e^{\eps}\,P(x)\} \geq 1- \delta\,,
    \]
    where $\E$ denotes the expectation over the randomness of $\wh P$.
\end{definition}
In the above definition, we often call $\epsilon$ the \emph{accuracy} of the algorithm and $\delta$ its \emph{coverage}.
If additionally the $(\epsilon,\delta)$-PAC learner runs in sample polynomial time, we call it \emph{efficient}.\footnote{As is standard, efficiency has two aspects: both (i) producing the output $\wh P$ and (ii) evaluating $\wh P(x)$ at any given point $x$ at inference time should run in time that is polynomial in the number of samples.}

Some further remarks are in order. {We do not require $\wh P$} to be a density. 
    In this sense, $\wh P$ {multiplicatively} estimates the density of $P$ on most of the domain, but is not a density itself.
    For parametric families, this means that PAC density estimation is weaker, perhaps strictly, than parameter estimation which would provide density estimation on the entire domain.
    Omitting this requirement is standard in both theoretical computer {science}~\citep[\eg{},][] {chan2014efficient,acharya2017sample} and (non-parametric) statistics \cite{scott2015multivariate,devroye2001combinatorial} (\eg{}, kernel density {estimators can output functions that neither integrate to 1 nor have a non-negative range}). 
    Our PAC density estimators are not densities (\ie{}, they do not integrate to 1) but they approximate the target density $P$ everywhere except for a $\delta$-fraction of the points {and take non-negative values}.

     To further motivate the non-triviality of getting such a PAC guarantee, we give evidence that this problem is essentially as hard as standard density estimation for standard families: We show that for H\"older classes, existing minimax lower bounds also hold against PAC density estimation (\Cref{sec:smoothDensityEstimation}), and that PAC density estimation for Gaussian mixtures can be used to solve (C)LWE and is therefore cryptographically hard in some regime (\Cref{sec:clwe-pac-density}). 
    Such a reduction was known for the stronger notion of density estimation in total variation and is an important reason why density estimation is believed to be computationally challenging \cite{bruna2021continuous,gupte2022continuous}.
    
The success probability of $\nfrac{9}{10}$ in \cref{def:PACdensityEstimation} is arbitrary and can straightforwardly be boosted to achieve $(\eps,3\delta)$-PAC density estimation with a success probability $1 -(\nfrac{1}{10})^{O(k)}$ by computing the median of $k$ $\inparen{\eps,\delta}$-PAC density estimators with success probability $\nfrac{9}{10}$. 

\subsubsection{Minimax optimal density estimation over H\"older classes}\label{sec:intro:densityEstimation}

    There is a long line of works that studied statistical rates for score estimation and thereby obtained minimax optimal generators. Using our tools, we can directly translate these results to PAC density estimators ({we briefly overview this below and refer the reader to \Cref{sec:smoothDensityEstimation} for more details}). 

    To see an application, we consider a representative result on score estimation coming from the {recent} work of \citet{dou2024optimalscorematchingoptimal}. Their work considered the H\"older class of densities (see \Cref{defn:holder}) and showed that the score $\nabla \log P_t$ can be approximated in $L^2(P_t)$ at a certain rate for any $P$ in that class. 
    Prior to our work, such a result had only implications to generation; the next result converts such a guarantee to a density estimation result. We mention that for this section, we actually achieve a stronger guarantee compared to PAC density estimation: we give minimax optimal rates for estimation in the standard $L^1$ risk using an estimator based on DDPM score matching.
    
    \begin{inftheorem}[Density estimation for H\"older classes; informal, see \Cref{thm:holder}]\label{infthm:minimax}
    
    Let $\hyH_s$ be the H\"older class of densities $P$ supported on $[-1,1]$ {with smoothness parameter $s > 0$}.
    Define the risk of an estimator $\wh P$ using $n$ samples to be
    $\cR_n(\wh P, P) \coloneqq \int_{[-1, 1]} \E_P|\wh P(x_0) - P(x_0)| \, \d x_0$.\footnote{Here, $\E_P[\cdot]$ denotes the expectation over the samples used to train the model $\wh P.$}
    Then, an estimator based on DDPM score estimation achieves the minimax risk $n^{-s/(2s+1)}$ over $\hyH_s$ up to a $\sqrt{\log n}$ factor, using $\poly(n)$ calls to the score estimation oracle.
    \end{inftheorem}
    
\subsubsection{Algorithms for PAC density estimation via score matching} 

    Apart from statistical implications, a key conceptual message for our density-to-score estimation reduction is that it is also \emph{computationally efficient.} 
    We believe that this can lead to novel algorithms for PAC density estimation. 
    To illustrate this, we demonstrate a perhaps surprising application to PAC density estimation for the following Gaussian location {mixture}, which goes back to \citet{kiefer1956consistency} and was studied from an algorithmic perspective in a recent work by  \citet{gatmiry2024learning}{: Given distribution $Q$, the corresponding Gaussian location {mixture} is}
    \[
            \textstyle P = Q * \normal{0}{\sigma^2\,\mathrm{Id}}\,,
    \]
    for some variance parameter $\sigma^2$ (controlling the smoothness of the target density).
    For instance, when $Q$ is supported on $k$ discrete points, {then} $P$ is a spherical Gaussian mixture model. 
    In general, there is a long list of works where the mixing distribution $Q$ is non-parametric~\citep[see, \eg,][] {kim2014minimax,genovese2000rates,ghosal2001entropies,sahaGuntuboyina2020gaussian,zhang2009generalized,polyanskiy2025nonparametric}.
    Following \citet{gatmiry2024learning}, it is assumed that $Q$ satisfies the following structural properties (see \cref{asmp:GLM:locality}):
    \begin{enumerate}[itemsep=0pt]
        \item The support of $Q$ is contained in $k$ {$\ell_2$}-balls of radius $R$.
        \item {The} {$\ell_2$-}ball {of radius $R$} around any point of the support of $Q$ has mass at least ${\nfrac{1}{k}}$.
        \item The support of $Q$ is {a subset of the $\ell_2$-ball} of radius $D$ centered at {the origin}.
    \end{enumerate}
    Our starting point is the important recent work by \citet{gatmiry2024learning} that provides a \emph{generator} for this family of multi-modal densities. 
    {Their algorithm} is based on score estimation, runs in time roughly $d^{\poly(\log(1/\epsilon), R/\sigma)}$ when $\epsilon < \min\inbrace{\nfrac{\sigma}{R},\nfrac{1}{D},\nfrac{1}{d},\nfrac{1}{k}}$, and learns a sampler for the target density, using the well-known reduction from {generation} to score estimation.  
    
   {There are several reasons why this is interesting.}
First, \citet{gatmiry2024learning} {learn} to generate from various non-parametric models for which no sub-exponential-in-$d$ algorithms were previously known.
    In particular, this applies to the case of Gaussian convolutions of distributions on low-dimensional manifolds or, more generally, sets with a small covering number.
    Moreover, their algorithm provides an alternative to the algebraic toolbox employed for the density estimation of spherical GMMs. 
    Indeed, the algorithm of \citet{diakonikolas2020small} outputs a density estimator in time $\poly(\nfrac{dk}{\eps}) + (\nfrac{k}{\eps})^{O(\log^2 k)}$ for mixtures of $k$ spherical Gaussians using tools from algebraic geometry (tailored to GMMs) and it does not seem to extend to the more general distribution class of Gaussian location mixtures. (We note that a follow-up work by the same authors \citep{DiaKan25HighOrder} learns a sampler for spherical GMMs in $\poly(d,k,\nfrac{1}{\epsilon})$ time under a boundedness assumption on the means, which also relies on the algebraic structure of Gaussians.) 
    
However, the result of \citet{gatmiry2024learning} does not have any implications for the density estimation problem studied by the majority of works on GMMs. 
\citet{gatmiry2024learning} mention that it may be possible to upgrade their algorithm's guarantee from generation to density estimation but leave it as an open problem.
We make progress on \mbox{their open question by establishing the following result.}

\begin{inftheorem}
[PAC density estimation for Gaussian mixtures; {see \cref{thm:GLM:DE}}]\label{infthm:glm}
    Let $\calM$ be an $(k,R,D)$-Gaussian location mixture over $\R^d$ with $\sigma\in (0,1]$.
    Fix a target accuracy $\eps\leq\min\{\nfrac{1}{2}, \nfrac{\sigma}{R}, \nfrac{1}{D}, \nfrac{1}{d}, \nfrac{1}{k}\}$.
    There is an algorithm that, given 
        accuracy parameter $\varepsilon$, 
        instance parameters $(\sigma,k,R,D)$, and 
        sample access to $\cM$, 
    draws $N$ i.i.d.\ samples from $\calM$, runs in {$\poly(N)$} time, and, for any coverage parameter $\delta\in (0,1)$, returns an $(\nfrac{\eps}{\delta},\delta)$-PAC density estimator for $\cM$, where $N =
        \Bigl(
            d\log{\frac{1}{\eps}}
        \Bigr)^{
                {\polylog({\nfrac{1}{\eps})}}
                +
                {
                    \poly(\sfrac{R}{\sigma}
                )}
        }.$
    \end{inftheorem}
    In other words, we match {the sample complexity and runtime of} the algorithm of \citet{gatmiry2024learning}, but instead of a generator for the class of interest, we return a PAC density estimator.
    
\subsection{Applications to lower bounds for score estimation}\label{sec:app:LB}

So far, we have shown that score estimation yields powerful results for distribution learning, both statistically and computationally,}beyond the ability to generate samples.
In this last application, we ask a natural question: Do our reductions lead to impossibility results for score estimation itself?

\subsubsection{Cryptographic lower bounds for score estimation}

In this section, we show that our reduction, from PAC density estimation to score estimation, allows us to deduce computational bottlenecks to score estimation from the hardness of the latter.
Since our reduction from PAC density to score estimation is efficient, we obtain the following tool to show computational hardness for score estimation.

\begin{tcolorbox}
    \begin{center}
        Blueprint for computational lower bounds for score estimation under certain condition $C$
    \end{center}
    \begin{enumerate}[itemsep=0pt]
        \item Show that PAC density estimation for $\hyP$ is computationally hard under condition $C$.
        \item Show that $\hyP$ has Lipschitz scores\footnote{It is sufficient to show the weaker requirement that the scores are sub-Gaussian.} and bounded second moments.
    \end{enumerate}
\end{tcolorbox}

\noindent As an illustration of the type of hardness results that can be obtained, we prove the following hardness result, which (conceptually) recovers the lower bound of \citet{song2024cryptographic}:\footnote{\citet{song2024cryptographic} showed a reduction from distinguishing a Gaussian pancake and the standard Gaussian to score estimation of a Gaussian pancake. Since Gaussian pancakes are morally behind cryptographic lower bounds for GMMs \cite{bruna2021continuous}, one can obtain a series of cryptographic hardness results for GMM score estimation (which are not explicitly stated in \citet{song2024cryptographic}). We recover these cryptographic hardness results for GMM score estimation.} 
Under standard hardness of Learning with Errors (\LWE{}), score estimation for general GMMs with at least $d^{\,\epsilon}$ components (or even $(\log d)^{1+\varepsilon}$ under a stronger condition, following \citet{gupte2022continuous}) requires super-polynomial time in the dimension for any constant $\epsilon$.

In order to get this result, it suffices to apply our blueprint for GMMs: we must show that PAC density estimation for GMMs is computationally intractable under some standard complexity assumption. 
Following \citet{bruna2021continuous}, we show that \CLWE{} reduces to it. To complete the reduction, we have to show that our ``hard'' GMM instance satisfies the assumptions of our density-to-score estimation reduction, leading to the following result.

\begin{inftheorem}[Hardness of score estimation for Gaussian mixtures; see \cref{thm:cryptoHardnessGMM}]\label{infthm:scoreCrypto}

Under polynomial hardness of LWE, score estimation for Gaussian mixtures with $k \gtrsim d^{\,\epsilon}$ components for any constant $\epsilon > 0$ cannot be done in $\poly(d,k)$ time {with accuracy equal to $\nfrac{1}{\sqrt{d\log d}}$}.
\end{inftheorem}
A result similar in spirit to \Cref{infthm:scoreCrypto} is the main result of \citet{song2024cryptographic} specialized to the Gaussian pancakes distribution. A comparison with \citet{song2024cryptographic} appears in \Cref{ssec:related}. 

\begin{remark}
[Density estimation and goodness-of-fit tests]
Our central result is that PAC density estimation yields efficient goodness-of-fit testing. Concretely, when the null distribution is known, its density is efficiently evaluable, and the alternative class is TV-separated from the null, a PAC-accurate density estimate of the unknown distribution suffices to decide between the null and alternative. However, goodness-of-fit testing subsumes classic hard distinguishing tasks (\eg{}, separating the standard Gaussian from the Gaussian pancake). Consequently, for null--alternative pairs where efficient goodness-of-fit testing is computationally hard, PAC density estimation inherits the same hardness. Finally, since we give an efficient reduction from PAC density estimation to score estimation, this intractability transfers to score estimation as well.
\end{remark}
\negSepEnumerate{}

\noindent Moreover, our connection between PAC densities and score estimation already has implications for \citet{song2024cryptographic}'s open problem on finding natural assumptions on data distributions that eliminate Gaussian pancakes while allowing for rich data distributions encountered in practice.
        Our results make significant progress on this by showing that $L^2$ score estimation is computationally hard for any family of distributions for which evaluating the density is computationally hard, in the sense of PAC density estimation (and the family satisfies the assumptions of our reduction). 

{The formal version} of \Cref{infthm:scoreCrypto} (\cref{thm:cryptoHardnessGMM}) and its proof appear in \Cref{sec:hardnessScoreEstimation}.

    \subsection{Second main tool: Reduction from density estimation to score estimation}\label{sec:reduction}

    In this section, we present our second main tool (apart from the likelihood identity {presented earlier in} \Cref{lem:identity}),  which is our key technical contribution and will be {used} for our applications in density estimation (\Cref{sec:applications-density}) and {computational hardness results for score estimation} (\Cref{sec:app:LB}). {In particular, we present} a reduction from PAC density estimation to score estimation under mild assumptions on the target density $P$, which is the analogue of the standard reduction from generation to score estimation~\citep{chen2023sampling,LeeLuTan23}. 

    Before delving into the details, let us first introduce the score estimation oracle.

    \begin{definition}
    [Score estimation oracle; informal, see \cref{def:scoreEstimation}] 
    Let $P$ be a density on $\R^d$. A score estimation oracle for $P$ gets as  input $t$ and outputs a model $\wh s_t$ with $\int \|\wh s_t - \nabla \log P_t \|^2\, \d P_t \leq \varepsilon_t^2$.
    \end{definition}
    We define the error of the oracle with early stopping $\tau > 0$ and terminal time $T$ as $\eps_*^2$, where
    \[
        \textstyle \varepsilon_*^2 = \int_\tau^T \varepsilon_t^2\, \d t\,.
    \] 
    In our reduction, we obtain a PAC density estimator given access to a score estimation oracle. 
    
    \begin{inftheorem}
    [Score to PAC density estimation; informal, see \Cref{thm:reduction,thm:early_stopping}]\label{infthm:reduction}
    
    Let $P$ be a distribution on $\R^d$, and let $\varepsilon > 0$ be the desired accuracy. Assume access to a score estimation oracle with early stopping $\tau$ and error $\varepsilon_*$.
    \begin{enumerate}
        \item
        Assume that $P$ has second moment bounded by $M_2  \le \poly(d)$ and $L$-Lipschitz score function.
        There is an algorithm that outputs a function $\wh P$ (in the form of an evaluation oracle) such that
        \[
            \textstyle \int \E\bigl\lvert\log \frac{\wh P(x_0)}{P(x_0)}\bigr\rvert\, P(\d x_0) \lesssim \varepsilon\,.
        \]
        The algorithm makes $
        N = \wt{\Theta}(\nfrac{Ld^2}{\epsilon^2})$ calls to the score estimation oracle with accuracy $\varepsilon_* = \wt{\Theta}(\nfrac{\epsilon}{\sqrt{d \log L}})$ and runs in $\poly(N)$ time. The above hold when $\tau \lesssim \nicefrac{\varepsilon^2}{L d^2}.$

         \item 
        Assume only that $P$ has second moment bounded by $M_2 \le \poly(d)$. For any $0 < \tau \le 1$, there is an algorithm that outputs a function $\wh P_\tau$ (in the form of an evaluation oracle) such that
        \[
            \textstyle  \int \E\bigl\lvert\log \frac{\wh P_\tau(x)}{P_\tau(x)}\bigr\rvert\, P_\tau(\d x) \lesssim \varepsilon\,.
        \]
        The algorithm makes $
        N = \wt{\Theta}(\nfrac{(d^2 + 1/\tau)}{\epsilon^2})$ calls to the score estimation oracle with accuracy $\varepsilon_* = \wt{\Theta}(\nfrac{\epsilon}{\sqrt{d\log(1/\tau)}})$ and runs in $\poly(N)$ time. 
    \end{enumerate}
    \end{inftheorem}
    The details of the above key reductions appear in \Cref{sec:reductionScore}.
    
    Some remarks are in order.  The above result draws parallels with standard results reducing sample generation to score estimation. Item 1 in \Cref{infthm:reduction} requires bounded second moment and Lipschitz scores and can be seen as the analogue of the result of \citet{chen2023sampling,LeeLuTan23} in the context of density estimation\footnote{We note that Item 1 requires that the score $\nabla \log P$ (at time $t=0$) is $L$-Lipschitz. On the other side, the analogous reductions from generation to score estimation require the scores are Lipschitz across all timescales $t \in [0,T]$.}. In Item 2, $P$ is only assumed to have a bounded second moment. Since it can even be discrete, we can only get guarantees slightly before time 0. Hence, we provide PAC density estimation guarantees for the early stopped distribution $P_\tau$, which is also common in the sample generation literature.

    The reduction is efficient: given access to the score estimation oracle, the algorithm makes polynomially many calls to the oracle and runs in polynomial time. Moreover, at inference time, given any point $x_0$, the estimation of the log-density at $x_0$ takes polynomial time. 
    We remind the reader that the outputs $\wh P$ (and $\wh P_\tau$) may not integrate to 1. Hence, the expected value is not an upper bound for a KL divergence. 

    The proof of Item 1 {proceeds in} two steps. 
    The starting point is the likelihood identity of \Cref{lem:identity} which, roughly speaking (for large enough $T)$ says that at any point $x_0 \in \R^d$:
    \begin{equation}
        \textstyle 
        \phantom{.}\hspace{-10mm}
        -\log P(x_0) \approx \int_0^T \int \bigl\{\|\nabla \log P_{t}\|^2 - 2\,\langle \nabla \log P_{t},\nabla \log Q_{t|0}(\cdot \mid x_0)\rangle\bigr\}\,\d Q_{t|0}(\cdot\mid x_0) \,\d t + \mathrm{const}\,.
    \end{equation}
    Hence, if we could estimate the right-hand side of the above equation, ignoring the absolute constant term, we would be able to get an estimation for the log-density at $x_0.$ For details, we refer to \Cref{sec:computational:integrated_to_density}.
    We call the problem of estimating this double integral \emph{integrated (DDPM) score estimation} (see also \Cref{fig:main}). Converting score estimation to integrated score estimation is the most technical part of our reduction and appears in \Cref{ssec:score_to_integrated}. 
    The key technical ingredient for this reduction builds on the recent work of \citet{AltChe24SCI} and is likely of independent interest.
    Details about the proof appear in \Cref{ssec:score_sg}. 
    \begin{lemma}[Lipschitz score implies sub-Gaussian score]\label{lem:lip_score_implies_subG-main}
        Let $P$ be a distribution on $\R^d$ such that the score $\nabla \log P$ is $L$-Lipschitz.
        Then, for every $t\ge 0$, $\nabla \log P_t$ is $\sqrt{L_t}$-sub-Gaussian under $P_t$, where $L_t \le 2L$ is explicit (see~\Cref{lem:score_subG}).
    \end{lemma}
    The above discussion sketches the main steps of the reduction for Item 1: the condition that the score is Lipschitz is used for the transformation from score to integrated score estimation, while the bound on the second moment is used to employ the likelihood identity and convert the integrated score to a PAC density estimator. For the proof of Item 2, it suffices to apply Item 1 with $P$ equal to $P_\tau$, since its score is sub-Gaussian with parameter $1/(1-e^{-2\tau})$
    (see \Cref{sec:earlyStop}).

    \subsection{Discussion and open problems}\label{ssec:discussion}
    
        A key contribution of our work is to bridge the extensive literature on score estimation, which has traditionally focused on generative modeling, with the literature on density and parameter estimation. 
        We believe that there is ample room to more thoroughly explore this connection
        and in light of {this perspective}, we identify and leave several open problems.
        
        \begin{description}
            \item[\textbf{Open problem 1.}] 
                {Using score estimation, is it possible to output a proper density estimator, \ie{}, a function $\wh P$ that is non-negative and integrates to $1$?}
        \end{description}
        \noindent In general, it seems intractable to evaluate $\int\wh P$ to normalize our estimator, and even if we could compute the integral, we cannot show that it is close to $1$ (hence, normalization may destroy the PAC density estimation guarantee). For specific families, such as mixtures of Gaussians, it may be more feasible to post-process our estimator to output a legitimate density.

        \begin{description}
            \item[\textbf{Open problem 2.}] 
                Our work focuses on a particular generative process, namely the Ornstein--Uhlenbeck process over finite-dimensional Euclidean spaces. 
                Are there analogous results in other domains, for example, in discrete domains, infinite-dimensional spaces, or manifolds, when the noising process is suitably adapted to the domain?
        \end{description}
    \noindent It is well-established that generative modeling via score matching extends to other (\eg{}, discrete) domains and processes. Generalizing our results to such settings could yield new implications for methods such as estimation via pseudo-likelihood~\cite{KoeLeeVuo25DataInit}.

        \begin{description}
            \item[\textbf{Open problem 3.}] 
                Can score estimation lead to algorithms for PAC density estimation of well-conditioned Gaussian mixture models (in the sense of \citet{chen2024learninggeneralgaussianmixtures}) with minimum weight $\geq 1/\poly(k)$ that run in time singly exponential in $k$? 
        \end{description}
    \noindent Roughly speaking, the guarantee of~\citet{chen2024learninggeneralgaussianmixtures} is that learning the score of a well-conditioned (non-spherical) GMM with $k$ components to accuracy $\zeta$ can be done with $d^{\,\poly(k/\zeta)}$ samples and compute. Unfortunately, our reduction requires taking $\zeta \asymp \nfrac{\eps}{\sqrt d}$, which trivializes the runtime guarantee, so it seems that new ideas are needed.

        \begin{description}
            \item[\textbf{Open problem 4.}] 
                Can one boost the coverage probability $\delta$ in PAC density estimation?
        \end{description}
        \noindent A weakness of our reduction is that through the use of Markov's inequality in our $(\eps,\delta)$-PAC density estimation guarantee, $\eps$ scales polynomially with $1/\delta$, rather than with $\log(1/\delta)$. This could perhaps be mitigated by assuming access to a stronger score estimation oracle, or via a generic ``boosting'' procedure in analogy to classical learning theory. We believe the latter is unlikely to exist, but it would be useful to formalize this.

        \begin{description}
            \item[\textbf{Open problem 5.}] 
            We now know that both density estimation and score estimation are cryptographically hard for Gaussian mixtures with many components. Is it also cryptographically hard to learn a sampler?
        \end{description}
        \noindent More broadly, it would be interesting to prove computational hardness results for score estimation for other natural distributions, {utilizing} cryptographic tools {different} from Gaussian pancakes.

        \subsection{Notation}

        We focus on continuous distributions over $\R^d$ that are absolutely continuous with respect to the Lebesgue measure.
        Given a distribution $P$, for each point $x\in \R^d$, we abuse notation by using $P(x)$ to denote its Lebesgue density evaluated at $x$.
        We use standard definitions of distances and divergences between distributions.
        Namely, for two distributions $P$ and $Q$ over $\R^d$,
        {the total variation distance between $P$ and $Q$ is $\tv{P}{Q}\coloneqq  (1/2) \int\abs{\d P-\d Q}$,}
        the KL divergence of $P$ with respect to $Q$ is 
        $\KL(P \mmid Q) \coloneqq \int \log \frac{\d P}{\d Q}\,\d P$ (provided $P \ll Q$), and
        the $2$-Wasserstein distance between $P$ and $Q$ is $W_2(P, Q)=\inf_{\gamma \in \cC(P,Q)} \,(\int \|x-y\|^2\,\gamma(\d x, \d y))^{1/2}$, where the infimum is over the set $\cC(P,Q)$ of all couplings of $P$ and $Q$.
        We use $f \lesssim g$ to denote $f=O(g)$, $f\gtrsim g$ to denote $f=\Omega(g)$, and $f\asymp g$ to denote $f=\Theta(g)$.
        {We also use the notation $f = \widetilde O(g)$ to hide polylogarithmic factors, namely, $f = O(g \log^{O(1)} g)$.} We also use $\land$ and $\vee$ to denote $\min$ and $\max$ respectively. 

    We say that a random vector $X$ in $\R^d$ is $\sigma^2$-sub-Gaussian if for all vectors $v\in\R^d$, $\langle v, X\rangle$ is a $\sigma^2\,\|v\|^2$-sub-Gaussian random variable~\citep[see Definition 2.1 in][]{Wai19Stats}.
    {See \cref{sec:facts:subGaussian} for further discussion of sub-Gaussianity.}
        
    Given a probability measure $P$ over $\R^d$, we let $P_t$ denote the law at time $t$ of the Ornstein--Uhlenbeck {(OU)} process started at $P$, and $Q_{t|0}$ 
    the transition density of the OU process.

\section{Other related work}\label{ssec:related}

    \paragraph{Score estimation for generation.}
    {There is a vast literature on convergence guarantees for diffusion models, and here we provide a brief discussion on the implications for learning a sampler. The first works that obtained polynomial-time guarantees for generation from general data distributions are~\citet{chen2023sampling, LeeLuTan23}. These works assumed that the score functions are Lipschitz continuous uniformly in time and are learned accurately in $L^2$, and obtained TV guarantees. When the data distribution does not admit a Lipschitz score, one can still obtain TV guarantees for generating from an early stopped distribution with polynomial complexity. The subsequent works of~\citet{DBLP:conf/icml/ChenL023, Ben+24Diffusion, ConDurSil25Diffusion} sharpened the guarantees, replacing the assumption of Lipschitz scores with the assumption that the initial distribution has finite Fisher information relative to the Gaussian. The current state-of-the-art runtime guarantee is~\citet{LiYan25Diffusion}, although there have been extensions in numerous directions, \eg{}, deterministic samplers, low-dimensional adaptation, and parallelization, and we do not survey them all here.}
        
        \paragraph{Score estimation for parameter recovery.}
        Closest to our paper is the work of \citet{koehler2023statistical} that studied the implicit score-matching estimator, as discussed in the introduction.
        To the best of our knowledge, the first appearance of an objective such as~\Cref{def:ddpmMain} for the purpose of point estimation is the work of \citet{shah2023learning},
        {in the context of Gaussian mixture models, which also showed how to algorithmically minimize the DDPM objective (at carefully selected noise levels).}
        We are not aware of general statistical theory for $\DDPM$. Most works studying score estimation in DDPM instead considered estimating the score functions at different times separately (as opposed to $\DDPM$, which finds the value of the parameter that optimizes an objective using all of the scores).
        In particular, a line of work showed that score estimation can achieve minimax rates for density estimation; see~\Cref{sec:smoothDensityEstimation}.
        Finally, we note that variants of~\Cref{infthm} have appeared in the literature, \eg{},~\citet{song2021MLE} showed that the DDPM loss can be pointwise lower bounded in terms of the MLE loss, \citet{DBLP:conf/iclr/ChenLT22} proved an analogous result for the Schr\"odinger bridge, and \citet{li2024score} established essentially the same formula along a slightly different process.
        {Variants of \cref{infthm} also appeared in a line of works aiming at estimating partition functions \cite{doucet2022score,guo2025complexity} and is related to Jarzynski's equality from statistical physics \cite{jarzynski1997nonequilibrium,vaikuntanathan2008escorted,hartmann2017variational}.}

    \paragraph{Computational aspects of score estimation.}
        {Apart from statistical questions regarding score matching, there is increasing interest in \emph{computational aspects of score estimation.} In particular, \citet*{pabbaraju2023provableBenefits} gave an example of an exponential family of distributions such that implicit score matching is computationally efficient to optimize (\ie{}, finding $\wh \theta^{\,\mathrm{ISM}}_n$ can be done efficiently), and has a comparable statistical efficiency to MLE, while the MLE objective is intractable to optimize using gradient-based methods.} 
        {Meanwhile,~\citet{chen2024learninggeneralgaussianmixtures,gatmiry2024learning} used score estimation to establish new algorithmic results for generating samples from certain families of Gaussian mixtures.}
        
        In terms of lower bounds, 
        \citet{song2024cryptographic} reduced the problem of distinguishing Gaussian pancakes from a standard Gaussian to the problem of score estimation in the $L^2$-norm (with error $1/\sqrt{\log d}$ which is larger -- better -- than the error in our hardness result (\cref{infthm:scoreCrypto})). 
        Since the problem of distinguishing Gaussian pancakes from a Gaussian is cryptographically hard \cite{gupte2022continuous,bruna2021continuous}, this establishes the cryptographic hardness for $L^2$-score estimation for this specific family. 
        {Due to this and the fact that \citet{song2024cryptographic}'s result implies hardness for larger score estimation errors,} the above result is implied by the work of \citet{song2024cryptographic}, however, the techniques in our work and theirs are different: we obtain the result showing that PAC density estimation enables some goodness-of-fit tests and then employing our reduction to score estimation, while Song constructs the tests directly from approximate scores.

        \paragraph{Other related work.}
    We discuss further related works specific to each application in the respective sections, and we provide additional discussions in \cref{appendix:relatedwork}.

\section{Formal statements of main results and tools}\label{sec:framework}

    {In this section, we present the formal statements and proofs of the two tools crucial to our applications establishing connections between score estimation and different notions of distribution learning.}

\subsection{{Connection between log-likelihood and DDPM score estimation} (Lemma~\ref{lem:identity})}\label{sec:identity_pf}

    {First, we give a proof of  \Cref{lem:identity} that establishes a precise link between the log-likelihood and a certain integrated score matching objective.}
        
    Recall that given a probability measure $P$ over $\R^d$, we let $P_t$ denote the law of the Ornstein--Uhlenbeck {(OU)} process started {from $P_0=P$} at time $t \in  [0,T]$.
    We further denote by $Q_{t|0}$ the transition density of {at time $t$}.

\begin{proof}[Proof of \cref{lem:identity}]
    Let ${(B_t)}_{t\ge 0}$ be standard Brownian motion and let ${(X_t)}_{t\ge 0}$ denote the OU process started at $X_0 = x_0$.
    By parabolic regularity (or direct computation with the OU semigroup), the mapping $(t,x) \mapsto P_t(x)$ is strictly positive and smooth on $\R_{>0} \times \R^d$, with $P_t \to P$ pointwise as $t\searrow 0$.
    Therefore, the Fokker{--}Planck equation implies
    \begin{align*}
        \partial_t \log P_{t}
        &= \frac{\Delta P_{t} + \dive(P_{t}\, x_t)}{P_{t}}
        = \Delta \log P_{t} + \|\nabla \log P_{t}\|^2 + d + \langle \nabla \log P_{t}, x_t\rangle\,.
    \end{align*}
    By It\^o's formula,
    \begin{align*}
        &\d \log P_{t}(X_t) \\
        &\qquad = \bigl\{\partial_t \log P_{t}(X_t) - \langle \nabla \log P_{t}(X_t), X_t\rangle + \Delta \log P_{t}(X_t) \bigr\}\,\d t +\sqrt 2\, \langle \nabla \log P_{t}(X_t),\d B_t\rangle \\
        &\qquad = \bigl\{ \|\nabla \log P_{t}(X_t)\|^2 + 2\,\Delta\log P_{t}(X_t) + d\bigr\}\,\d t+\sqrt 2\, \langle \nabla \log P_{t}(X_t),\d B_t\rangle\,.
    \end{align*}
    Integrating over time and taking expectations, for $\varepsilon > 0$,
    \begin{align*}
        \E\bigl[\log P_T(X_T) - \log P_\varepsilon(X_\varepsilon)\bigr] = d\,(T-\varepsilon) + \int_\varepsilon^T \E\bigl\{ \|\nabla \log P_{t}(X_t)\|^2 + 2\,\Delta\log P_{t}(X_t)\bigr\}\,\d t\,,
    \end{align*}
    where we used the fact that {${\{\int_\varepsilon^t \langle \nabla \log P_s(X_s),\d B_s\rangle\}}_{t\in [\varepsilon,T]}$ is a martingale which, in turn, can be deduced because $\E[\norm{\nabla \log P_t(X_t)}^2] = O(\nfrac{1}{t^2})$~\citep[cf.][]{OttVil01Comment}.}
 On the other hand, for any $t > 0$, we note that 
    \begin{align*}
        \int \langle \nabla \log P_{t}(x_t), \nabla \log Q_{t|0}(x_t \mid x_0)\rangle \,Q_{t|0}(\d x_t \mid x_0)
        &= \int \langle \nabla \log P_{t}(x_t), \nabla Q_{t|0}(x_t \mid x_0)\rangle \,\d x_t \\
        &= -\int \Delta \log P_{t}(x_t)\, Q_{t|0}(\d x_t \mid x_0)\,.
    \end{align*}
    Substituting this in and taking $\varepsilon\searrow 0$ completes the proof.
\end{proof}

\subsection{{Score estimation implies PAC density estimation}}\label{sec:reductionScore}

    A long line of works shows that, under minimal regularity assumptions, access to a score estimation oracle is sufficient for learning to sample; see~\Cref{appendix:relatedwork} for a discussion of {this} literature.
    However, the precise connection between score estimation and density estimation remains elusive.    
    {Next,} we prove that access to a score estimation oracle is sufficient for PAC density estimation, in the sense of~\Cref{def:PACdensityEstimation}, {under essentially the weakest regularity assumptions as required for generation obtained by the above line of works}.

    \subsubsection{Relevant oracles}

        We begin by introducing the two oracles relevant to our reduction.
        The first oracle formalizes our notion of score estimation.
   
    \begin{definition}
    [Score estimation oracle]\label{def:sxc}\label{def:scoreEstimation}
        A score estimation oracle for a density $P$ on $\R^d$ is a primitive that receives as inputs a time $t\geq 0$ and a point $x_t\in\R^d$, and outputs $\hat s_t(x_t) \in\R^d$.
        {The error of the oracle with early stopping $\tau > 0$ {and terminal time $T$} is defined to be $\eps_*^2 \deq \int_\tau^T \varepsilon_t^2\,\d t$, where for each $t\in [\tau,T]$,}
        \[
            \int 
                \norm{
                    \wh{s}_t(x_t) - \nabla\log{P_t}(x_t)
                }^2
                P_t(\d x_t)
                \leq 
                \eps_t^2
                \,.
        \]
    \end{definition}
    {This definition only requires  good score estimation for times bounded away from 0 (\ie{}, $t\geq \tau$).}
   {The latter is particularly useful since} the regularity of the score function typically degrades as $t\searrow 0$, so it becomes more difficult to estimate the score at small times.
    {Indeed, early} stopping is a commonly used device in the literature to circumvent this issue~\citep[see, \eg{},][]{chen2023sampling}.
    {This weakening only makes the  oracle easier to implement.}
    
    Our reduction from PAC density estimation {to score estimation} passes through the following intermediate oracle.
    
    \begin{definition}[Integrated score estimation oracle]\label{def:integratedOracle}
        An integrated score estimation oracle for a density $P$ on $\R^d$ is a primitive that receives as inputs a point $x_0 \in \R^d$ and a terminal time $T$, and outputs a {(possibly random)} value $\hat v(x_0)\in\R$. 
        {The oracle is said to have error $\eps$ if}
        \begin{align*}
            \int \E|\wh v(x_0) - v(x_0)|\,P(\d x_0) \le \varepsilon\,,
        \end{align*}
        where $\E$ is over the randomness of $\hat v$ and
        \[
            v(x_0)\deq \int_{0}^T\int 
                \inbrace{
                    \norm{\nabla \log{P_t}(x_t)}^2
                    - 
                    2\inangle{
                        \nabla{\log{P_t}(x_t)}, 
                        \nabla{\log{Q_{t|0}(x_t \mid x_0)}}
                    }
                }\, Q_{t|0}(\d x_t\mid x_0)\, \d t\,.
        \]
    \end{definition}
    The motivation for this oracle comes from \Cref{lem:identity}, which implies that for any point $x_0$, the output $\wh v(x_0)$ of the integrated score estimation oracle is close to the target log-density $-\log P(x_0)$.
    {Note that this oracle only requires a bound on the average error across the draw of the initial sample $x_0\sim P$.}
    {In contrast to the score estimation oracle, the above oracle does not allow early stopping.
    Nevertheless, we will show that under a mild regularity assumption (see \cref{ssec:score_sg}), the score estimation oracle \textit{with} early stopping is sufficient to implement the above integrated oracle.}
   
    \paragraph{Outline of this section.}
    First, in~\Cref{ssec:score_sg}, {we prove a result about the sub-Gaussianity of the score along the OU process. This result is a key technical ingredient in our subsequent reduction and likely of independent interest.}
    Next, in \cref{ssec:score_to_integrated}, we show that under mild assumptions on $P$, the score estimation oracle can be efficiently transformed to an integrated score estimation oracle (\cref{thm:scoreEstimationtoIntegrated}) {in polynomial time with polynomially many calls to the score estimation oracle}. 
    Then, in \cref{sec:computational:integrated_to_density}, we show how to transform the integrated score estimation oracle to a PAC density estimation oracle using \cref{lem:identity}.

\subsubsection{Sub-Gaussianity of the score}\label{ssec:score_sg}

    The key technical ingredient in the subsequent reduction is the following lemma, which ensures that the score function remains sub-Gaussian along the OU process.
    It builds upon diffusion estimates recently developed in~\citet{AltChe26SCII}.
    
\begin{lemma}[Sub-Gaussianity of the score]\label{lem:score_subG}
    Assume that for $X_0 \sim P_0$, the score $\nabla \log P(X_0)$ is $\sqrt L$-sub-Gaussian.
    Then, for all $t \ge 0$ and $X_t \sim P_t$, $\nabla \log P_t(X_t)$ is $\sqrt{L_t}$-sub-Gaussian, where
    \begin{align*}
        L_t \deq \min\Bigl\{L\exp(2t),\, \frac{1}{1-\exp(-2t)}\Bigr\}
        \le 2\, (1\vee L)\,.
    \end{align*}
\end{lemma}
\begin{proof}
    The main fact that we use is that by~\citet[Theorem 5]{AltChe26SCII}, the score of $P$ is $\sqrt L$-sub-Gaussian under $P$ if and only if $P$ satisfies the local gradient-entropy (LGE) inequality with constant $L$:
    \begin{align*}
        \frac{\|\int \nabla f\,\d P\|^2}{\int f\,\d P} \le 2L\, \Ent_P(f) \qquad\text{for all smooth, compactly supported}~f\colon \R^d\to\R\,.
    \end{align*}
    Actually,~\citet[Theorem 5]{AltChe26SCII} only states one direction of this implication (namely, LGE implies sub-Gaussianity of the score), but one can see from the proof that it is an equivalence.
    In light of this, our goal is, therefore, to verify that $P_t$ satisfies the LGE inequality 
    with parameter $L_t$, where
    \[
    L_t = \min \Bigl\{ L \exp(2t),\, \frac{1}{1-\exp(-2t)} \Bigr\}\,.
    \]

    We combine two different bounds.
    The first bound is effective for small $t$.
    Let $f$ be such that $\int f\,\d P_t = 1$ and let ${(Q_t)}_{t\ge 0}$ denote the OU semigroup.
    Then,
    \begin{align*}
        \int \nabla f\,\d P_t
        &= \int Q_t\nabla f\,\d P
        = \exp(t)\int \nabla Q_t f\,\d P\,.
    \end{align*}
    By applying the LGE inequality for $P$, since $\int Q_t f\,\d P = \int f\,\d P_t = 1$,
    \begin{align*}
        \Bigl\lVert \int \nabla f\,\d P_t\Bigr\rVert^2
        &\le 2L\exp(2t) \,\Ent_P(Q_t f)
        \le 2L\exp(2t) \,\Ent_{P_t}(f)\,,
    \end{align*}
    where the last inequality follows from the entropy decomposition
    \begin{align*}
        \Ent_{P_t}(f)
        &= \Ent_P(Q_t f) + \int \Ent_{Q_{t|0}(\cdot \mid x_0)}(f)\,P(\d x_0)\,.
    \end{align*}
    Next, we consider a bound for large $t$.
    Since $(a,b) \mapsto \|a\|^2/b$ is jointly convex,
    \begin{align*}
        \frac{\|\int \nabla f\,\d P_t\|^2}{\int f\,\d P_t}
        &\le \int \frac{\|Q_t \nabla f\|^2}{Q_t f}\,\d P
        \le \frac{2}{1-\exp(-2t)} \int \Ent_{Q_{t|0}(\cdot \mid x_0)}(f)\,P(\d x_0) \\[0.25em]
        &\le \frac{2}{1-\exp(-2t)} \,\Ent_{P_t}(f)\,,
    \end{align*}
    where we applied the LGE inequality along the OU semigroup~\citet[see Theorem 4 in][]{AltChe26SCII}.

    Putting the two cases together, we have shown that
    \begin{align*}
        \frac{\|\int \nabla f\,\d P_t\|^2}{\int f\,\d P_t}
        &\le 2\min\Bigl\{L\exp(2t),\, \frac{1}{1-\exp(-2t)}\Bigr\}\,\Ent_{P_t}(f)
        \le 4L\,\Ent_{P_t}(f)\,. \qedhere
    \end{align*}
\end{proof}

In order to apply~\Cref{lem:score_subG} for our purposes, we must verify that the initial score is sub-Gaussian.
Toward that end, we provide two tools for checking this assumption.

\begin{lemma}[Lipschitz score implies sub-Gaussian score]\label{lem:lip_score_implies_subG}
    Let $P$ be a probability distribution over $\R^d$ such that the score $\nabla \log P$ is $L$-Lipschitz.
    Then, $\nabla \log P$ is $\sqrt L$-sub-Gaussian under $P$.
\end{lemma}
\begin{proof}
    This fact was established in~\citet{Neg22Thesis}; the simple argument is reproduced as~\citet[Remark 20]{AltChe26SCII}.
    Alternatively, it follows from Corollary 19 and Theorem 4 in \citet{AltChe26SCII}.
\end{proof}

\begin{lemma}[Score of a mixture]\label{lem:score_mixture}
    Let $\mu$ be a probability measure over a space $\cX$, and let $P$ be a Markov kernel from $\cX$ to $\R^d$.
    Let $X\sim \mu$ and conditionally on $X$, let $Y \sim P(X,\cdot)$; thus, the marginal law of $Y$ is the mixture $\mu P$.
    Then, the score of $\mu P$ can be expressed as
    \begin{align*}
        \nabla \log \mu P(y)
        &= \E[\nabla \log P(X,\cdot) \mid Y=y]\,.
    \end{align*}
\end{lemma}
\begin{proof}
    Since
    \begin{align*}
        \mu P(y)
        &= \int P(x,y)\,\mu(\d x)\,,
    \end{align*}
    then
    \begin{align*}
        \nabla \log \mu P(y)
        &= \frac{\int \nabla_y \log P(x,y)\,P(x,y)\,\mu(\d x)}{\int P(x,y)\,\mu(\d x)}
        = \E[\nabla \log P(X,\cdot) \mid Y=y]\,. \qedhere
    \end{align*}
\end{proof}

\begin{lemma}[Sub-Gaussianity of the score of a mixture]\label{lem:subG_score_mixture}
    In the setting of~\Cref{lem:score_mixture}, suppose that for each $x\in\cX$, $\nabla \log P(x,\cdot)$ is $\sigma^2$-sub-Gaussian under $P(x,\cdot)$.
    Then, $\nabla \log \mu P$ is also $\sigma^2$-sub-Gaussian under $\mu P$.
\end{lemma}
\begin{proof}
    For any vector $v\in\R^d$, Jensen's inequality implies
    \begin{align*}
        \E\exp{\langle v, \nabla \log \mu P(Y)\rangle}
        = \E\exp{\langle v, \E[\nabla \log P(X,Y) \mid Y]\rangle}
        &\le \E \exp{\langle v, \nabla \log P(X,Y)\rangle} \\
        &\le \exp \frac{\sigma^2\,\|v\|^2}{2}\,,
    \end{align*}
    where the last inequality follows by first conditioning on $X$.
\end{proof}

For example, we use these facts to verify that {Gaussian location {mixture}s} satisfy the assumptions for our reduction (\Cref{fact:GLM:satisfiesLBMAssumptions}).

\subsubsection{Score estimation implies integrated score estimation}\label{ssec:score_to_integrated}

We are now ready to show that there is a polynomial-time reduction from score estimation to integrated score estimation whenever the distribution $P$ has a sub-Gaussian score.

\begin{assumption}\label{asmp:logLipschitz}
    There is a constant $L \ge 1$ such that the distribution $P$ over $\R^d$ has a score function $\nabla \log P$ which is $\sqrt L$-sub-Gaussian under $P$.
\end{assumption}
The main result of this section is the following.

\begin{theorem}
[Score estimation implies integrated score estimation]\label{thm:scoreEstimationtoIntegrated}
    Let $P$ be a distribution on $\R^d$ that satisfies \cref{asmp:logLipschitz} with parameter $L$.
    There is an algorithm that, given  accuracy $\eps \in (0,1)$, constant $L$, terminal time $T \geq 1$, and query access to a score estimation oracle for $P$ with early stopping parameter $\tau$,
    implements an $(\eps,T)$-integrated score estimation oracle for $P$.
    The algorithm makes 
    $N$ calls to the score estimation oracle with accuracy $\eps_*$ for
    \[
        N = \wt{O}\Bigl(
            \frac{(L+T)\, Td^2}{\eps^2}
        \Bigr)
        \qquadand
        \eps_* = \wt{O}\Bigl(\frac{\eps}{\sqrt{d\,(\log L + T)}}\Bigr)\,,
    \]
    and the early stopping parameter $\tau$ of the score estimation oracle is required to satisfy $\tau \lesssim \nicefrac{\varepsilon^2}{Ld^2}$.
\end{theorem}
Recall that the integrated score oracle is a primitive that aims at estimating the values 
\[
            v(x_0)\deq \int_{0}^T\int 
                \inbrace{
                    \norm{\nabla \log{P_t}(x_t)}^2
                    - 
                    2\inangle{
                        \nabla{\log{P_t}(x_t)}, 
                        \nabla{\log{Q_{t|0}(x_t \mid x_0)}}
                    }
                }\, Q_{t|0}(\d x_t\mid x_0)\, \d t
        \]
in expectation over $x_0 \sim P$.
We are now ready to define the (randomized) output {$\wh v(x_0)$} of the integrated score estimation oracle on input $x_0 \in \R^d$:
\[
    \wh v(x_0) \deq \frac{T-\tau}{m} \sum_{i \in [m]} \left\{ 
\|\wh s_{t_i}(x^i_{t_i})\|^2 - 2\, \<\wh s_{t_i}(x^i_{t_i}), \nabla \log Q_{t_i | 0}(x^i_{t_i} \mid x_0)\>
\right\}\,.
\]
where for $i=1,\dotsc,m$, the pairs $(t_i, x_{t_i}^i)$ are i.i.d.\ and drawn as follows: first, $t_i \sim \unif([\tau,T])$, and then, conditionally on $t_i$, $x_{t_i}^i \sim Q_{t_i|0}(\cdot \mid x_0)$.
Here, we set $\tau \asymp \nicefrac{\varepsilon^2}{Ld^2}$; even if the score estimation oracle provides score estimates for smaller times, we do not use them.

\textcolor{black}{\begin{remark}
Before proceeding to the proof of \Cref{thm:scoreEstimationtoIntegrated}, we mention that even getting accuracy $\epsilon > 1$ for the integrated score oracle is non-trivial. To see this observe that, using the likelihood identity, the integrated score oracle with accuracy $\epsilon$ allows us to find a model $\wh P$ that achieves $\E_{x \sim P}|\wh P - \log P(x)| \lesssim \epsilon$, which as we have already mentioned is very similar to the KL divergence loss and hence getting bounds of the order, e.g., $o(d)$ is non-trivial. Hence, we mention that any error $\epsilon_* > 0$ on the score estimation oracle implies an error to the integrated score oracle of the form:
\begin{align*}
\int \E|\wh v(x_0) - v(x_0)|\,P(\d x_0)
    &\lesssim {\underbrace{Ld\tau + \sqrt Ld\sqrt\tau}_{\text{early stopping}}} + {\underbrace{\varepsilon_*^2 + \varepsilon_*\sqrt{d\,\bigl(T + \log\frac{1}{\tau} \bigr)}}_{\text{score error}}} \\
    &\qquad{} + {\underbrace{\frac{d\sqrt{T\,(T+L\log\nfrac{1}{L\tau})}}{\sqrt m}}_{\text{sampling error}}}\,,
\end{align*}
where $v,\wh v$ are defined as above. This is similar to standard bounds in the KL divergence for the DDPM process~\citep[see, e.g.,][]{Ben+24Diffusion}. The rates of \Cref{thm:scoreEstimationtoIntegrated} follow by making the above RHS equal to $\epsilon$ but the RHS is non-trivial even if it is larger than 1.
\end{remark}}

\begin{proof}[Proof of~\Cref{thm:scoreEstimationtoIntegrated}]
Our goal is to control the error
\begin{align*}
    \int \E|\wh v(x_0) - v(x_0)|\,P(\d x_0)\,,
\end{align*}
where $\E$ denotes the expectation over the randomness of $\wh v$.
Throughout the proof, we repeatedly use the sub-Gaussianity of the score (\Cref{lem:score_subG}).

The first step is to bound this error by (I) $+$ (II) $+$ (III), where
\begin{align*}
    ({\rm I})
    &\deq \frac{T-\tau}{m} \int \E\Bigl\lvert \sum_{i\in [m]} \bigl\{\|\wh s_{t_i}(x_{t_i}^i)\|^2 - \|\nabla \log P_{t_i}(x_{t_i}^i)\|^2 \bigr\}\Bigr\rvert \, P(\d x_0)\,, \\[0.25em]
    ({\rm II})
    &\deq \frac{2\,(T-\tau)}{m} \int \E\Bigl\lvert \sum_{i\in [m]} \langle \wh s_{t_i}(x_{t_i}^i)- \nabla \log P_{t_i}(x_{t_i}^i), \nabla \log Q_{t_i | 0}(x^i_{t_i} \mid x_0)\rangle \Bigr\rvert \, P(\d x_0)\,, \\[0.25em]
    ({\rm III})
    &\deq \int \E|({\rm III}_0)(x_0)|\,P(\d x_0)\,, \\[0.25em]
    ({\rm III}_0)(x_0)
    &\deq \frac{T-\tau}{m} \sum_{i \in [m]} \left\{ 
\|\nabla \log P_{t_i}(x^i_{t_i})\|^2 - 2\, \<\nabla \log P_{t_i}(x^i_{t_i}), \nabla \log Q_{t_i | 0}(x^i_{t_i} \mid x_0)\>
\right\} \\
&\qquad{} - v(x_0)\,.
\end{align*}
We also define the quantities
\begin{align*}
    L_* \deq \int_\tau^T L_t\,\d t\,, \qquad L_{*,2} \deq \Bigl(\int_\tau^T L_t^2\,\d t\Bigr)^{1/2}\,, \qquad L_{*,3} \deq \int_\tau^T \frac{L_t}{1-e^{-2t}}\,\d t\,,
\end{align*}
where $L_t$ is the constant from~\Cref{lem:score_subG}.

\paragraph{Control of term I\@.} We start by controlling term I. The error of term I relies on how well the squared norm of the score oracle approximates the squared norm of the actual score.

\begin{claim}[Controlling term (I)]\label{claim:termI}
Let $\varepsilon_*^2 = \int_\tau^T \varepsilon_t^2\, \d t$.
    It holds that 
    \[
    \frac{T-\tau}{m} \int \E\Bigl\lvert \sum_{i\in [m]} \bigl\{\|\wh s_{t_i}(x_{t_i}^i)\|^2 - \|\nabla \log P_{t_i}(x_{t_i}^i)\|^2 \bigr\}\Bigr\rvert \, P(\d x_0) 
    \lesssim \varepsilon_*^2 + \sqrt{L_* d}\,\varepsilon_*\,.
    \]
\end{claim}
\begin{proof}[Proof of \Cref{claim:termI}]
We can bound term I by
\begin{align*}
    ({\rm I})
    &\le \frac{T-\tau}{m} \int \E \sum_{i\in [m]} \|\wh s_{t_i}(x_{t_i}^i)-\nabla \log P_{t_i}(x_{t_i}^i)\|\, \bigl\lvert \|\wh s_{t_i}(x_{t_i}^i)\| + \|\nabla \log P_{t_i}(x_{t_i}^i)\|\bigr\rvert \, P(\d x_0) \\[0.25em]
    &\le \frac{T-\tau}{m} \int \E \sum_{i\in [m]} \|\wh s_{t_i}(x_{t_i}^i)-\nabla \log P_{t_i}(x_{t_i}^i)\|^2 \, P(\d x_0) \\[0.25em]
    &\qquad{} + \frac{2\,(T-\tau)}{m}\int \E \sum_{i\in [m]} \|\wh s_{t_i}(x_{t_i}^i)-\nabla \log P_{t_i}(x_{t_i}^i)\| \,\|\nabla \log P_{t_i}(x_{t_i}^i)\| \, P(\d x_0) \\[0.25em]
    &\le \frac{T-\tau}{m}\, \E \sum_{i\in [m]} \varepsilon_{t_i}^2
    + \frac{2\,(T-\tau)}{m} \,\E\sum_{i\in [m]} \varepsilon_{t_i} \,\Bigl(\int\|\nabla \log P_{t_i}(x_{t_i}^i)\|^2\,P_{t_i}(\d x_{t_i}^i)\Bigr)^{1/2} \\[0.25em]
    &\lesssim \varepsilon_*^2 + \frac{(T-\tau)\,\sqrt d}{m} \,\E\sum_{i\in [m]} \varepsilon_{t_i}\sqrt{L_{t_i}}
    \lesssim \varepsilon_*^2 + \sqrt{L_* d}\,\varepsilon_*\,.
\end{align*}
The second inequality follows by observing that $\|a-b\| \cdot | \|a\| - \|b\| + 2\|b\| | \leq \|a-b\|^2 + 2\,\|a-b\| \, \|b\|$ for any vectors $a,b$, the third inequality follows by Cauchy--Schwarz and the property of the score estimation oracle at times $\{t_i\}_{i \in [m]}$, and the fourth inequality follows by sub-Gaussianity of the score and the definition of the integrated error $\varepsilon_*.$
\end{proof}

\paragraph{Control of term II\@.}
Similarly, we can control term II, which involves again a difference between the score oracle and the actual score function. In contrast to term I which contained the difference of the norms, Term II (roughly speaking) involves the difference of the two vectors in the direction of the associated OU process.
\begin{claim}[Controlling term (II)]\label{claim:termII}
    Let $\varepsilon_*^2 = \int_\tau^T \varepsilon_t^2\, \d t.$
    It holds that 
    \begin{align*}
    &\frac{2\,(T-\tau)}{m} \int \E\Bigl\lvert \sum_{i\in [m]} \langle \wh s_{t_i}(x_{t_i}^i)- \nabla \log P_{t_i}(x_{t_i}^i), \nabla \log Q_{t_i | 0}(x^i_{t_i} \mid x_0)\rangle \Bigr\rvert \, P(\d x_0) \\
    &\qquad \lesssim 
    \varepsilon_* \sqrt{d\,\bigl(T+\log\frac{1}{\tau}\bigr)}\,.
    \end{align*}
\end{claim}
\begin{proof}
    [Proof of \Cref{claim:termII}] We have the following for Term II:
\begin{align*}
    ({\rm II})
    &\le \frac{2\,(T-\tau)}{m} \sum_{i\in [m]} \int \E\bigl[\|\wh s_{t_i}(x_{t_i}^i)- \nabla \log P_{t_i}(x_{t_i}^i)\|\,\|\nabla \log Q_{t_i | 0}(x^i_{t_i} \mid x_0) \|\bigr]\,P(\d x_0) \\[0.25em]
    &\lesssim (T-\tau) \, \Bigl( \int \E\bigl[\|\wh s_{t_1}(x_{t_1}^1)- \nabla \log P_{t_1}(x_{t_1}^1)\|^2\bigr]\,P(\d x_0)\Bigr)^{1/2} \\
    &\qquad{}\times \Bigl(\int \E\bigl[\|\nabla \log Q_{t_1 | 0}(x^1_{t_1} \mid x_0) \|^2\bigr]\,P(\d x_0)\Bigr)^{1/2} \\[0.25em]
    &\le (T-\tau)\,\Bigl( \E[\varepsilon_{t_1}^2] \,\E\bigl[\frac{d}{1-\exp(-2t_1)}\bigr] \Bigr)^{1/2}
    = \Bigl( \int_\tau^T \varepsilon_t^2\,\d t \int_\tau^T \frac{d}{1-\exp(-2t)}\,\d t \Bigr)^{1/2} \\[0.25em]
    &\lesssim \varepsilon_* \sqrt{d\,\bigl(T+\log\frac{1}{\tau}\bigr)}\,.
\end{align*}
The second inequality follows by Cauchy--Schwarz and the fact that we use i.i.d.\ samples and the third inequality uses the property of the score estimation oracle at time $t_1$ and the sub-Gaussianity of the OU process.
\end{proof}

\paragraph{Control of term III\@.} The last step is to control term (III). 
Recall the quantities
\begin{align*}
    L_{*,2} \deq \Bigl(\int_\tau^T L_t^2\,\d t\Bigr)^{1/2}\,, \qquad L_{*,3} \deq \int_\tau^T \frac{L_t}{1-e^{-2t}}\,\d t\,,
\end{align*}
where $L_t$ is the constant from~\Cref{lem:score_subG}.
\begin{claim}[Controlling term (III)]\label{claim:termIII}
    It holds that
    \begin{align*}
         &\int \E\Bigl\lvert\frac{T-\tau}{m} \sum_{i \in [m]} \bigl\{ 
\|\nabla \log P_{t_i}(x^i_{t_i})\|^2 \\
&\qquad\qquad\qquad{} - 2\, \<\nabla \log P_{t_i}(x^i_{t_i}), \nabla \log Q_{t_i | 0}(x^i_{t_i} \mid x_0)\>
\bigr\}
- v(x_0)\Bigr\rvert\, P(\d x_0) \\
&\qquad  \lesssim \underbrace{Ld \tau + \sqrt{L} d \sqrt{\tau}}_{\mathrm{early~stopping}} ~ + ~ \underbrace{\frac{L_{*,2}\sqrt T d + d\sqrt{L_{*,3}T}}{\sqrt m}}_{\mathrm{variance}}\,.
    \end{align*}
\end{claim}
For term (III), we perform a bias-variance decomposition. First, the estimator in the above expression is biased since it only integrates from time $\tau$ to $T$, while $v(x_0)$ integrates from time 0. Hence, the bias term corresponds to the error due to early stopping. On the other side, we also have to deal with the variance of the estimator.

\begin{proof}[Proof of \Cref{claim:termIII}]
We bound term (III) by (IV) + (V), where
\begin{align*}
    ({\rm IV})
    &= \int |\E ({\rm III}_0)(x_0)|\,P(\d x_0) \\[0.25em]
    &= \int \Bigl\lvert \int_0^\tau \int \bigl\{\|\nabla \log P_t(x_t)\|^2 \\
    &\qquad\qquad\qquad{} - 2\,\langle \nabla \log P_t(x_t),\nabla \log Q_{t|0}(x_t \mid x_0)\rangle\bigr\}\,Q_{t|0}(\d x_t\mid x_0)\,\d t \Bigr\rvert \,P(\d x_0)\,,\\[0.25em]
    ({\rm V})
    &\le \frac{T-\tau}{\sqrt m} \int\sqrt{\Var\bigl(\|\nabla \log P_{t_1}(x_{t_1}^1)\|^2 - 2\,\langle \nabla \log P_{t_1}(x_{t_1}^1), \nabla \log Q_{t|0}(x_{t_1}^1 \mid x_0)\rangle\bigr)}\,P(\d x_0)\,.
\end{align*}
Note that term (IV) is the early stopping error, which we control as follows:
\begin{align*}
    ({\rm IV})
    &\le \int_0^\tau \iint \bigl\{\|\nabla \log P_t(x_t)\|^2 \\
    &\qquad\qquad\qquad{} + 2\,\|\nabla \log P_t(x_t)\|\,\|\nabla \log Q_{t|0}(x_t \mid x_0)\|\bigr\}\,Q_{t|0}(\d x_t\mid x_0) \,P(\d x_0)\,\d t \\[0.25em]
    &\lesssim \int_0^\tau \bigl\{Ld + \frac{\sqrt Ld}{\sqrt{1-\exp(-2t)}}\bigr\}\,\d t
    \lesssim Ld\tau + \sqrt Ld\sqrt\tau\,.
\end{align*}
Note that the second inequality follows immediately by noting that $\|\nabla \log P_t\|^2$ is $\sqrt{2Ld}$-sub-Gaussian and that the OU process $\| \nabla \log Q_{t|0}\|$ contributes another $\sqrt{d}/\sqrt{1-e^{-2t}}$ factor.

Finally, for the variance term (V), by the triangle inequality, we bound it by (VI) + (VII), where
\begin{align*}
    ({\rm VI})
    &= \frac{T-\tau}{\sqrt m} \int \sqrt{\Var\bigl(\|\nabla \log P_{t_1}(x_{t_1}^1)\|^2\bigr)}\,P(\d x_0)\,, \\[0.25em]
    ({\rm VII})
    &= \frac{2\,(T-\tau)}{\sqrt m} \int \sqrt{\Var\bigl(\langle \nabla \log P_{t_1}(x_{t_1}^1), \nabla \log Q_{t_1|0}(x_{t_1}^1 \mid x_0)\rangle\bigr)}\,P(\d x_0)\,.
\end{align*}
By sub-Gaussianity of the score,
\begin{align*}
    ({\rm VI})
    &\le \frac{T-\tau}{\sqrt m} \sqrt{\int \E[\|\nabla \log P_{t_1}(x_{t_1}^1)\|^4]\,P(\d x_0)}
    \lesssim \frac{(T-\tau)\,d}{\sqrt m} \sqrt{\E[L_{t_1}^2]}
    \lesssim \frac{L_{*,2}\sqrt T d}{\sqrt m}\,.
\end{align*}
For the last term,
\begin{align*}
    ({\rm VII)}
    &\le \frac{2\,(T-\tau)}{\sqrt m} \int \sqrt{\E\bigl[\|\nabla \log P_{t_1}(x_{t_1}^1)\|^2 \,\|\nabla \log Q_{t|0}(x_{t_1}^1 \mid x_0)\|^2\bigr]}\,P(\d x_0) \\[0.25em]
    &\le \frac{2\,(T-\tau)}{\sqrt m} \sqrt{\int \E\bigl[\|\nabla \log P_{t_1}(x_{t_1}^1)\|^2 \,\|\nabla \log Q_{t|0}(x_{t_1}^1 \mid x_0)\|^2\bigr]\,P(\d x_0)} \\[0.25em]
    &\le \frac{2\,(T-\tau)}{\sqrt m} \,\biggl(\int \E\Bigl[\sqrt{\int \|\nabla \log P_{t_1}(x_{t_1})\|^4\,Q_{t_1|0}(\d x_{t_1} \mid x_0)} \\[0.25em]
    &\qquad\qquad\qquad{}\times \sqrt{\int \|\nabla \log Q_{t_1|0}(x_{t_1}\mid x_0)\|^4\,Q_{t_1|0}(\d x_{t_1} \mid x_0)}\Bigr]\,P(\d x_0)\biggr)^{1/2} \\[0.25em]
    &\lesssim \frac{T-\tau}{\sqrt m} \,\Bigl(\E\Bigl[\frac{d}{1-\exp(-2t_1)} \int\sqrt{\int \|\nabla \log P_{t_1}(x_{t_1})\|^4\,Q_{t_1|0}(\d x_{t_1} \mid x_0)}\,P(\d x_0)\Bigr] \Bigr)^{1/2} \\[0.25em]
    &\le \frac{T-\tau}{\sqrt m} \,\Bigl(\E\Bigl[\frac{d}{1-\exp(-2t_1)} \sqrt{\iint \|\nabla \log P_{t_1}(x_{t_1})\|^4\,Q_{t_1|0}(\d x_{t_1} \mid x_0)\,P(\d x_0)}\Bigr] \Bigr)^{1/2} \\[0.25em]
    &\lesssim \frac{(T-\tau)\,d}{\sqrt m} \,\Bigl(\E\Bigl[\frac{L_{t_1}}{1-\exp(-2t_1)} \Bigr] \Bigr)^{1/2}
    \lesssim \frac{d\sqrt{L_{*,3}T}}{\sqrt m}\,.
\end{align*}
In the above, the second and the fifth inequality use that $\E \sqrt{X} \leq \sqrt{\E X}$.
The third inequality follows by Cauchy--Schwarz, the fourth by properties of the OU process, and the sixth one by sub-Gaussianity of the score.
\end{proof}

\paragraph{Putting everything together.}
We now apply the integral estimates from~\Cref{lem:integrals}, combine the bounds of \Cref{claim:termI,claim:termII,claim:termIII}, and simplify to obtain
\begin{align*}
    \int \E|\wh v(x_0) - v(x_0)|\,P(\d x_0)
    &\lesssim {\underbrace{Ld\tau + \sqrt Ld\sqrt\tau}_{\text{early stopping}}} + {\underbrace{\varepsilon_*^2 + \varepsilon_*\sqrt{d\,\bigl(T + \log\frac{1}{\tau} \bigr)}}_{\text{score error}}} \\
    &\qquad{} + {\underbrace{\frac{d\sqrt{T\,(T+L\log\nfrac{1}{L\tau})}}{\sqrt m}}_{\text{sampling error}}}\,.
\end{align*}
To make the overall error at most $\varepsilon$, it suffices to take
\begin{align*}
    \tau &\asymp \frac{\varepsilon^2}{Ld^2}\,, \qquad \varepsilon_* \lesssim \frac{\varepsilon}{\sqrt{d\,(T + \log(\nfrac{Ld^2}{\varepsilon^2}))}}\,, 
    \qquadand m \gtrsim \frac{d^2\,(T^2+ LT\log(\nfrac{d^2}{\varepsilon^2}))}{\varepsilon^2}\,. \qedhere
\end{align*}
\end{proof}

\subsubsection{Integrated score estimation implies PAC density estimation}\label{sec:computational:integrated_to_density}

    In this section, we reduce the problem of PAC density estimation to that of implementing the integrated score estimation oracle introduced in \cref{def:integratedOracle}.
    
We need the following mild assumption on the density $P$.

\begin{assumption}[Second moment bound]\label{assumption:density}
There exists $M_2 > 0$ such that
the density $P$ has second moment {at most} $M_2$, \ie{}, $\int \|x\|^2\, P(\d x) \leq M_2$.
\end{assumption}
In this section, we prove the following.

\begin{theorem}[Integrated score estimation implies PAC density estimation]\label{thm:integratedscoretodensity} 
Consider a density $P$ satisfying \Cref{assumption:density} with parameter $M_2$.
For any $\eps\in (0,1)$, if there is an efficient $\epsilon$-integrated score estimation oracle for $P$ with terminal time $T\ge \frac{1}{2}\log(1+\nfrac{2M_2}{\varepsilon})$, then, there exists an efficient algorithm outputting a function $\wh P\colon \R^d\to\R_+$ (as an evaluation oracle) such that
\begin{align*}
    \int \E\Bigl\lvert \log \frac{\wh P(x_0)}{P(x_0)}\Bigr\rvert \,P(\d x_0) \le 2\varepsilon\,.
\end{align*}
\end{theorem}
We remark that while the output $\wh P$ is non-negative, it does not necessarily integrate to $1$ and is therefore not a valid probability density.
Thus, despite appearances,~\Cref{thm:integratedscoretodensity} does not provide a KL divergence guarantee.

\begin{proof}[Proof of~\Cref{thm:integratedscoretodensity}]
    By the existence of the integrated score estimation oracle, we know that there is an efficient algorithm that outputs a function $\wh{v}\colon \R^d\to \R$ such that
    \begin{align*}
        \int \E|\wh v(x_0) - v(x_0)|\,P(\d x_0) \le \varepsilon\,,
    \end{align*}
    where $v(\cdot)$ is given in~\cref{def:integratedOracle}.
    Since \cref{assumption:density} holds, $P$ has finite second moment, due to which \Cref{lem:identity,lem:err_term,eq:identity} are applicable.
    First, from \Cref{lem:identity}, we know that for all $x_0\in\R^d$:
    \begin{align*}
        &\int \log P_{T}\,\d Q_{T|0}(\cdot \mid x_0) - \log P(x_0) \\
        &\qquad = \underbrace{\int_0^T \int \bigl\{\|\nabla \log P_{t}\|^2 - 2\,\langle \nabla \log P_{t},\nabla \log Q_{t|0}(\cdot \mid x_0)\rangle\bigr\}\,\d Q_{t|0}(\cdot\mid x_0) \,\d t}_{=v(x_0)} + dT\,.
    \end{align*}
    This means that we can use the integrated score oracle as an estimator for the negative log-likelihood $-\log P$ by defining the function $\ell\colon \R^d \to \R$
    with
    \begin{align*}
        \ell(x_0) 
        & \coloneqq  \wh{v}(x_0) + d\,\bigl(T + \frac{1}{2}\log(2\pi e\,(1-\exp(-2T))) \bigr)\,.
    \end{align*}
    Next, from \cref{eq:identity}, we know that for any $x_0$,
    \[
        \abs{-\log{P(x_0)} - \ell(x_0) }
        \le |v(x_0) - \wh v(x_0)| + 
        \KL(Q_{T|0}(\cdot\mid x_0) \mmid P_{T})\,.  
    \]
    It remains to show that for $T$ sufficiently large the above error is negligible.
    {Toward this, we use the following bound from \cref{lem:err_term}:}
    \[
    \KL(Q_{T|0}(\cdot\mid x_0) \mmid P_{T}) \leq \frac{1}{\exp(2T)-1} \,\Bigl(\|x_0\|^2 + \int \|x\|^2\,P(\d x)\Bigr)\,.  
    \]
    It follows from~\Cref{assumption:density} that
    \begin{align*}
        \int |{-\log P(x_0)} - \ell(x_0)| \, P(\d x_0)
        &\le \varepsilon + \frac{2M_2}{\exp(2T)-1}\,.
    \end{align*}
    This is made at most $2\epsilon$ if $T \ge \frac{1}{2} \log(1+\nfrac{2M_2}{\eps})$.
    So, if we let $\wh P \deq \exp(-\ell)$, we have shown that
    \begin{align*}
        \int \E\bigl\lvert \log \frac{\wh P(x_0)}{P(x_0)}\bigr\rvert \,P(\d x_0) &\le 2\varepsilon\,. \qedhere
    \end{align*}
\end{proof}

\begin{remark}
    To explain why the theorem implies PAC density estimation, suppose that the score estimation oracle is implemented on the basis of samples and yields score estimates satisfying $\E\int_\tau^T \|s_t-\nabla \log P_t\|_{L^2(P_t)}^2\,\d t \le \eps_*^2$.
    By Markov's inequality, with probability at least $\nfrac{9}{10}$ over the samples, it holds that $\int_\tau^T \|s_t-\nabla \log P_t\|_{L^2(P_t)}^2\,\d t \le 10\eps_*^2$.
    Conditioned on this event, we can apply~\Cref{thm:integratedscoretodensity} and Markov's inequality to deduce that
    \begin{align*}
        \E P\{x\in\R^d : \wh P(x) \notin [e^{-2\eps/\delta}\,P(x),\, e^{2\eps/\delta}\,P(x)]\} \le \delta\,,
    \end{align*}
    \ie{}, it yields a {$(2\eps/\delta, \delta)$}-PAC density estimator.
\end{remark}

\subsubsection{Completing the reduction}

    {Combining \cref{thm:scoreEstimationtoIntegrated,thm:integratedscoretodensity}, we obtain the following result.
    
    \begin{theorem}[Reduction from score estimation to PAC density estimation]\label{thm:reduction}
        Let $P$ be a distribution on $\R^d$ that satisfies \cref{asmp:logLipschitz} with parameter $L$ and \Cref{assumption:density} with parameter $M_2\geq 1$.
        There is an efficient algorithm that, given access to a score estimation oracle for $P$ with $T \ge \frac{1}{2}\log(1+\nfrac{2M_2}{\varepsilon})$, outputs a function $\wh P\colon \R^d\to \R_+$ (as an evaluation oracle) such that 
        \begin{align*}
            \int \E\bigl\lvert \log \frac{\wh P(x_0)}{P(x_0)}\bigr\rvert \,P(\d x_0) \le 2\varepsilon\,.
        \end{align*}
        \noindent The algorithm makes $N$ calls to the score estimation oracle with accuracy $\eps_*$ for
        \[
            N = \wt{O}\Bigl(
                \frac{Ld^2\log^2(M_2)}{\eps^2}
            \Bigr)
            \qquadand
            \eps_* = \wt{O}\Bigl(\frac{\eps}{\sqrt{d\log(LM_2)}}\Bigr)\,,
        \]
        and the early stopping parameter $\tau$ of the score estimation oracle is required to satisfy $\tau \lesssim \nicefrac{\varepsilon^2}{Ld^2}$.
        Moreover, the algorithm takes $\poly(N)$ time.
    \end{theorem}}
\begin{proof}
    We set $T\asymp \log\nfrac{M_2}{\varepsilon^2}$.
    A more precise choice of parameters, obtained from the proof of~\Cref{thm:scoreEstimationtoIntegrated}, is given by
    \begin{align*}
        \eps_* &\lesssim \frac{\eps}{\sqrt{d\,(\log\nfrac{M_2}{\eps^2} + \log \nfrac{Ld^2}{\eps^2})}}\qquadand N \gtrsim \frac{d^2\,(\log^2(\nfrac{M_2}{\eps^2}) + L\log(\nfrac{M_2}{\eps^2})\log(\nfrac{d^2}{\eps^2}))}{\eps^2}\,. \qedhere
    \end{align*}
\end{proof}

{
\begin{remark}
    It is interesting to compare this reduction with the one for generation.
    If we ignore algorithmic considerations and solely focus on how the score estimation error $\eps_*$ translates into the error for distribution learning, then existing works on sampling from diffusion models (or simply Girsanov's theorem) imply the following statement.
    If $\wh P_{\msf{gen}}$ denotes the law of the output of the diffusion model with estimated scores, then
    \begin{align*}
        2\tv{\wh P_{\msf{gen}}}{P}^2 \le \KL(P \mmid \wh P_{\msf{gen}}) \lesssim \eps_*^2\,.
    \end{align*}
    Thus, $\eps_*$ score estimation error leads to $\eps_*$ error in total variation.
    On the other hand, our reduction shows that for density estimation,
    \begin{align*}
        \int \E\bigl\lvert \log \frac{\wh P}{P}\bigr\rvert\,\d P
        \lesssim \eps_*^2 + \wt O(\eps_*\sqrt d)\,.
    \end{align*}
    Since this performance metric greatly resembles a KL divergence, it is natural to wonder if the extra term $\wt O(\eps_*\sqrt d)$ is superfluous.
    Perhaps surprisingly, the answer is no: the right-hand side must contain a term scaling linearly with $\eps_*$, or else it would violate minimax lower bounds for density estimation; see~\Cref{sec:smoothDensityEstimation}.
\end{remark}
}

\subsection{Early stopping}\label{sec:earlyStop}

The reduction in~\Cref{sec:reductionScore} requires the assumption that the initial distribution $P$ has a sub-Gaussian score.
In this section, we show that even if this assumption is removed, we can still output an estimate of the density $P_\tau$ with early stopping.
This is used for our result on density estimation over H\"older classes in~\Cref{sec:smoothDensityEstimation}.

\begin{theorem}[PAC density estimation with early stopping]\label{thm:early_stopping}
    Let $P$ be a distribution on $\R^d$ that satisfies~\Cref{assumption:density} with parameter $M_2 \ge d$.
    Let $0 < \tau, \varepsilon < 1$.
    Then, given access to a score estimation oracle with early stopping $\tau$ and terminal time $T \ge \frac{1}{2}\log(1+\nfrac{2M_2}{\varepsilon})$, there is an algorithm that outputs a function $\wh P_\tau \colon \R^d\to\R_+$ (as an evaluation oracle) such that
    \begin{align*}
        \int \E\bigl\lvert \log \frac{\wh P_\tau(x_0)}{P_\tau(x_0)}\bigr\rvert \,P_\tau(\d x_0)
        \le 2\varepsilon\,.
    \end{align*}
    The algorithm makes $N$ calls to the score estimation oracle with accuracy $\eps_*$ for
    \[
        N \asymp 
        \frac{d^2\,(\log^2(\nfrac{M_2}{\eps^2}) + \tau^{-1}\log(\nfrac{M_2}{\eps^2})\log(\nfrac{d^2}{\eps^2}))}{\eps^2}\quad\text{and}\quad
        \eps_* \asymp \frac{\eps}{\sqrt{d\,(\log\nfrac{M_2}{\eps^2} + \log \nfrac{d^2}{\tau\eps^2})}}\,,
    \]
    Moreover, the algorithm takes $\poly(N)$ time.
\end{theorem}
\begin{proof}
    We apply the reduction in~\Cref{sec:reductionScore}, {after} replacing $P$ with $P_\tau$.
    By~\Cref{lem:score_subG}, $P_\tau$ has a sub-Gaussian score with parameter $L = 1/(1-e^{-2\tau}) = O(1/\tau)$.
\end{proof}

\subsection{Application to estimating the differential entropy}

In this paper, we focus on applications of our reduction from PAC density estimation. However, here we briefly mention that our guarantee immediately implies that a score estimation oracle can be used to estimate the differential entropy of the distribution $P$, which is also a well-studied problem~\citep[see][and references therein]{han2020differentialEntropy}.

To see how to estimate differential entropy using our tools, assume that we have access to $\wh P$ satisfying the guarantee of~\Cref{thm:integratedscoretodensity}. The estimator is simply defined by drawing $n$ i.i.d.\ samples $X_0^{(1)},\dotsc,X_0^{(n)}$ from $P$ and outputting the mean of the values $\{-\log \wh P(X_0^{(i)})\}_{i\in [n]}$. By the guarantee for $\wh P$,
\begin{align*}
    \E\Bigl\lvert \frac{1}{n} \sum_{i=1}^n \log \wh P(X_0^{(i)}) - \int \log P\,\d P\Bigr\rvert
    &= \E\Bigl\lvert \frac{1}{n} \sum_{i=1}^n \log \frac{\wh P(X_0^{(i)})}{P(X_0^{(i)})} + \sum_{i=1}^n \log P(X_0^{(i)}) - \int \log P\,\d P\Bigr\rvert \\[0.25em]
    &\le 2\eps + \sqrt{\frac{\Var_P \log P}{n}}\,.
\end{align*}

\section{Proof of Informal Theorem~\ref{thm:normalityEfficiency} (DDPM is an asymptotically efficient parameter estimator)}\label{sec:efficiency}

    {In this section, we study the use of DDPM score estimation for estimating parameters over a parametric family $P \in \hyP = \{P_\theta : \theta \in\Theta\}$.
    We consider the following idealized DDPM estimator, which is equivalent to selecting the parameter that minimizes the DDPM risk in \cref{lem:identity} over samples $x_0$ from $P$.}
    
    \ddpmEstimator*
    
    \noindent The main result of this section is that, under mild regularity assumptions on the distribution family $\hyP$ (essentially the same conditions needed for the asymptotic normality of the MLE, see~\Cref{ass:mle}) and by choosing the terminal time $T = T_n$ to grow sufficiently rapidly with the number of samples $n$ (namely, $T_n - \frac{1}{2}\log n \to \infty$), 
    the DDPM estimator $\DDPM$ converges in distribution to a Gaussian centered at $\theta^\star$ with covariance \emph{exactly} equal to the inverse Fisher information.

     \subsection{Implications {of the likelihood identity} for parameter estimation}

       {We begin by noting the following immediate consequence of ~\Cref{lem:identity}}: When specialized to a parametric family $P \in \hyP = \{P_\theta : \theta \in\Theta\}$, set $x_0 = X_0^{(i)}$, and sum over $i\in [n]$, where $X_0^{(1)},\dotsc,X_0^{(n)}\simiid P_{\theta^\star}$, \Cref{lem:identity} implies that the empirical risk for the maximum likelihood estimator (MLE) coincides with the empirical score matching loss, up to a known constant and a vanishing error.
    More precisely:

      \begin{proposition}[Tight connection between DDPM and MLE]\label{infthm}
        The {DDPM} objective $\hat{\cR}_n^{\rm DDPM}$ and the maximum likelihood objective $\hat{\cR}_n^{\rm MLE}$ satisfy:
        \begin{align*}
            \hat{\cR}_n^{\rm MLE}(\theta)  
            = 
            \hat{\cR}_n^{\rm DDPM}(\theta) + C_{d,T} + 
            \frac{1}{n} \sum_{i=1}^n \KL\bigl(Q_{T|0}(\cdot \mmid X_0^{(i)}) \bigm\Vert P_{\theta,T}\bigr)
        \end{align*}
        where $\hat{\cR}_n^{\rm MLE}(\theta) \coloneqq - \frac{1}{n} \sum_{i=1}^n \log P_\theta(X_0^{(i)})$, $\hat{\cR}_n^{\rm DDPM}$ is given in \Cref{def:ddpmMain},
        and $C_{d,T} = d\,(T+\frac{1}{2}\log(2\pi e\,(1-e^{-2T})))$ is a fixed constant.
    \end{proposition}
    Since we show later in this section (\Cref{lem:err_term}) that the final term above decays as $\exp(-2T)$, it is intuitive from~\Cref{infthm} that the DDPM estimator inherits the favorable properties of the MLE\@, including its statistical efficiency.
     We make this precise in~\Cref{sec:proof-limit} under the assumption that the family $\{P_\theta : \theta\in\Theta\}$ is differentiable in quadratic mean, which is essentially the weakest regularity condition under which the Fisher information is well-defined.

    \subsection{The proof of Informal Theorem~\ref{thm:normalityEfficiency}}\label{sec:proof-limit}

    We now proceed with the proof of \Cref{thm:normalityEfficiency} regarding the asymptotic efficiency of DDPM score matching.
    Let $Q_{t|0}(\cdot \mid x_0)$ denote the transition density of the OU process run until time $t$ started at time $0$ at $x_0.$ 
    
\paragraph{{Step 1 (Likelihood identity).}} As a first step,~\Cref{lem:identity} implies that for any $\theta\in\Theta$,
    \begin{align*}
        &\int \log P_{\theta,T}\,\d Q_{T|0}(\cdot \mid x_0) - \log P_\theta(x_0) \\
        &\qquad = \int_0^T \int \bigl\{\|\nabla \log P_{\theta,t}\|^2 - 2\,\langle \nabla \log P_{\theta,t},\nabla \log Q_{t|0}(\cdot \mid x_0)\rangle\bigr\}\,\d Q_{t|0}(\cdot\mid x_0) \,\d t + dT\,.
    \end{align*}

\paragraph{{Step 2} (Relating MLE and DDPM).}
Therefore, if we consider the one-sample empirical risks (where Equation \eqref{eq:ddpm-emp} is proved in \Cref{appendix:ddpm}),
\begin{align}
    \widehat{\cR}^{\rm MLE}(\theta)
    &= -\log P_\theta(x_0)\,, \nonumber \\
    \widehat{\cR}^{\rm DDPM}(\theta)
    &= \int_0^T \int \{\|\nabla \log P_{\theta,t}\|^2 - 2\,\langle \nabla \log P_{\theta,t}, \nabla \log Q_{t|0}(\cdot \mid x_0)\rangle\}\,\d Q_{t|0}(\cdot \mid x_0)\,\d t \,,
    \label{eq:ddpm-emp}
\end{align}
we can rewrite the identity above as
\begin{align}
    \widehat{\cR}^{\rm MLE}(\theta)
    &= \widehat{\cR}^{\rm DDPM}(\theta) + dT - \int \log P_{\theta,T}\,\d Q_{T|0}(\cdot \mid x_0)\nonumber \\
    &= \widehat{\cR}^{\rm DDPM}(\theta) + d\,\bigl(T + \frac{1}{2}\log(2\pi e\,(1-\exp(-2T))) \bigr) + \KL(Q_{T|0}(\cdot\mid x_0) \mmid P_{\theta,T})\,,\label{eq:identity}
\end{align}
where the last line follows by adding and subtracting $\int \log Q_{T|0}(\cdot \mid x_0)\,\d Q_{T|0}(\cdot \mid x_0)$ and using the formula for the differential entropy of a Gaussian.
This proves~\Cref{infthm}.

\paragraph{{Step 3} (Exponential decay of KL).}
We observe that the last term in \eqref{eq:identity} is controlled by the following lemma.

\begin{lemma}\label{lem:err_term}
    For any probability measure $P$ with finite second moment and any $x_0 \in \R^d$,
    \begin{align*}
        \KL(Q_{T|0}(\cdot\mid x_0) \mmid P_T)
        &\le \frac{1}{\exp(2T)-1}\,\Bigl(\|x_0\|^2 + \int \|x\|^2\,P(\d x)\Bigr)\,. 
    \end{align*}
\end{lemma}
\begin{proof}[Proof of \cref{lem:err_term}]
    By the dimension-free log-Harnack inequality~\citep[see, \eg{},][]{BobGenLed01Hyper, Wang06HarnackApplications, AltChe24SCI}, 
    \begin{align*}
        \KL(Q_{T|0}(\cdot \mid x_0) \parallel P_T) \le \frac{W_2^2(\delta_{x_0}, P)}{2\,(\exp(2T)-1)}\,.
    \end{align*}
    The result follows from the triangle inequality for $W_2$.
\end{proof}

\paragraph{{Statement and proof of \Cref{thm:normalityEfficiency}.}} 
Given steps 1--3, we are now ready to prove \Cref{thm:normalityEfficiency}.
To state our asymptotic normality result for the DDPM score matching estimator, we build on the following standard conditions for asymptotic normality of the MLE\@.
Note that it is implicitly assumed that the MLE exists for sufficiently large $n$.\footnote{This assumption could also be relaxed.}

\begin{assumption}[Conditions for asymptotic normality of MLE~\citep{Vaart98Asymptotic}]\label{ass:mle}
    The family ${\{P_\theta\}}_{\theta\in\Theta}$ is differentiable in quadratic mean (DQM) at an interior point $\theta^\star \in \Theta \subseteq \R^p$.
    Furthermore, there exists a function $L$ such that for all $\theta$, $\theta'$ in a neighborhood of $\Theta$, $\abs{\log P_\theta - \log P_{\theta'}} \le L\norm{\theta-\theta'}$ with $\int L^2\,\d P_{\theta^\star} < \infty$.
    The Fisher information matrix $I_{\theta^\star}$ is positive definite.
    Finally, the MLE $\MLE$ is consistent: $\MLE \to \theta^\star$ in probability as $n\to\infty$.
\end{assumption}
Here, the DQM condition weakens the classical assumptions for asymptotic normality of the MLE, which require the existence of a third derivative of $\theta \mapsto \log P_\theta$, and instead asks for the existence of a derivative of $\theta \mapsto \sqrt{P_\theta}$ at $\theta^\star$ in $L^2(P_{\theta^\star})$.
This covers non-differentiable examples such as the two-sided exponential location family.
Under~\Cref{ass:mle}, it is shown in~\citet[Theorem 5.39]{Vaart98Asymptotic} that $\sqrt n\,(\MLE - \theta^\star) \todist{} \cN(0, {I(\theta^\star)}^{-1})$.
We prove the following result.

\begin{theorem}[Asymptotic normality of the DDPM estimator]\label{thm:efficiency}
    Adopt~\Cref{ass:mle}.
    Consider the DDPM estimator $\DDPM$ where the time $T_n$ of the diffusion satisfies $T_n - \frac{1}{2}\log n \to \infty$.
    Assume also that for some neighborhood $\Theta'$ of $\theta^\star$, it holds that $\sup_{\theta \in \Theta'} \int \|x\|^2\,P_\theta(\d x) < \infty$, and that the DDPM estimator is consistent.
    Then, the DDPM estimator is asymptotically efficient: $\sqrt n\,(\DDPM - \theta^\star) \todist{} \cN(0, {I(\theta^\star)}^{-1})$.
\end{theorem}
\begin{proof}[Proof of \Cref{thm:efficiency}]
    We modify the proof of~\citet[Theorem 5.39]{Vaart98Asymptotic}, which relies on Theorem 5.23 therein.
    For $\theta\in\Theta$, let $m_\theta \deq \log p_\theta$ and $\bbP_n \deq \inparen{\nfrac{1}{n}}\sum_{i=1}^n \delta_{X_i}$.
    In order to invoke Theorem 5.23, it suffices to show that $\bbP_n m_{\DDPM} \ge \bbP_n m_{\MLE} - o_{P_{\theta^\star}}(n^{-1})$.
    By~\eqref{eq:identity},
    \begin{align*}
        -\bbP_n m_{\DDPM}
        &= \hat{\cR}_n^{\rm DDPM}(\DDPM) + c_{d,T} + \bbP_n\err(\DDPM) \\
        &\le \hat{\cR}_n^{\rm DDPM}(\MLE) + c_{d,T} + \bbP_n\err(\DDPM) \\
        &= -\bbP_n m_{\MLE} + \bbP_n[\err(\DDPM) - \err(\MLE)]\,,
    \end{align*}
    where $c_{d,T}$ is a constant and $\err(\theta, x) \deq \KL(Q_{T|0}(\cdot \mid x) \mmid P_{\theta,T})$.
    Since $\err$ is non-negative, it yields $\bbP_n m_{\DDPM} \ge \bbP_n m_{\MLE} - \bbP_n \err(\DDPM)$.
    Since the DDPM estimator is consistent,~\Cref{lem:err_term} and our assumptions imply $P_{\theta_\star} \err(\DDPM) \le 2\,{(\exp(2T)-1)}^{-1} \sup_{\theta\in \Theta'} \int \|x\|^2\,P_\theta(\d x) = o(\nfrac{1}{n})$.
    By Markov's inequality, $\bbP_n \err(\DDPM) = o_{P_{\theta^\star}}(\nfrac{1}{n})$.
    The rest of the proof is unchanged.
\end{proof}

\noindent The assumption of consistency for the MLE and the DDPM estimators is typically mild and can be handled by standard tools, \eg{},~\citet[\S 5.2]{Vaart98Asymptotic}.

\section{Proof of Informal Theorem~\ref{infthm:minimax}(Minimax optimal density estimation over the H\"older class)}\label{sec:smoothDensityEstimation}

Recently, many works have studied the statistical rates of estimation over non-parametric classes of densities, both for the score function (along the OU process) and its implications for learning a sampler, together with matching minimax lower bounds. 
For example,~\citet{DBLP:journals/corr/abs-2002-00107} obtained rates for estimating score functions based on Rademacher complexity, and~\citet{wibisono2024optimalscoreestimationempirical} established the minimax rate for estimating a Lipschitz score for a sub-Gaussian density.
The works~\citet{oko2023minimalDiffusion, dou2024optimalscorematchingoptimal, Zha+24MinimaxScore} showed that DDPM score estimation can lead to minimax optimal rates for distribution learning.
However, we emphasize that these prior works only showed that one can learn a sampler using the existing reduction (\Cref{ssec:related}), whereas our goal is to show that DDPM score estimation leads to density estimators.

In this subsection, we start with a representative result on score estimation from the literature, namely the result of~\citet{dou2024optimalscorematchingoptimal}.
Their work considered the following H\"older class of densities.

\begin{definition}[H\"older class]\label{defn:holder}
    For $C > 2$ and $s, L > 0$, let $\hyH_s(C, L)$ denote the class of probability densities $P$ supported on $[-1,1]$ with the following properties:
    \begin{itemize}
        \item $P$ is continuous on $[-1,1]$, admits $\lfloor s\rfloor$ derivatives on $(-1,1)$, and
        \begin{align*}
            |D^{\lfloor s\rfloor} P(x) - D^{\lfloor s \rfloor} P(y)| \le L\,|x-y|^{s-\lfloor s\rfloor}\,, \qquad\text{for all}~x,y\in (-1,1)\,.
        \end{align*}
        \item On {the domain} $[-1,1]$, {the density} $P$ is bounded away from $0$ and $\infty$, \ie{}, $C^{-1} \le P \le C$.
    \end{itemize}
\end{definition}
We note that the restriction to one dimension is purely for ease of exposition (as is common in the literature).
All of the ideas below can be adapted to the higher-dimensional case by replacing the rate $n^{-s/(2s+1)}$ with $n^{-s/(2s+d)}$.

For $P \in \hyH_s(C, L)$,~\citet{dou2024optimalscorematchingoptimal} proved that the score $\nabla \log P_t$ can be estimated in $L^2(P_t)$ at a certain rate (see~\Cref{rmk:dou_score_rate} below).
Our goal is to convert this into a result for density estimation.
Although this is ultimately a consequence of our framework in~\Cref{sec:framework}, the main effort here is to provide a common setting in which we can apply the reductions in~\Cref{sec:framework}, the score estimation rates of~\citet{dou2024optimalscorematchingoptimal}, and the lower bounds in the density estimation literature so that the final result is \textbf{minimax optimal}.
This is not entirely trivial.
For example, since the densities in~\Cref{defn:holder} are compactly supported, they do not have globally Lipschitz scores, and for $s < 2$ they do not even have Lipschitz scores in the interior $(-1,1)$.
Hence, to apply our reductions, we utilize early stopping.
Moreover, since our approach leads to PAC density estimation, which is a relatively weak solution concept, it is not immediately clear that it is compatible with {existing notions of risk in the literature on density estimation.}

Our notion of risk is defined as follows.
Given an estimator $\wh P$ using $n$ samples and a probability density $P$, {we define the $L^1$ risk 
\begin{align*}
    \cR_n(\wh P, P)
    \deq \int_{[-1,1]} \E_P|\wh P(x_0) - P(x_0)|\,\d x_0\,.
\end{align*}
Note that if $\wh P$ were a probability density on $[-1,1]$, this would correspond to twice the total variation distance.}
Here, we use the subscript on $\E_P$ to indicate that the estimator is based on $n$ i.i.d.\ samples from $P$.
Henceforth, all of the asymptotic notation (\eg{}, $\lesssim$, $O(\cdot)$, \ldots) suppresses constants which do not depend on $n$.

Our main result {of this section} is stated below.

\begin{theorem}[Density estimation for H\"older classes]\label{thm:holder}
    Let $C > 2$, $s, L > 0$.
    {
    \begin{enumerate}
        \item The following minimax lower bound holds:
        \begin{align*}
            \inf_{\wh P} \sup_{P\in \hyH_s(C,L)} \cR_n(\wh P, P) \gtrsim n^{-s/(2s+1)}\,.
        \end{align*}
        \item There is an estimator $\wh P$ based on DDPM score estimation such that
        \begin{align*}
            \sup_{P\in \hyH_s(C,L)} \cR_n(\wh P, P) \lesssim n^{-s/(2s+1)}\sqrt{\log n}\,.
        \end{align*}
    \end{enumerate}
    }
\end{theorem}
Before proceeding to the proof, we need the following remark.

\begin{remark}\label{rmk:dou_score_rate}
    The paper~\citet{dou2024optimalscorematchingoptimal} actually considers estimation of the score function along the heat flow, rather than the OU process: $\wt P_t \deq P * \cN(0,t\,\mathrm{Id})$.
    From Tweedie's identity~\cite{Rob1956EmpBayes},
    \begin{align*}
        -(1-e^{-2t})\,\nabla \log P_t(x_t)
        &= x_t - e^{-t}\, \E[X_0 \mid e^{-t}\,X_0 + \sqrt{1-e^{-2t}}\,Z = x_t]\,, \\
        -t\,\nabla \log \wt P_t(\wt x_t)
        &= \wt x_t - \E[X_0 \mid X_0 + \sqrt t\,Z = \wt x_t]\,,
    \end{align*}
    one can relate the two score functions:
    \begin{align*}
        \nabla \log P_t(x_t)
        &= e^t\,\nabla \log \wt P_{e^{2t}-1}(e^t x_t)\,.
    \end{align*}
    Hence, if $\wt s_t$ is an estimator for $\nabla \log \wt P_t$ and we set $s_t(x_t) \deq e^{t}\,\wt s_t(e^t x_t)$, then
    \begin{align*}
        \|s_t - \nabla \log P_t\|_{L^2(P_t)}
        = e^{t}\,\|\wt s_t - \nabla \log \wt P_{e^{2t}-1}\|_{L^2(\wt P_{e^{2t}-1})}\,.
    \end{align*}
    Applying this to the score estimator of~\citet{dou2024optimalscorematchingoptimal} for the class $\hyH_s(C, L)$, we obtain
    \begin{align*}
        \|s_t - \nabla \log P_t\|_{L^2(P_t)}^2
        &\lesssim \frac{1}{n\,(e^{t}-1)^2} \wedge \frac{1}{nt^{3/2}} \wedge (n^{-2(s-1)/(2s+1)} + t^{s-1})\,.
    \end{align*}
\end{remark}

\begin{proof}[Proof of~\Cref{thm:holder}]
    {We divide the proof into two parts corresponding to the lower bound and the upper bound.}
    
    \paragraph{Lower bound.}
    The lower bound is classical, {see, \eg{},~\citet{YanBar1999Minimax} (which also considers the more general Besov classes, among others).}

    \paragraph{Upper bound.}
    For the upper bound, we apply~\Cref{thm:early_stopping}  with the score estimator in~\Cref{rmk:dou_score_rate}.
    Let $\phi_{\sigma^2}$ denote the density of the Gaussian $\cN(0, \sigma^2)$.
    For $x_0 \in [-1, 1]$ and $\tau\lesssim 1$, since $P$ is H\"older continuous on $[-1,1]$ with exponent $s \wedge 1 \coloneqq \min\{s,1\}$ (see~\Cref{lem:holder_lip}) and the density $P$ is upper bounded by $C$,
    \begin{align*}
        &|P_\tau(x_0) - P(x_0)|\\
        &\qquad = \Bigl\lvert \int e^\tau P(e^\tau x)\,\phi_{1-e^{-2\tau}}(x_0 - x)\,\d x - P(x_0)\Bigr\rvert \\
        &\qquad \le C\,(e^\tau-1) + \int |P(e^\tau x) - P(x_0)|\,\phi_{1-e^{-2\tau}}(x_0-x)\,\d x \\
        &\qquad \le C\,(e^\tau-1) + \int_{|x|\le e^{-\tau}} |P(e^\tau x) - P(x_0)|\,\phi_{1-e^{-2\tau}}(x_0-x)\,\d x \\
        &\qquad\qquad + C \int_{|x|\ge e^{-\tau}}\phi_{1-e^{-2\tau}}(x_0-x)\,\d x  \\[0.25em]
        &\qquad \lesssim \tau + \int |e^\tau x-x_0|^{s\wedge 1}\,\phi_{1-e^{-2\tau}}(x_0-x)\,\d x + C \int_{|x|\ge e^{-\tau}}\phi_{1-e^{-2\tau}}(x_0-x)\,\d x \\[0.25em]
        &\qquad \lesssim \tau^{s\wedge 1} + \int |x-x_0|^{s\wedge 1}\,\phi_{1-e^{-2\tau}}(x_0-x)\,\d x + C \int_{|x|\ge e^{-\tau}}\phi_{1-e^{-2\tau}}(x_0-x)\,\d x \\[0.25em]
        &\qquad \lesssim \tau^{(s\wedge 1)/2} + C \int_{|x|\ge e^{-\tau}}\phi_{1-e^{-2\tau}}(x_0-x)\,\d x\,.
    \end{align*}
    In the above, the first inequality follows by adding and subtracting $\int P(e^\tau x)\, \phi_{1-e^{-2\tau}}(x_0-x)\, \d x$ and the triangle inequality, and the third inequality by H\"older continuity.
    The last inequality follows by standard tail bounds on Gaussian random variables with variance $1-e^{-2\tau}$ (\cref{fact:momentSubGaussian}).
    {The last term is bounded by the probability that a centered Gaussian with variance $1-e^{-2\tau}$ exceeds $1-e^{-\tau}-|x_0|$.
    By standard Gaussian tail estimates,
    \begin{align*}
        \int_{|x|\ge e^{-\tau}}\phi_{1-e^{-2\tau}}(x_0-x)\,\d x
        &\lesssim \begin{cases}
            \tau\,, & |x_0| \le 1-\Omega(\sqrt{\tau\log \nfrac{1}{\tau}})\,, \\
            1\,, & \text{otherwise}\,.
        \end{cases}
    \end{align*}
    This leads to the integrated estimate
    \begin{align*}
        \int_{[-1,1]} |P_\tau(x_0) - P(x_0)|\,\d x_0
        \lesssim \tau^{(s\wedge 1)/2} + \tau + \sqrt{\tau\log\nfrac{1}{\tau}}
        \lesssim \tau^{(s\wedge 1)/2}\sqrt{\log\nfrac{1}{\tau}}\,.
    \end{align*}
    We choose $\tau$ so that this quantity is bounded by $n^{-s/(2s+1)}$, so that $\log\nfrac{1}{\tau} \lesssim \log n$.

    Our estimator is a clipped version of the early stopped PAC density estimator, \ie{}, we set
    \begin{align*}
        \wh P \deq \max\{1/C', \min\{\wh P_\tau, C'\}\}\,.
    \end{align*}
    Here, $C' > 0$ is a constant not depending on $n$, $\wh P_\tau$ is the output of early stopping (see~\Cref{thm:early_stopping}), and $\wh P$ is not to be confused with the PAC density estimator without early stopping.
    We choose $C'$ so that $1/C' \le P_\tau \le C'$ on $[-1,1]$; such a constant exists by~\citet[Lemma 11]{dou2024optimalscorematchingoptimal}.

    From~\Cref{thm:early_stopping}, since $\eps_*\asymp r_n^* \deq n^{-s/(2s+1)}$, we obtain $\eps \asymp \eps_* \sqrt{\log(\nfrac{1}{\eps_*})}$, \ie{}, $\eps \asymp r_n^*\sqrt{\log n}$.
    Hence, we have the guarantee
    \begin{align*}
        \int \E_P\bigl\lvert \log \frac{\wh P_\tau(x_0)}{P_\tau(x_0)}\bigr\rvert \,P_\tau(\d x_0)
        \lesssim r_n^*\sqrt{\log n}\,.
    \end{align*}
    Recall that $\E_P$ corresponds to the expectation over the $n$ i.i.d.\ samples used for the estimator $\wh P_\tau$.
    Since $P_\tau \gtrsim 1$ on $[-1, 1]$~\citep[see Lemma 11 in][] {dou2024optimalscorematchingoptimal}, it implies
    \begin{align*}
        \int_{[-1,1]} \E_P\bigl\lvert \log \frac{\wh P_\tau(x_0)}{P_\tau(x_0)}\bigr\rvert \,\d x_0
        \lesssim r_n^*\sqrt{\log n}\,.
    \end{align*}
    Since $1/C' \le P_\tau \le C'$ on $[-1,1]$ and $\wh P$ is $\wh P_\tau$ clipped to $[1/C', C']$, it follows that
    \begin{align*}
        \int_{[-1,1]} \E_P\bigl\lvert \log \frac{\wh P(x_0)}{P_\tau(x_0)}\bigr\rvert \,\d x_0
        \lesssim r_n^*\sqrt{\log n}\,.
    \end{align*}
    Now, since $\wh P/P_\tau$ is bounded away from $0$ and $\infty$, Taylor expansion of the logarithm shows that
    \begin{align*}
        \int_{[-1,1]} \E_P\bigl\lvert \frac{\wh P(x_0)}{P_\tau(x_0)} - 1\bigr\rvert \,\d x_0
        \lesssim r_n^*\sqrt{\log n}\,.
    \end{align*}
    Finally, since $P_\tau$ is lower bounded on $[-1,1]$, it implies
    \begin{align*}
        \int_{[-1,1]} \E_P|\wh P(x_0) - P_\tau(x_0)| \,\d x_0
        \lesssim r_n^*\sqrt{\log n}\,.
    \end{align*}
    Combining this with the upper bound on $\int_{[-1,1]} |P_\tau(x_0) - P(x_0)|\,\d x_0$ finishes the proof.
    }
\end{proof}
    
\noindent Some more remarks are in order.

\begin{remark}
    Regarding the computational cost of our estimator, once the score estimates have been computed, the number of evaluations of the estimated scores is {$\wt O(n^{4s/(2s+1)}) \le \wt O(n^2)$}.
\end{remark}

\begin{remark}
    We believe that the same strategy yields density estimators for other settings, such as the one considered in~\citet{YakPuc25Score}.
    For brevity, we do not pursue such results here.
\end{remark}

\section{Proof of Informal Theorem~\ref{infthm:glm} (PAC density estimation for Gaussian location mixtures)}\label{sec:GLM}

    In this section, we study density estimation for the {classical} family of Gaussian location mixtures.
    This family is parameterized by a distribution $Q$ (of the means):
    given $Q$, the corresponding Gaussian Location Mixture (GLM) is 
    \[
        \cM = Q~*~\normal{0}{\sigma^2\,\mathrm{Id}}\,.
        \tag{Gaussian location mixture}
    \]
    This is a (possibly) continuous mixture of spherical Gaussians where $Q$ is the distribution of means.
    In the special case where $Q$ is discrete, say $Q=\sum_{i=1}^k w_i \delta_{\mu_i}$, then the above is exactly a mixture of $k$ spherical-covariance Gaussians with means $\mu_1,\dots,\mu_k$.
    However, in general, $Q$ can be continuous, and then the above GLM family is non-parametric.
    
    This family can also be seen as the smoothening of an underlying family of distributions.
    Smoothness is a very natural property of real-world distributions (which are subject to independent errors) and a huge body of work in theoretical computer science studies algorithms in the presence of smoothed data; a partial list is \citet{spielman2004smooth,haghtalab2020smooth,pmlr-v247-chandrasekaran24a,haghtalab2022efficientSmooth,haghtalab2020smooth,haghtalab2024smoothJACM,block2024performanceempiricalriskminimization} and we refer the reader to Chapter 13 in \citet{roughgarden2021beyond} and \citet{beier2004typical} for an overview of these works.
    
    Apart from computer science, this family has also appeared in the statistics literature {at least as early as the work of \citet{kiefer1956consistency}}, where it is called the Gaussian location mixture \cite{kimGuntuboyina2022minimax,sahaGuntuboyina2020gaussian}.
    These works consider arbitrary mixing measures $Q$ and focus on the sample complexity (without computational considerations). \citet{sahaGuntuboyina2020gaussian} studied the finite sample complexity bounds for non-parametric maximum likelihood estimation in squared Hellinger distance, while \citet{kimGuntuboyina2022minimax} {gave} a minimax bound for estimation in squared Hellinger distance using kernel density estimation.
    For the case of finite mixtures,~\citet{Doss+23Mixtures} obtained optimal rates of estimation for both the mixing and mixture densities with efficient algorithms.
    
    Recently, \citet{gatmiry2024learning}, in a surprising result developed a quasi-polynomial time generator for a subset of this family satisfying the following locality assumption which restricts the choice of $Q$; but still allows it to be continuous and non-parametric.
    
    \begin{definition}[$(k,R,D,w_{\min})$-locality \cite{gatmiry2024learning}]\label{asmp:GLM:locality}
        Given parameters $R\geq 1$, $D >0$, $k\in \mathbb{N}$, and $w_{\min}\in (0,\nfrac{1}{k})$, the GLM with distribution $Q$ and variance $\sigma^2$ is said to be $(k,R,D,w_{\min})$-local if the following hold:
        \begin{itemize}
            \item For every point $x$ in the support of $Q$, $Q\inparen{B(x,R)}\geq w_{\min}$.
            \item There exist points $x_1,x_2,\dots,x_k$ such that the support of $Q$ is a subset of $\bigcup_{i=1}^k B(x_i, R)$.
            \item $Q(B(0,D))=1$.
        \end{itemize}
    \end{definition}
    However, they left the problem of obtaining a density estimation algorithm for this family open.

    The main result of this section is a quasi-polynomial time PAC density estimator for this family.
    
    \begin{theorem}[PAC density estimator for Gaussian location mixtures] \label{thm:GLM:DE}
    
    Let $\calM$ be a $(k,R,D,w_{\min})$-local GLM with variance ${\sigma^2} \in (0,1]$. 
    Fix a target accuracy $\eps\leq\min\inbrace{\nfrac{1}{2}, \nfrac{\sigma}{R}, \nfrac{1}{D}, \nfrac{1}{d}, w_{\min}}$.
    Define 
    \[
        N =
        \Bigl(
            d\log{\frac{1}{\eps}}
        \Bigr)^{
            O\bigl(
                (\log{\frac{1}{\eps}})^7
                +
                (
                    \frac{R}{\sigma}
                    \log{\frac{1}{\eps}}
                )^4
            \bigr)
        }\,.
    \]
    There is an algorithm that, given 
        accuracy parameter $\epsilon$, 
        instance parameters $(\sigma,k,R,D,w_{\min})$, and 
        sample access to $\cM$, 
    draws $N$ \iid{} samples from $\calM$, runs in {$\poly(N)$} time, and  returns an $(\nfrac{\eps}{\delta},\delta)$-PAC density estimator for $\cM$ for any coverage parameter $\delta\in (0,1)$.
    \end{theorem}
    Note that due to the requirement on $\eps$, the exponent has an implicit dependence on $\log{(dk)}$ (since {$\eps \le w_{\min}\leq \nfrac{1}{k}$}).
    
    To the best of our knowledge, this is the first sub-exponential PAC density estimation algorithm for such a general and non-parametric family of distributions.

    To gain some intuition about the above result, consider the special case of spherical Gaussian mixture models with $k$ components, which have significantly more structure than the (continuous) general Gaussian location mixture.
    In this case, we can improve the dependence on $d$ to $\min\inbrace{d,k}$ by using an SVD-based pre-processing scheme by incurring an additive cost of $\poly(d)$ in the running time~\citep[see][]{vempala2004spectral}.
    This results in a time and sample complexity $\poly(\nfrac{dk}{\eps})+ k^{\polylog(dk)}$, which comes very close to the running time of $\poly(\nfrac{dk}{\eps})+\inparen{\nfrac{k}{\eps}}^{O(\log^2 k)}$ by \citet{diakonikolas2020small}.
    While the result has poorer polynomial-dependence in the exponent, the power of the result comes by its generality in extending beyond  Gaussian mixture models -- which prohibits the use of specialized algebraic tools.\footnote{As noted in \citet{gatmiry2024learning}, the algorithm of \citet{diakonikolas2020small} relies on finding $\eps$-covers of plausible parameters. Since the $\eps$-cover of even a constant radius ball has size exponential in the dimension, it seems unlikely that their methods will extend to the more general GLMs that we study.} We mention that a recent result by \citet{DiaKan25HighOrder} gives a generator for spherical GMMs with $\poly(d,k,1/\epsilon)$ runtime under the assumption that the means are bounded by $O(\sqrt{\log k})$~\citep[Theorem 4.22]{DiaKan25HighOrder}.

    Finally, as mentioned before, the guarantees of \cref{thm:GLM:DE} go beyond spherical Gaussian mixture models by allowing $Q$ to be a continuous distribution, provided it satisfies \cref{asmp:GLM:locality}.
    In particular, as a corollary, we obtain a PAC density estimator for the family of distributions satisfying a weak manifold assumption introduced by \citet{gatmiry2024learning}.
        In particular, this assumption requires the support $S$ of the distribution to be coverable by $C^\ell$ $\ell_2$-balls of radius $R$ where $\ell>0$ is a parameter controlling the sample complexity.
        
    \begin{corollary}
        Fix GLM parameters $\sigma=1$ and $C, R>1$ and accuracy parameter $0<\eps<\nfrac{1}{2}$.
        Suppose $Q$ belongs to the family of distributions satisfying the following: Each distribution is supported on some set $S_Q$ such that $S_Q$ has radius $D$, $S_Q$ can be covered with $C^\ell$ $\ell_2$-balls of radius $R$, and every point $\mu\in S$ satisfies $Q(B(\mu,R))\geq \eps\cdot C^{-\ell}$.
        Let $\cM$ be the resulting $d$-dimensional spherical Gaussian location mixture, \ie{}, $\cM=Q*\cN(0,\sigma^2\,\mathrm{Id}).$
        
        Then, there is an algorithm that, given $\eps,C,\ell,R,D$ and sample access to $\cM$, draws $N=\inparen{d}^{O(\ell)+O(\log{\nfrac{(dD)}{\eps}})^7}$ 
        samples from $P$, runs in $\poly(N)$ time, and outputs an $(\nfrac{\eps}{\delta},\delta)$-PAC density estimator for $\cM$ for any coverage parameter $\delta\in (0,1)$.
    \end{corollary}
    As noted in \citet{gatmiry2024learning}, this is a family of distributions for which diffusion models can perform density estimation while standard methods (such as binning or kernel density estimation) do not work.
    In particular, while binning can learn the distribution $Q$, the above algorithm gives a method to learn $Q$ convolved with a spherical Gaussian, a more challenging problem.

    \subsection{Proof of Theorem~\ref{thm:GLM:DE}}

        The algorithm in \Cref{thm:GLM:DE} combines our results {reducing density estimation to score estimation (which, in turn, utilizes the connection between DDPM score estimation and MLE)} (see \Cref{sec:reductionScore}) with the following score estimation algorithm that is implicit in~\citet[\S 5.5]{gatmiry2024learning}.

        \begin{theorem}[{General Gaussian mixture score estimator}]\label{thm:GLM:score}
            Let $\calM$ be an $(k,R,D,w_{\min})$-local GLM with variance $\sigma\in (0,1]$. 
            For $\zeta\leq\min\inbrace{\nfrac{1}{2}, \nfrac{\sigma}{R}, \nfrac{1}{D}, \nfrac{1}{d}, w_{\min}}$ and $\eta\in (0,1)$, define
                \[
                    N_{\zeta,\eta} =
                    \Bigl(
                        d\log{\frac{1}{\eta}}
                    \Bigr)^{
                        O\bigl(
                            (\log{\frac{1}{\zeta}})^7
                            +
                            (
                                \frac{R}{\sigma}
                                \log{\frac{1}{\zeta}}
                            )^4
                        \bigr)
                    }\,.
                \]
        There is an algorithm that, given 
            accuracy and confidence parameters $(\zeta,\eta)$, 
            time $\zeta\sigma^2\lesssim t \lesssim \log\inparen{\nfrac{{(d+D)}}{\zeta}}$,\footnote{The guarantee in \citet{gatmiry2024learning} holds for a larger range of times $t$. We restrict to this range as it is sufficient for our use and simplifies presentation.}
            instance parameters $(\sigma,k,R,D,w_{\min})$, and 
            sample access to $\cM$, 
        draws $N_{\zeta,\eta}$ i.i.d.\ samples from $\calM$, runs in {$\poly(N_{\zeta,\eta})$} time and, returns a score function $\wh s_t$ that satisfies 
        \[
            \|\wh s_t - \nabla \log \calM_t\|_{L^2(\cM_t)}^2 
            ~~\leq~~
            \wt{O}\Bigl(\frac{\zeta^2\,\inparen{1 + t}}{\log(d+D)}\Bigr)
            ~~\Stackrel{}{\leq}~~ 
            \wt{O}\inparen{\zeta^2}
            \,,
        \]
         with probability $1-\eta$ over the samples generated from the mixture $\calM$.
        Here, $\cM_t$ is the distribution obtained by running the OU process for time $t$ starting from $\cM_0=\cM$.
        \end{theorem}
        Some remarks are in order.
        First, the above result can be extended to times $t$ larger than $\log(\nfrac{(d+D)}{\zeta})$, although we do not require this guarantee and, hence, to simplify presentation, we omit this.
        Second, the specific polynomial dependence of the score estimation error on $\zeta$ is not very important as it only affects constant factors in the exponent.
        Further, the exponent in the definition of $N_{\zeta,\eta}$ has an implicit dependence on $\log{d}$, $\log{D}$, and $\log{\nfrac{1}{w_{\min}}}$, since we require $\zeta \leq\min\inbrace{\nfrac{1}{2}, \nfrac{\sigma}{R}, \nfrac{1}{D}, \nfrac{1}{d}, w_{\min}}$.
        Further, it also has an implicit dependence on $\log{k}$ since $w_{\min}\leq \nfrac{1}{k}$.
        
        Next, while the above result only provides a bound on the score estimation error at time $t$, the algorithm used computes the score on a list of times $t_1,\dotsc,t_N$, which contains $t$.
        This is necessary because \citet{gatmiry2024learning}'s algorithm iterates over times $t_1,\dots,t_N$ and to estimate the score at time $t_i$, it requires the clustering at time $t_{i-1}$, which, in turn, requires an estimate of the score at time $t_{i-1}$.
        
        \begin{proof}[Proof of 
            \cref{thm:GLM:score}]
            Let $M_2\coloneqq \int\norm{x}^2\,\cM(\d x)$. 
            Since $Q$ satisfies \cref{asmp:GLM:locality} and $\cM$ is a convolution of $Q$ by a spherical Gaussian with variance $\sigma^2\leq 1$, $M_2\lesssim D^2+d$.
            In~\citet[\S 5.5]{gatmiry2024learning}, they showed that for any sequence of times $0<t_1<\dots<t_N$ satisfying Properties P1 and P2 below, their algorithm satisfies the following score estimation guarantee:
            for each $1\leq i\leq N$,
            the score $s_{t_i}$ computed for the $i$-th noise level is $\zeta_i$-accurate in $L^2(\cM_{t_i})$ for $\zeta_i^2=\frac{\zeta^2\,(\sigma^2+t+1)}{\log(t_N+1)}$, \ie{}, 
            \[
                \norm{s_{t_i}-\nabla\log{\cM_{t_i}}}_{L^2(\cM_{t_i})}^2
                \leq \zeta_i^2\,.
            \]
            They needed the following properties on the time sequence $(t_1,\dots,t_N)$:
            \begin{enumerate}
                \item \textit{P1 (Start and end points):}~~ $t_1\asymp \frac{\zeta^2\sigma^2}{2\sqrt{d}}$ and $t_N\asymp \frac{d+M_2}{\zeta^2}$.
                \item  \textit{P2 (Recurrence):}~~ For each $1\leq i\leq N-1$,
                \[
                    t_k+1 = \inparen{t_{k+1} +1}\cdot \max\inbrace{e^{-2\alpha}, \inparen{t_{k+1} +1}^{-\alpha}}
                    \quadwhere 
                    \alpha \coloneqq \frac{\zeta^2}{M_2 + d\log{T+1}}\,.
                \]
            \end{enumerate}
            The above theorem follows by constructing a sequence $0<t_1<\dots<t_N$ satisfying the above properties and containing the noise scale $t$ provided in the theorem (\ie{}, ensuring there is an index $i$ with $t_i=t$).
        \end{proof}
        
        {We are now ready to employ our polynomial-time reduction to produce a PAC density estimator.}
        To do this, we first have to verify that the well-conditioned model $\calM$ satisfies {the} mild assumptions {(see \cref{asmp:logLipschitz,assumption:density}) required by our reduction from {PAC density estimation to score estimation}}.
        Recall that \cref{asmp:logLipschitz,assumption:density} are parameterized by constants $L$ (sub-Gaussian score) and $M_2$ (second moment bound).
        To bound these, we use the following lemma.
        
        \begin{restatable}[]{lemma}{gaussianMixtureSatifyAssumptions}\label{fact:GLM:satisfiesLBMAssumptions}
                Let $\calM$ be an $(k,R,D, w_{\min})$-local GLM with variance $\sigma\in (0,1]$. 
                Let $L$ and $M_2$ be the following constants:
                \[
                    L = \frac{1}{\sigma^2}\quadand
                    M_2 = D^2 + \sigma^2 d\,.
                \]
                Then, $\nabla \log \cM$ is $\sqrt L$-sub-Gaussian under $\cM$, and $\int \|x\|^2\,\cM(\d x) \le M_2$.
                Hence, the mixture $\calM$ satisfies \cref{asmp:logLipschitz,assumption:density} with the constants above.
            \end{restatable}
            
        \begin{proof}[Proof of \cref{fact:GLM:satisfiesLBMAssumptions}]
            Recall that $\cM=Q*\normal{0}{\sigma^2\, \mathrm{Id}}$.
            We divide the proof into two parts.
    
            \paragraph{Sub-Gaussian score.}
                By \Cref{lem:subG_score_mixture}, it suffices to show that $N_{\sigma^2}=\normal{0}{\sigma^2\, \mathrm{Id}}$ has a $\sigma^{-1}$-sub-Gaussian score, and this follows from \Cref{lem:lip_score_implies_subG} since $\nabla \log{N_{\sigma^2}(x)}=-\sigma^{-2}x$ is $\sigma^{-2}$-Lipschitz.
    
            \paragraph{Bound on the second moment.}
                Since $\cM$ is a convex combination of $\normal{\mu}{\sigma^2\,\mathrm{Id}}$ for $\mu$ in the support of $Q$, it suffices to bound $\int \|\cdot\|^2\,\d\normal{\mu}{\sigma^2\, \mathrm{Id}} = \|\mu\|^2 + \sigma^2 d \le D^2 + \sigma^2 d$.
        \end{proof}
        
        Now, we are ready to prove \cref{thm:GLM:DE}.
        
        \begin{proof}[Proof of \cref{thm:GLM:DE}]   
            \cref{thm:reduction} implies that to obtain an $(\nfrac{\eps}{\delta},\delta)$ PAC density estimator for any coverage probability $\delta>0$, it is sufficient to make $C=C(\eps,L,M_2)$ calls to a score estimation oracle with aggregated error $\eps_*=\eps_*(\eps,L,M_2)$, terminal time $T\asymp \log(1+\nfrac{2M_2}{\eps})$, where using the values of $L$ and $M_2$ from \cref{fact:GLM:satisfiesLBMAssumptions}, we have 
            \[
                C = \tilde{O}\binparen{\frac{d^2\log^2{D}}{\sigma^2\eps^2}}\,,\quad 
                \eps_* = \tilde{O}\binparen{
                    \frac{\eps}{\sqrt{d\log{\nfrac{D}{\sigma}}}}
                }
                \,,\quadand 
                T \asymp \log{\binparen{\frac{D+d}{\eps}}}
                \,.
                \yesnum\label{eq:GLM:reductionParam}
            \]
            Moreover, the calls to the score estimation oracle are made at times $t$ satisfying $\tau\leq t\leq T$, where the starting time $\tau$ is $\tau \asymp \frac{\eps^2}{Ld^2}$ and, substituting $L=\nfrac{1}{\sigma^2}$, is 
            \[
                \tau \asymp \frac{\sigma^2\eps^2}{d^2}\,.
                \yesnum\label{eq:GLM:valueOfTau}
            \] 
            It remains to show that \cref{thm:GLM:score} with a suitably small $\zeta$ implies an efficient score estimation oracle with the desired aggregate error.
            First, note that to obtain an aggregate error $\int_\delta^T\eps^2_t\, \d t\leq \eps_*^2$, it suffices to let accuracy $\zeta\leq \wt{O}(\nfrac{\eps_*}{\sqrt{T}})$ for each time $t$ it is queried.
            Next, the reduction in \cref{thm:reduction} makes $N$ queries with timescale $\tau \leq t\leq T$ (for $\tau$ as in \eqref{eq:GLM:valueOfTau}), the requirement on timescales is satisfied for 
            \[
                \poly\bigl(\frac{\eps}{d+D}\bigr)
                \lesssim 
                \zeta\lesssim \frac{\sigma^2\eps^2}{d^2}\,,
            \]
            The above observations allow us to use the score estimation oracle in \cref{thm:GLM:score}  at these times.
            For each call $1\leq i\leq C$, we query the score estimation oracle in \cref{thm:GLM:score} with 
            \[
                \zeta = \min\Bigl\{
                    \frac{ \sigma^2 \eps^2}{d^2},~ \wt{O}\bigl(\frac{\eps}{\sqrt{d\,\log{\nfrac{D}{\sigma}}}}\bigr)
                \Bigr\}\,,
            \]
            and confidence $\nfrac{\delta}{N}$.
            Observe that this choice of $\zeta$ satisfies the requirements that $\poly(\nfrac{\eps}{(d+D)}) \lesssim \zeta\lesssim \nfrac{\eps^2}{(Ld^2)}$ and $\zeta\leq \wt{O}(\nfrac{\eps_*}{\sqrt{T}})$ mentioned above; in fact, it satisfies $\zeta\geq \wt{\Omega}(\eps^6)$ as $L=\nfrac{1}{ \sigma^2 }$ and $\eps\leq \min\inbrace{\nfrac{1}{d},\nfrac{\sigma}{R},\nfrac{1}{D}}$ (\cref{fact:GLM:satisfiesLBMAssumptions}).
           
            Taking a union bound over all $C$ calls implies that with probability $1 - \delta$, the construction in \cref{thm:scoreEstimationtoIntegrated} outputs a valid $(\eps,T)$-integrated score estimation oracle as required.
            It remains to bound the total number of samples from $\cM$ used.
            Since each call to the score estimation in \cref{thm:GLM:score}
            requires 
            \[
                \Bigl(
                        d\log{\frac{N}{\delta}}
                    \Bigr)^{
                        O\bigl(
                            (\log{\frac{1}{\zeta}})^7
                            +
                            (
                                \frac{R}{\sigma}
                                \log{\frac{1}{\zeta}}
                            )^4
                        \bigr) 
                    } 
                \qquadtext{samples from $\cM$\,,}
            \]
            and we ensured $\zeta\geq \Omega(\eps^6)$, $\eps\leq \min\inbrace{\nfrac{\sigma}{R},\nfrac{1}{d},\nfrac{1}{D}}$, and \cref{eq:GLM:reductionParam}, and
            the total samples used are
            \begin{align*}
                C\,\Bigl(
                        d\log{\frac{C}{\delta}}
                \Bigr)^{
                        O\bigl(
                            (\log{\frac{1}{\eps}})^7
                            +
                            (
                                \frac{R}{\sigma}
                                \log{\frac{1}{\eps}}
                            )^4
                        \bigr) 
                    } 
                ~~~&=~~~
                    \wt O\bigl(\frac{d^2\log^2{D}}{ \sigma^2 \eps^2}\bigr)\,
                    \Bigl(
                        d\log{\frac{C}{\delta}}
                    \Bigr)^{
                        O\bigl(
                            (\log{\frac{1}{\eps}})^7
                            +
                           (
                                \frac{R}{\sigma}
                                \log{\frac{1}{\eps}}
                            )^4
                        \bigr) 
                    } \\
                &=~~~
                    \Bigl(
                        d\log{\frac{1}{\delta\eps}}
                    \Bigr)^{
                        O\bigl(
                            (\log{\frac{1}{\eps}})^7
                            +
                            (
                                \frac{R}{\sigma}
                                \log{\frac{1}{\eps}}
                            )^4
                        \bigr) 
                    } \,.
            \end{align*}
            The running time follows since the score estimation oracle and the reduction in \cref{thm:scoreEstimationtoIntegrated} run in sample-polynomial time.
        \end{proof}

\section{Proof of Informal Theorem~\ref{infthm:scoreCrypto} (Cryptographic lower bounds for score estimation)}\label{sec:hardnessScoreEstimation}
 
In this section, we discuss how our connections between score estimation and PAC density estimation imply computational bottlenecks for score estimation.
In \Cref{sec:clwe-pac-density}, we show that an algorithm that performs PAC density estimation for Gaussian mixtures implies an algorithm for distinguishing homogeneous Continuous Learning with Errors (\hCLWE{}). In \Cref{sec:lwe}, using standard reductions from \citet{bruna2021continuous,gupte2022continuous}, we show that this implies an algorithm for the {Continuous Learning with Errors} (\CLWE{}) and {Learning with Errors} (\LWE{}) problems. Combined with our reduction from {PAC density estimation to score estimation}, {these} establish cryptographic hardness results for score estimation.
See \cref{fig:crypto} for a summary of the reductions.

\begin{remark}[SQ lower bounds] 
A natural question is whether it is meaningful to derive lower bounds for score estimation in a restricted computational model, such as the Statistical Query (SQ) model~\cite{kearns1998efficient}.
Notably, existing SQ lower bounds for density estimation, specifically, lower bounds for evaluators, directly imply SQ lower bounds for score estimation.
    This follows from the fact that, information theoretically, generators and evaluators for densities~\citep[in the sense of][]{kearns1994learnability} are equivalent objects, and, since the SQ model permits arbitrary additional computation beyond the SQ queries themselves, this equivalence holds within the SQ framework as well.
   Consequently, known reductions from generation to score estimation already yield SQ hardness results; that is, SQ lower bounds for density evaluation~\citep[\eg{},][] {diakonikolas2017statistical} imply corresponding SQ lower bounds for score estimation.
    Therefore, we restrict our attention to lower bounds based on complexity theory. 
    To obtain such computational lower bounds, it seems to be more difficult to use the existing works on generation because generators and evaluators are \emph{not} computationally equivalent. 
    However, as we will see, our framework can be used to derive computational bottlenecks for score estimation (based on standard complexity-theoretic assumptions).
\end{remark}

\begin{figure}[bht]
    \centering
    \includegraphics[width=0.75\linewidth]{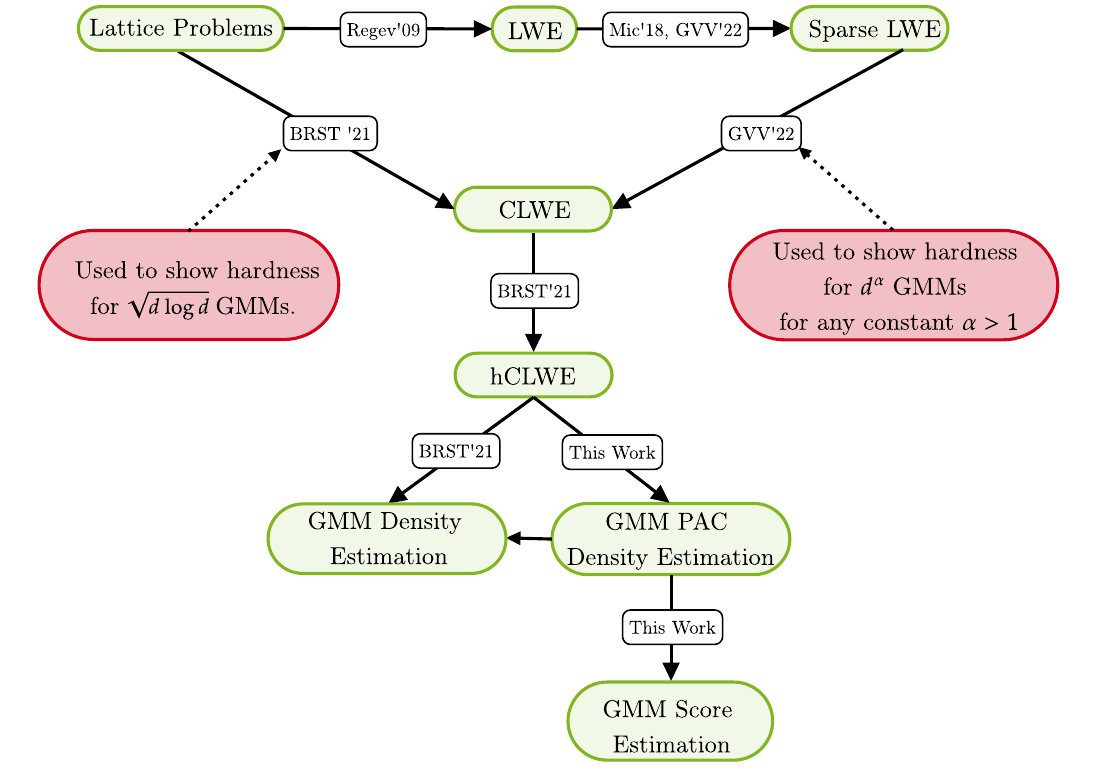}
    \caption{Illustration of reductions between different cryptographic problems, density estimation, PAC density estimation, and score estimation. We use these reductions to obtain the hardness of GMM score estimation (\cref{infthm:scoreCrypto,thm:cryptoHardnessGMM}).
    }
    \label{fig:crypto}
\end{figure}

\subsection{PAC density estimation for GMMs implies homogeneous CLWE}\label{sec:clwe-pac-density}

Our starting point is the {seminal} work of \citet{bruna2021continuous}, which introduced \CLWE{}. One of their most important ideas is a reduction from \hCLWE{} (\ie{}, homogeneous \CLWE{}) to GMM density estimation.
For background on \CLWE{}, we refer {the reader} to \Cref{CLWE}.

In this section, we modify this reduction and show that PAC density estimation also implies an algorithm for \hCLWE{}{; note that this is not immediate from the reduction in \citet{bruna2021continuous} since PAC density estimation is an easier task than density estimation}.

We say that an algorithm $A$ is a $(c, \epsilon,\delta)$-PAC density estimator of the density $P$ if, with probability at least $c$ over $S \sim P^n$, the output $\wh P \sim A(S)$ satisfies
\[
    \E\,P\{x\in\R^d : e^{-\epsilon} P(x) \leq \wh P(x) \leq e^\epsilon P(x)\} \geq 1-\delta\,.
\]
Originally, in~\Cref{def:PACdensityEstimation}, we stated this definition with $c = \nfrac{9}{10}$ because this parameter can be boosted, but it is convenient to keep $c$ as a free parameter for some of the results in this section.
A model $\wh P$ that satisfies the above guarantee is simply called a $(c,\epsilon,\delta)$-PAC density estimator for $P.$

Next, we show that an algorithm that solves PAC density estimation for $d$-dimensional GMMs with $k$ components can be used to solve $\hCLWE{}_{\beta,\gamma}$ better than {the naïve algorithm that selects one option uniformly at random} for some choice of $k,\beta$, and $\gamma$.

Before stating the result, let us shortly recall $\hCLWE{}$. For details, see \Cref{CLWE}.

\begin{definition}[\hCLWE{} \cite{bruna2021continuous}]
For parameters $\beta, \gamma > 0$, the average-case decision problem $\hCLWE{}_{\beta, \gamma}$ is to distinguish with probability $>1/2$ the following two distributions over $\mathbb{R}^d$: 
\begin{enumerate}
    \item[$(\mathsf H_0)$] The Gaussian distribution in $d$ dimensions with mean $0$ and covariance $\mathrm{Id}/(2\pi).$
    \item[$(\mathsf H_1)$] the \hCLWE{} distribution $H_{w, \beta, \gamma}$ (see \Cref{def:hclwe}) for some uniformly random unit vector $w \in \mathbb{R}^d$ (which is fixed for all samples).
\end{enumerate}
\end{definition}
\noindent 
This distinguishing problem is considered to be computationally hard for some parameters $\beta(d), \gamma(d)$ even for advantage $1/\poly(d)$. For instance, the main result of \citet{bruna2021continuous} is that when $\beta(d) \in (0,1)$ and $\gamma(d) \geq 2 \sqrt{d}$ (such that the ratio $\gamma/\beta$ is polynomially bounded), then there is a polynomial-time quantum reduction from standard lattice problems such as $\mathrm{GapSVP}_{\alpha}$ \cite{regev2009lattices} (for some approximation factor $\alpha \approx d/\beta)$ to $\hCLWE{}_{\beta, \gamma}$. In other words, an efficient algorithm for $\hCLWE{}$ would imply an efficient quantum algorithm that approximates worst-case lattice problems 
within polynomial factors.
We prove the following reduction from $\hCLWE{}$ to PAC density estimation for Gaussian mixture models.

\begin{proposition}[PAC density estimation for GMMs implies \hCLWE{}]\label{prop:hCLWE-PAC}
Let $\beta = \beta(d) \in (0,\nfrac{1}{32})$, $\gamma = \gamma(d) \geq 1$, and $g(d) \geq 4\pi$.
For $k \gtrsim \gamma \sqrt{g(d)}$,
if there is an $\exp(g(d))$-time algorithm that solves $(\nfrac{9}{10},\epsilon,
\delta)$-PAC density estimation for mixtures of $2k+1$ Gaussians in $d$ dimensions for sufficiently small absolute constants $ \epsilon,\delta$, then there is a
$O(\exp(g(d)))$-time algorithm that solves $\hCLWE{}_{\beta, \gamma}$.
\end{proposition} 
Since the reduction of \citet{bruna2021continuous} from lattice-based problems to $\hCLWE{}$ holds in the regime where $\beta \in (0,1)$ and $\gamma \geq 2\sqrt{d}$, the above result (taking $g(d) = O(\log d)$), implies that $(\nicefrac{9}{10}, \epsilon,\delta)$-PAC density estimation for GMMs with $\Omega(\sqrt{d \log d})$ components requires super-polynomial time for some absolute constants $\epsilon,\delta$; otherwise there is a polynomial-time quantum algorithm for standard lattice problems.
Before proving this result, we need the following two intermediate lemmas.

\begin{lemma}[PAC density estimation solves simple vs.\ composite testing]\label{lemma:tester}
Fix $\epsilon,\delta \in (0,1)$.
Let $\mathsf H_0 = \{P_0\}$ and $\msf H_1$ be families of probability distributions.
Let $\wh P$ be a $(c,\eps,\delta)$-PAC density estimator over $\msf H_0 \cup \msf H_1$.
Then, there is a tester that draws $m$ (fresh) i.i.d.\ samples from $P \in \msf H_0 \cup \msf H_1$, makes $m$ queries to $\wh P$, and has the following properties: 
\begin{enumerate}
    \item If $P = P_0$, then the tester outputs $\msf H_0$ with probability at least $c\,{(1-\delta)}^m$. 
    \item If $P \in \msf H_1$, $\tv{P_0}{P} \geq \nfrac{1}{10}$, $\eps \le \nfrac{1}{160}$, and $\delta \le \nfrac{1}{80}$, then the tester outputs $\msf H_1$ with probability at least $c\,(1-(\nfrac{79}{80})^m)$.
\end{enumerate}
In particular, if $c > \nfrac{1}{2}$ and we take $m$, $\nfrac{1}{\eps}$, and $\nfrac{1}{\delta}$ to be sufficiently large absolute constants, then in both cases the probability of success is {strictly larger than $\nfrac{1}{2}$}.
\end{lemma}
\begin{proof}
    The testing algorithm draws $m$ points $x_1,\dots,x_m$ i.i.d.\ from $P$ and evaluates each one of them using the model $\wh P$. Then it outputs $\msf H_1$ if there exists an $1\leq i\leq m$ such that
    \begin{align}\label{eq:good_pt}
        \frac{\wh P(x_i)}{P_0(x_i)} \notin [e^{-\epsilon}, e^{\epsilon}]\,.
    \end{align}
    Otherwise, it outputs $\msf H_0$.
    Throughout, we always work conditionally on the event of probability at least $c$ that the PAC density estimation succeeds.
\begin{enumerate}
    \item Let us assume that $P = P_0$.
    By the guarantee of $\wh P$, for each $i$, the probability that condition~\eqref{eq:good_pt} fails is at most $\delta$.
    Thus, the probability that the tester outputs $\msf H_0$ is at least ${(1-\delta)}^m$.
    
    \item Now, let us assume that $P \in \msf H_1$ and $\tv{P_0}{P} > \nfrac{1}{10}$.
        Let $S = \{x\in\R^d : P(x) > P_0(x)\}$. By the definition of total variation distance, $P(S) - P_0(S) {\,>\,} \nfrac{1}{20}$. For $\eta \in (0,1)$, consider the set $T = \{x \in S : P(x)/P_0(x) > e^{\eta}\}$. We show that $P(T) \geq \nfrac{1}{40}$ provided that $\eta \le \nfrac{1}{80}$. 

    For any $x \in S \setminus T$, we have $P(x) \leq e^\eta P_0(x)$, which means that $P(S \setminus T) - P_0(S \setminus T) \leq (e^\eta-1)\, P_0(S \setminus T) \leq 2\eta$.
    Hence, we can write
    \[
    \frac{1}{20} \le P(S)-P_0(S) = P(T) + P(S \setminus T) - P_0(T) - P_0(S \setminus T)
    \le P(T) + 2\eta\,.
    \]
    Choosing $\eta = \nfrac{1}{80}$, we obtain $P(T) \ge \nfrac{1}{40}$.
    
    Next, consider the event $G \deq \{\wh P(x_1)/P(x_1) \ge e^{-\eps}\}$.
    Note that on the event $G \cap \{x_1 \in T\}$,
    \begin{align*}
        \wh P(x_1)
        \ge e^{-\eps} P(x_1)
        \ge e^{\nfrac{1}{80}-\eps} P_0(x_1)
        \ge e^{\eps} P_0(x_1)\,,
    \end{align*}
    provided $\eps \le \nfrac{1}{160}$, and in this case the tester outputs $\msf H_1$.
    Also, by the PAC guarantee, the event $G \cap \{x_1\in T\}$ has probability at least $\nfrac{1}{40}-\delta \ge \nfrac{1}{80}$, provided $\delta \le \nfrac{1}{80}$.
    The probability that the tester fails to reject on any of the $m$ samples is therefore bounded by $(\nfrac{79}{80})^m$.
\end{enumerate}
\end{proof}

\noindent The next lemma is used to show that PAC density estimation over mixtures of finitely many Gaussians is enough to perform PAC density estimation over $\hCLWE{}$ distributions, which technically are Gaussian mixtures {with infinitely many components}.

\begin{lemma}[\hCLWE{} from mixtures of Gaussians]\label{lemma:transfer}
Assume that there is a $(\nfrac{9}{10},\epsilon,\delta)$-PAC density estimation algorithm for mixtures of $2k+1$ Gaussians in $d$ dimensions that uses $\exp(g(d))$ samples for sufficiently small absolute constants $\epsilon,\delta$. Then, there is a $(\nfrac{9}{10} - 2e^{-4\pi},\epsilon+4e^{-4\pi}/\delta, 2\delta + 2e^{-4\pi})$-PAC density estimation algorithm for the distribution $H_{w,\beta,\gamma}$ that uses $\exp(g(d))$ samples.
\end{lemma}
We recall that the success probability above can be increased from $\nfrac{9}{10} - 2e^{-4\pi}$ via boosting (by, \eg{}, computing the median of several PAC density estimators).

\begin{proof}
The idea for this lemma follows from the work of \citet{bruna2021continuous}. From \citet[Proposition 5.2]{bruna2021continuous}, we know that if we truncate the distribution $H_{w,\beta,\gamma}$ to its first $2k+1$ central mixture components to form a GMM $H^{(k)}$ with $2k+1$ components, it holds that 
\[
\tv{H_{w,\beta,\gamma}}{H^{(k)}} \leq 2 \exp(-\pi k^2/ (2\gamma^2))\,,
\]
when $\beta = \beta(d) \in (0,1)$ and $\gamma = \gamma(d) \geq 1.$

The total variation distance between the joint distribution of $\exp(g(d))$ samples from $H_{w, \beta, \gamma}$ and
that of $\exp(g(d))$ samples from $H^{(k)}$
is bounded by
\[
\exp(g(d)) \cdot 2 \exp(-\pi k^2/ (2\gamma^2))\,.
\]
By picking $k = 2\gamma\sqrt{g(d)/\pi}$ and since $g(d) \geq 4\pi$, we get that the total variation is of order
\[
2\exp(-g(d)) \leq 2 \exp(-4\pi)\,.
\]
Hence, we can condition on the event that the $\exp(g(d))$ samples are drawn from $H^{(k)}$ by reducing the success of the algorithm to $\nfrac{9}{10} - 2e^{-4\pi}$. We now have to deal with the density ratio. Let us assume that we have an $(\epsilon,\delta)$-PAC density estimator $\wh P$ for $H^{(k)}$ and write
\[
\frac{\wh P(x)}{H_{w,\beta,\gamma}(x)}
= 
\frac{\wh P(x)}{H^{(k)}(x)} \cdot 
\frac{H^{(k)}(x)}{H_{w,\beta,\gamma}(x)}\,.
\]
Using the above calculations, we know that
\[
\tv{H^{(k)}}{H_{w,\beta,\gamma}} \leq 2\exp(-g(d)) \leq 2\exp(-4\pi)\,.
\]
Let $\epsilon_0 \in (0,1)$ and $B \deq \{x\in\R^d : |H^{(k)}(x)/H_{w,\beta,\gamma}(x) - 1| \leq \epsilon_0\}$. We know that
\begin{align*}
    \int \bigl\lvert\frac{H^{(k)}(x)}{H_{w,\beta,\gamma}(x)} - 1\bigr\rvert\,H_{w,\beta,\gamma}(\d x) \leq  4 \exp(-4\pi)\,.
\end{align*}
Hence, by Markov's inequality, $H_{w,\beta,\gamma}(B^c) = H_{w,\beta,\gamma}\{|H^{(k)}/H_{w,\beta,\gamma} - 1| > \epsilon_0\} \leq \nfrac{ 4 \exp(-4\pi)}{\epsilon_0}$. Hence, there exists a set $B$ with mass at least $1-\nfrac{ 4 \exp(-4\pi)}{\epsilon_0}$ such that for any $x \in B$, the density ratio is in the interval $[1-\epsilon_0,1+\epsilon_0]$. {Let us take $\epsilon_0 =  4 \exp(-4\pi)/\delta$.
Since $\wh P$ is an $(\eps,\delta)$-PAC density estimator for $H^{(k)}$, we obtain
\[
    \E H^{(k)}\Bigl\{x\in\R^d : e^{-\eps-\eps_0} \leq \frac{\wh P(x)}{H_{w,\beta,\gamma}(x)} \leq e^{\eps+2\eps_0}\Bigr\} \geq 1-2\delta\,.
\]
This further implies that
\[
    \E H_{w,\beta,\gamma}\Bigl\{x\in\R^d : e^{-\eps-\eps_0} \leq \frac{\wh P(x)}{H_{w,\beta,\gamma}(x)} \leq e^{\eps+2\eps_0}\Bigr\} \geq 1-2\delta- 4 e^{-4\pi}\,.
\]
Hence, $\wh P$ is a $(\nfrac{9}{10}- 4 e^{-4\pi}, \epsilon + 2\eps_0, 2\delta +  4 e^{-4\pi})$-PAC density estimator for $H_{w,\beta,\gamma}$.}
\end{proof}

\noindent We are now ready to complete the reduction from \hCLWE{} to PAC density estimation.

\begin{proof}[Proof of \Cref{prop:hCLWE-PAC}]
We apply the PAC density estimation algorithm to the unknown given distribution $P$, which is either $D_{\R^d}$ or $H_{w,\beta,\gamma}$.
Let $\wh P$ be the output of the $(\eps,\delta)$-PAC learner using $\exp(g(d))$ samples from $P$.
{
We claim that $\wh P$ is a PAC density estimator for $P$.
Since the guarantee of the algorithm is that it is a $(\nfrac{9}{10}, \epsilon,\delta)$-PAC density estimator for mixtures of $2k+1$ Gaussians in $d$ dimensions, we directly get the guarantee when $P = D_{\R^d}$. When $P = H_{w,\beta,\gamma}$, we instead apply~\Cref{lemma:transfer}, which implies that $\wh P$ is a $(\nfrac{9}{10} -  4 e^{-4\pi},\epsilon_0,\delta_0)$-PAC density estimator for $H_{w,\beta,\gamma}$ with slightly worsened parameters $\eps_0$, $\delta_0$.

Next, to obtain a tester for \hCLWE{}, we apply~\Cref{lemma:tester}, where $P_0 = D_{\R^d}$ and $\msf H_1 = \{H_{w,\beta,\gamma} : w \in \R^d\,,\;\|w\| = 1\}$.
It is known that for every $H_{w,\beta,\gamma} \in \msf H_1$, $\tv{D_{\R^d}}{H_{w,\beta,\gamma}} > 1/2$.
We check the numerical constants to ensure that the probability of success for the tester is $>1/2$ for both cases.

For the case where $P = H_{w,\beta,\gamma}$, since $c = \nfrac{9}{10} -  4 e^{-4\pi} > 0.89$, it suffices to take $m = 66$.
And for the case where $P = D_{\R^d}$, it suffices to have $ce^{-\delta_0 m} > 1/2$, so we require $\delta_0 \le 0.0087$ (which satisfies $\delta_0 \le \nfrac{1}{80}$).
Since $\delta_0 = 2\delta +  4 e^{-4\pi}$, we take $\delta = 0.0043$.
This leads to $\eps_0 = \eps + 4e^{-4\pi}/\delta \le \nfrac{1}{160}$ provided that $\eps \le 0.003$.
All of the conditions of~\Cref{lemma:tester} are thus met.
}
\end{proof}

\subsection{Lower bounds for score estimation for GMMs}\label{sec:lwe}

    In this section, we use our reduction from PAC density estimation to score estimation (\Cref{sec:reductionScore}) in combination with our reduction from \hCLWE{} to GMM PAC density estimation to prove computational hardness results for score estimation for GMMs. The root of computational intractability will be problems like \CLWE{} and \LWE{}, which are assumed to be intractable due to reductions from standard problems on lattices \cite{regev2009lattices,bruna2021continuous,gupte2022continuous}. For the required background on \LWE, we refer the reader to \Cref{LWE}, for basic properties on lattices to \Cref{Lattice}, and for details on \CLWE, we refer to \Cref{CLWE}.
    
To illustrate the flavor of our hardness reductions, using the results of \citet{bruna2021continuous} on \CLWE{}, combined with our reduction, we can exclude score estimation for Gaussian mixtures with $k \geq \sqrt{d \log d}$ components in $\poly(d,k)$ time with error smaller than $O(1/\sqrt{d \log d})$.
 
This result follows because we can efficiently transform a score estimation oracle to a PAC density estimator and, using \Cref{sec:clwe-pac-density}, such a PAC density estimator for GMMs can be used to efficiently solve the \CLWE{} problem. However, from \citet{bruna2021continuous}, we know that there is a polynomial-time quantum reduction from standard lattice problems such as \GapSVP{} and \textsf{SIVP} \cite{regev2009lattices}\footnote{\GapSVP{} and \textsf{SIVP} are among the main computational problems
on lattices and are believed to be computationally hard (even with quantum computation) for
polynomial approximation factors.} to \CLWE{}. Note that we can only exclude $\wt{O}(1/\sqrt{d})$-accurate score estimation oracles for GMMs and this is inherent to our {density-to-score estimation} reduction (since getting an $(\epsilon,\delta)$-PAC density estimator for absolute constants $\epsilon,\delta$ in $\R^d$ requires calling the score estimation oracle with accuracy $\nicefrac{1}{\sqrt{d\log d}}$).

In the above reduction, we use the hard instance of \citet{bruna2021continuous}. To make the above chain of reductions fully rigorous, we also have to verify that the hard GMM satisfies the assumptions of our {density-to-score} estimation reduction.

As we will see later, we can use the follow-up work of \citet{gupte2022continuous} and 
improve the above result showing that score estimation is computationally intractable for a smaller number of components in the GMM (see \Cref{sec:LWE-hardness}).

\subsubsection{Score estimation for GMMs implies (homogeneous) CLWE}

First, we combine the reduction of \Cref{sec:clwe-pac-density} with our {PAC density to score estimation reduction} to show:

\begin{theorem}[GMM Score estimation implies \hCLWE{}]\label{thm:score-crypto-GMM}
Let $\beta = \beta(d) \in (0,1/32)$, $\gamma = \gamma(d) \geq 1$, and $g(d) \geq 4\pi$.
For $k \gtrsim \gamma \sqrt{g(d)}$,
if there is an $\exp(g(d))$-time algorithm that returns an $\epsilon_*$-score estimation oracle for $2k+1$ mixtures of Gaussians in $d$ dimensions that satisfy \Cref{asmp:logLipschitz,assumption:density} with parameters $\max\{L,M_2\} = \poly(d)$  for some $\epsilon_* \leq C/\sqrt{d \log d}$ for some absolute constant $C>0$, then there is a
$\poly(d) \cdot \exp(g(d))$-time algorithm that solves $\hCLWE{}_{\beta, \gamma}$.
\end{theorem}
Given this result, we can invoke known hardness results for \hCLWE{} \cite{bruna2021continuous,gupte2022continuous} (see, \eg{}, \Cref{CLWE-hard}) to show the hardness of score estimation for mixtures of Gaussians.

\begin{proof}
Assume that there is an $\exp(g(d))$-time algorithm that returns an $\epsilon$-score estimation oracle for $2k+1$ mixtures of Gaussians in $d$ dimensions for $\varepsilon_* \leq C/\sqrt{d \log d}$. Since the target family satisfies our assumptions with $\max\{L,M_2\} = \poly(d)$, we can employ the {PAC} density estimation to score estimation reduction: For any sufficiently small constants $\varepsilon_1, \delta_1$, this reduction implies that there is an $(\varepsilon_1, \delta_1)$-PAC density estimator for $2k+1$ mixtures of Gaussians in dimension $d$ with $N = \wt{O}(\varepsilon_1^{-2}{L d^2 \log^2(M_2)}{}) = \poly(d)$ calls to the score estimation oracle with accuracy ${O}(\varepsilon_1/\sqrt{d 
\log(L M_2)}) = C/\sqrt{d \log d}$ for some absolute constant $C>0$. This means that we can get, in time $\poly(d)\, e^{g(d)}$, a $(9/10, \varepsilon_1,\delta_1)$-PAC density estimation algorithm for $2k+1$ mixtures of Gaussians in $d$ dimensions for some sufficiently small absolute constants $\varepsilon_1,
\delta_1$. Since $k \gtrsim \gamma\sqrt{d}$, we can use \Cref{prop:hCLWE-PAC} to complete the proof.
\end{proof}

\noindent Let us now take $\gamma(d) = 2\sqrt{d}$ and $g(d) = O(\log d)$. Corollary 4.2 in \citet{bruna2021continuous} shows that $\hCLWE{}_{\beta,\gamma}$ is hard for that choice of $\gamma$   assuming hardness for worst-case lattice problems. Then, under the same hardness assumption, \Cref{thm:score-crypto-GMM} excludes a $\poly(d,k)$-time algorithm for GMM score estimation when $k \gtrsim \sqrt{d \log d}$ and the error of the oracle is as small as $\nfrac{1}{\sqrt{d \log d}}$.

The last step in order
to complete the proof, is to show that the instance that realizes the above reduction satisfies the assumptions of our density-to-score reduction with parameters $L,M_2$ that are $\poly(d,k).$ The key observation of
\citet{bruna2021continuous} is that \hCLWE{} has a natural
interpretation as an instance of mixtures of Gaussians. 
This mixture has \emph{infinitely} many components, but they manage to reduce \hCLWE{} to a \emph{truncated} version of \hCLWE{}, which contains only the first $2k+1$ central ones \citep[Proposition 5.2]{bruna2021continuous}. Finally, using this truncation, they show that any {distinguisher} between the \hCLWE{} distribution and the standard multivariate Gaussian can be used to solve \CLWE{}.
Therefore, an algorithm for density estimation for Gaussian mixtures implies a solver for \CLWE{}. We now explicitly compute the parameters of our assumptions for the truncated \hCLWE{} distribution with parameters $\beta,\gamma$.

\begin{definition}[Truncated \hCLWE{} distribution]\label{def:truncatedhCLWE}
Let $s$ be the hidden direction (with $\|s\|=1$).
The density of the truncated \hCLWE{} distribution with parameters $\beta,\gamma$ at $x$ is {proportional to}
\[
\sum_{i = - k }^{ k } \rho_{\sqrt{\beta^2 + \gamma^2}}(i) \cdot
\rho(x^{\perp s}) \cdot 
\rho_{\beta/\sqrt{\beta^2 + \gamma^2}}\Bigl(\<s,x\> - \frac{\gamma}{{\beta^2 + \gamma^2}}\,i \Bigr)\,,
\]
where $\rho_r(x) = \exp(-\pi \|x\|^2/r^2)$ and $\rho = \rho_1  \propto  \cN(0,\mathrm{Id}/(2\pi))$.
\end{definition}
The above distribution is a mixture of Gaussians with $2k+1$ components of width $\beta/\sqrt{\beta^2+\gamma^2}$ in the secret direction and unit width in the orthogonal space. We next compute the parameters $L,M_2$ of our assumptions as a function of $\beta$, $\gamma$, and $d$.

\begin{proposition}
   The truncated \hCLWE{} distribution with parameters $\beta, \gamma$ has a 
   {$\sqrt L$-sub-Gaussian score} and second moment bounded by $M_2$ with
   $
       L =  1 + \nfrac{\gamma^2}{\beta^2}
       $ and $ 
       M_2 = d+ \nfrac{k^2}{\gamma^2}\,.
   $ \label{prop:LWE}
\end{proposition}

\begin{proof}
    First, let us state the mean and covariance of each component of the GMM in \cref{def:truncatedhCLWE}.
    Consider the $i$-th component, which has density 
    \[
        \propto \rho(x^{\bot s}) \cdot \rho_{\beta/\sqrt{\beta^2+\gamma^2}}\Bigl(\inangle{s,x} - \frac{\gamma}{\beta^2+\gamma^2}\, i\Bigr)\,.
    \]
    Here, the first term is the standard Gaussian distribution in the directions orthogonal to $s,$ and the second component is a Gaussian in the direction $s$, with mean $\frac{\gamma}{\beta^2+\gamma^2}\cdot i$ and variance $\frac{\beta^2}{\beta^2+\gamma^2}$.
    Therefore, the $i$-th component itself is also a Gaussian with mean and covariance
    \[
        \mu_i = \frac{\gamma\cdot i}{\beta^2 + \gamma^2}\, s 
        \qquadand
        \Sigma_i = \begin{bmatrix}
                \mathrm{Id} & 0\\
                0 & \frac{\beta^2}{\beta^2 + \gamma^2}
        \end{bmatrix}
        = \mathrm{Id} - \Bigl(1 - \frac{\beta^2}{\beta^2 + \gamma^2}\Bigr)\, ss^\top\,.
    \]
    Observe the following inequality, which will be useful in the subsequent proof:
    \[
        \norm{\Sigma_i^{-1}} = \lambda_{\min}(\Sigma_i) \leq 1 + \frac{\gamma^2}{\beta^2}\,.
        \yesnum\label{eq:crypto:assumption}
    \]
    Now, we are ready to bound $L$ and $M_2.$
    
    \paragraph{Sub-Gaussian score.}
        By \Cref{lem:subG_score_mixture}, it suffices to show each component has a {$\sqrt{1 + \nfrac{\gamma^2}{\beta^2}}$}-sub-Gaussian score, and this follows from \Cref{lem:lip_score_implies_subG} since the $i$-th component has distribution $\cN_{\mu_i,\Sigma_i}=\normal{\mu_i}{\Sigma_i}$ and its score is $\nabla \log{\cN_{\mu_i,\Sigma_i}(x)}= -\Sigma_i^{-1}(x-\mu_i)$ is $(1+\nfrac{\gamma^2}{\beta^2})$-Lipschitz (see \cref{eq:crypto:assumption}).
    
    \paragraph{Bound on the second moment.}
        Observe that any $-k\leq i\leq k$, we have the following bound:
        \begin{align*}
            \int \|\cdot\|^2\,\d\normal{\mu_i}{\Sigma_i} 
            &= \|\mu_i\|^2 + \Tr(\Sigma_i) 
            \le 
                \frac{(\gamma\cdot i)^2}{(\beta^2+\gamma^2)^2} 
                + (d-1) + \frac{\beta^2}{\beta^2+\gamma^2}
            \leq \frac{(\gamma\cdot i)^2}{(\beta^2+\gamma^2)^2}+d \\
            &~~\Stackrel{\abs{i}\leq k}{\leq}~~ \frac{k^2}{\beta^2+\gamma^2} + d\,.
        \end{align*}
        Since $\gamma\geq 1$ and $\beta\in (0,1)$, this further simplifies to 
        \[
            \int \|\cdot\|^2\,\d\normal{\mu_i}{\Sigma_i} \leq \frac{k^2}{\gamma^2} + d\,.
        \]
        Since the truncated \hCLWE{} distribution is a convex combination of $\inbrace{\normal{\mu_i}{\Sigma_i}: \abs i \le k}$, we conclude that $M_2\leq d+ \frac{k^2}{\gamma^2}.$
\end{proof}

\subsubsection{LWE-hardness of score estimation for GMMs}\label{sec:LWE-hardness}

The work of \citet{bruna2021continuous} managed to show that $\hCLWE{}_{\beta,\gamma}$ is hard when $\gamma \geq 2\sqrt{d}$ and $\beta$ is a small constant (and the ratio $\gamma/\beta$ is polynomially bounded). \Cref{prop:LWE} implies that the truncated \hCLWE{} distribution (which is a GMM) with these parameters satisfies our assumptions with $\poly(d)$-bounded parameters.
Hence, our \Cref{thm:score-crypto-GMM} then implies that score estimation for Gaussian mixtures with {$\wt{\Omega}(\sqrt{d})$} components and {error smaller than $\nfrac{1}{\sqrt{d\log d}}$} would imply the existence of an efficient quantum algorithm that approximates worst-case lattice problems within polynomial factors, which is believed to be hard.

The follow-up work of \citet{gupte2022continuous} showed a direct reduction from \LWE{} to \CLWE{} (and hence \hCLWE{}), allowing us to extend the above hardness result to smaller values of $\gamma(d)$ (and hence fewer Gaussian components)\footnote{Moreover, the result of \citet{gupte2022continuous} provides hardness of \CLWE{} under the \emph{classical} (instead of quantum)
worst-case hardness of \GapSVP{}.}. \citet{gupte2022continuous} show that assuming the polynomial-hardness assumption on \LWE{}, we get the hardness of \CLWE{} for $\gamma(d) = d^{\epsilon}$ for any arbitrary small $\epsilon > 0$~\citep[see Corollary 1.2 in][] {gupte2022continuous} and $\beta = o(1)$. This can be used to show the following:

\begin{theorem}[Cryptographic hardness of score estimation for general GMMs]\label{thm:cryptoHardnessGMM}
Let $\alpha > 1$ be an absolute constant. 
Let $\ell$ be the dimension of \LWE{}, $d$ be the dimension of the GMMs, and let $d = \ell^\alpha$. Fix some modulus $q = \ell^2$ and error parameter $\sigma = \sqrt{\ell}$.
Let also $m$ be such that $\ell \leq m \leq \poly(\ell).$
Assuming that the $\LWE{}_{q,\sigma}$ problem (see \Cref{def:lwe}) is hard to distinguish for $\poly(\ell)$-time algorithms with advantage $\Omega(1/m^3)$ and $d$ samples, then there exists no algorithm implementing the $c$-score estimation oracle for $k$ Gaussians over $\R^d$ in $\poly(d)$ time with $m$ samples for $k = d^{1/(2\alpha)} \log d$ and accuracy $c \leq O(1/\sqrt{d \log d})$.
\end{theorem}
The constant in the exponent of $k = d^{1/(2\alpha)}\log d$, which controls the number of Gaussian components we want to exclude, leads to some specific regime for \LWE{}.
For the exact details on the polynomial factors in the above statement, we refer to~\citet[Corollary 9]{gupte2022continuous}. 

The key idea of \citet{gupte2022continuous} is a reduction from \LWE{} to a \emph{sparse}\footnote{Sparsity with parameter $\Delta$ means that the secret vector of \LWE{} or \CLWE{} has exactly $\Delta$ non-zero entries.} version of \CLWE{} (which is enabled by a reduction from \LWE{} to \LWE{} with sparse secrets \cite{micciancio2018hardness,gupte2022continuous}). This sparsity parameter, which we denote by $\Delta$ below, will allow us to show hardness for $\CLWE{}_{\beta,\gamma}$, roughly speaking even for smaller values of $\gamma$ compared to the range $\gamma(d) \geq 2\sqrt{d}$, shown by \citet{bruna2021continuous} (essentially the sparsity level $\Delta$ will appear in the right-hand side of the inequality).
We will now sketch the ideas for showing \Cref{thm:cryptoHardnessGMM}, which can be proved using the reductions from \LWE{} to sparse-\CLWE{}~\citep[Lemma 20]{gupte2022continuous}, from \CLWE{} to \hCLWE{}  \cite{bruna2021continuous}, our reduction from \hCLWE{} to GMM PAC density estimation, and finally, our reduction from PAC density estimation to score estimation. 

Using the first three reductions, one can show the following result for PAC density estimation.
\begin{lemma}
[{Analogue of~\citet[Corollary 7]{gupte2022continuous}}]
Assume that the following conditions hold for the parameters $m,d,\ell,\sigma,q,\Delta$:
\begin{enumerate}
    \item $10 \sqrt{\ln m + \ln d} \leq \sigma$,
    \item $\omega(\sigma \sqrt{\Delta}) \leq q \leq \poly(\ell)$,
    \item $\Delta \log_2(d/\Delta) = (1+\Theta(1))\, \ell \log_2 q $,
    \item $q \leq m^2$,
    \item $m \leq \poly(d)$.
\end{enumerate}
Assuming that the $\LWE_{q,\sigma}$ problem in $\ell$ dimensions is hard to distinguish for $\inparen{T(\ell) + \poly(d)}$-time algorithms with advantage $\Omega(\nfrac{1}{m^3})$, there is no algorithm solving PAC density estimation for Gaussian mixtures in $d$ dimensions with $m$ samples for $k = O(\sqrt{\Delta \log m \log d})$ components.
\end{lemma}
Combining the above with our reduction from PAC density estimation to score estimation and setting 
$\Delta = 4\ell/(\alpha -1) = 4d^{1/\alpha}/(\alpha-1)$ for $\alpha > 1$, $\ell \leq m \leq \poly(d)$, $q = \ell^2$, $\sigma = \sqrt{\ell},$ and $T(\ell) = \poly(\ell)$
gives \Cref{thm:cryptoHardnessGMM} with the number of components $k = d^{1/(2\alpha)} \log d$.

\ifanon
\else 
\section*{Acknowledgments}
    We thank Harrison H.\ Zhou for numerous helpful conversations, and Tudor Manole and anonymous reviewers for useful feedback.
    Alkis Kalavasis was supported by the Institute for Foundations of Data Science at Yale.
\fi 
\newpage
\printbibliography 
\appendix
\newpage

\section{Background on denoising diffusion probabilistic modeling}\label{appendix:ddpm}

    In this section, we provide standard background on denoising diffusion probabilistic models (DDPMs); see~\citet*{chen2023sampling} for further details on sampling.
    
    DDPM employs two Markov chains.
    The first Markov chain iteratively adds noise to the data and the second Markov chain reverses this process, converting noise back to the original data.
    The first Markov chain is usually handcrafted, the most prevalent choice being the addition of standard Gaussian noise.
    The second Markov chain, \ie{}, the reverse process, is parameterized by learned neural networks.
    Below, we present the continuous time extension of DDPM with standard Gaussian noise, which corresponds to the Ornstein--Uhlenbeck (OU) process in continuous time.
    
    \paragraph{Forward process arising from OU.}
    The forward process arising from the OU process is the following stochastic differential equation (SDE):
    \begin{equation}
        \d X_t = -X_t\, \d t + \sqrt{2}\, \d B_t\,,\qquad X_0 \sim  P\,,\label{eq:Forward}
    \end{equation}
    where $(B_t)_{t \geq 0}$ is a standard Brownian motion in $\R^d$. The forward process transforms samples from the data distribution $P$ into
    {standard Gaussian noise} $\cN(0,\mathrm{Id})$. Namely, if we denote by $P_{t}$ the law of $X_t$, then $\lim_{t\to \infty}  P_{t} = \cN(0, \mathrm{Id})$. In fact, the convergence is exponentially fast in many metrics and divergences \citep{BGL14}.

    \paragraph{Reversing the OU process.}
    The ultimate goal of generative modeling is to generate samples from $P.$ To this end, we must reverse the process~\eqref{eq:Forward} in time, which yields the reverse process that transforms pure noise back to samples from the target distribution.
    
    Fix a terminal time $T > 0.$ Denote
    \[
    X_t^\leftarrow = X_{T-t}\,, \qquad t\in [0,T]\,.
    \]
    It turns out that the time reversal of~\eqref{eq:Forward} is
    \begin{equation}\label{eq:Reverse}
        \d X_t^\leftarrow
        =
        \{X_t^\leftarrow + 2\,\nabla \log  P_{T-t}(X_t^\leftarrow)\}\, \d t 
        + 
        \sqrt{2}\, \d B_t\,, 
        \qquad 
        X_0^\leftarrow \sim P_{T}\,,
    \end{equation}
    where now $(B_t)_{t \geq 0}$
is the reversed Brownian motion.

   \subsubsection*{Denoising score matching with DDPM}

   To implement the reverse process, one needs to learn the unknown score functions $\nabla \log P_{T-t}$ for $t \in [0, T]$.
   This is where denoising score matching (DSM) is utilized in the DSM--DDPM estimator \cite{jonathan2020ddpm,song2020denoising}.
   We start with the objective
\begin{align}\label{eq:ddpm_original}
       \argmin_{s = {\{s_t\}}_{t\in [0,T]} \in \cF} \int_0^T \|s_t - \nabla \log P_{t}\|_{L^2(P_{t})}^2\,\d t\,.
   \end{align}
   This objective is not amenable to empirical risk minimization since it involves the unknown score function $\nabla \log P_{t}$.
   Instead, we note that for fixed $t \in (0,T]$,
   \begin{align*}
       &\int \|s_t - \nabla \log P_{t}\|^2\,\d P_{t} \\
       &\qquad = \int \|s_t\|^2\,\d P_{t} - 2\int \langle s_t, \nabla \log P_{t}\rangle\,\d P_{t} + \textsc{constant}\,.
   \end{align*}
   In this derivation, $\textsc{constant}$ refers to any term that does not depend on the optimization variable $s_t$.
   Continuing, the second term above can be written
   \begin{align*}
       &- 2\int \langle s_t, \nabla \log P_{t}\rangle\,\d P_{t}
       = -2\int \langle s_t(x_t), \nabla P_{t}(x_t)\rangle\,\d x_t \\
       &\qquad = -2\int \Bigl\langle s_t(x_t), \nabla_{x_t} \int Q_{t|0}(x_t \mid x_0)\,P(\d x_0)\Bigr\rangle\,\d x_t \\
       &\qquad = -2\iint \langle s_t(x_t), \nabla Q_{t|0}(x_t \mid x_0)\rangle\,P(\d x_0)\,\d x_t \\
       &\qquad = -2\iint \langle s_t(x_t), \nabla \log Q_{t|0}(x_t \mid x_0)\rangle\, Q_{t|0}(\d x_t \mid x_0)\,P(\d x_0)\,.
   \end{align*}
   We deduce that the original problem~\eqref{eq:ddpm_original} is equivalent to
   \begin{align*}
       \argmin_{s \in \cF} \int \ell^{\rm DDPM}(s; x_0)\,P(\d x_0)\,,
   \end{align*}
   where
   \begin{align*}
       \ell^{\rm DDPM}(s; x_0)
       &\deq \int_0^T\int \bigl\{\|s_t\|^2 - 2\,\langle s_t, \nabla \log Q_{t|0}(\cdot \mid x_0)\rangle\bigr\}\,\d Q_{t|0}(\cdot \mid x_0)\,\d t\,.
   \end{align*}
   This is now amenable to empirical risk minimization, and hence we define the empirical version
   \begin{align*}
       \hat{\cR}_n^{\rm DDPM}(s)
       &\deq \frac{1}{n} \sum_{i=1}^n \ell^{\rm DDPM}(s; x_0^{(i)})
   \end{align*}
   where $x_0^{(1)},\dotsc,x_0^{(n)}$ are samples.

   Finally, to see that this is equivalent to~\Cref{def:ddpmMain}, we take $\cF = \{ \{\nabla \log P_{\theta,t}\}_{t\in [0,T]} : \theta \in \Theta\}$ and we note that by well-known properties of the OU semigroup, if $z_t \sim \cN(0,\mathrm{Id})$ is independent of $x_0$ and $x_t \deq \exp(-t)\,x_0 + \sqrt{1-\exp(-2t)}\,z_t$, then $x_t \sim Q_{t|0}(\cdot \mid x_0)$.
   Then, we can write
   \begin{align*}
       \nabla \log Q_{t|0}(x_t \mid x_0)
       &= - \frac{x_t - \exp(-t)\,x_0}{1-\exp(-2t)}
       = -\frac{z_t}{\sqrt{1-\exp(-2t)}}\,.
   \end{align*}

\section{Background on LWE and Continuous LWE}

    \subsection{Learning with errors}\label{LWE}

    The Learning with Errors (\LWE{}) problem, introduced by \citet{regev2009lattices}, is a fundamental problem in lattice-based cryptography and theoretical computer science. The problem is inspired by the classical learning parity with noise (\textsf{LPN}) problem and is believed to be computationally hard, even for quantum computers.

    Let $d$ and $q$ be positive integers, and $\sigma > 0$ be the error parameter. 
    We denote the quotient ring of integers modulo $q$ as $\mathbb{Z}_q = \mathbb{Z}/q\mathbb{Z}$. 
    Central to LWE is the following distribution.

\begin{definition}[\LWE{} distribution] 
For dimension $d$, modulus $q \ge 2$, and secret vector $s \in \mathbb{Z}^d$, the \emph{\LWE{} distribution} $A_{s, \sigma}$ over $\mathbb{Z}_q^d \times \R/q\Z$ is sampled by independently choosing uniformly random
$a \in \mathbb{Z}_q^d$ and $e \sim \cN(0,\sigma^2)$, and outputting $(a, (\langle a, s \rangle + e) \bmod q)$.
\end{definition}
In words, \LWE{} asks to efficiently distinguish (from polynomially many samples) between the \LWE{} distribution and the uniform distribution over $\Z_q^d \times \R/q\Z$ with strict advantage over random guessing.

\begin{definition}[\LWE{}]\label{def:lwe}
Let $\epsilon > 0$ be the advantage\footnote{ For a distinguisher $A$ running on two distributions $P,Q$, we say
that $A$ has advantage $\epsilon$ if $|\Pr_{x \sim P, A}[A(x) = 1] - \Pr_{x \sim Q,A}[A(x) =1]| = \epsilon.$}.
For dimension $d$, number of samples $m$, modulus $q = q(d) \geq 2$, and error parameter $\sigma = \sigma(d) > 0$, the average-case decision problem $\LWE_{q,\sigma}$ is to distinguish with advantage $\epsilon$ the following two distributions over $\mathbb{Z}_q^d \times \R/q\Z$ using $m$ i.i.d. samples:
(1) the \LWE{} distribution $A_{s, \sigma}$ for some uniformly random $s \in \mathbb{Z}^d$ (which is fixed for all samples), or (2) the uniform distribution over $\Z_q^d \times \mathbb R/q\Z$.
\end{definition}
\LWE{} is conjectured to be computationally hard even for error $\sigma = \poly(d)$, modulus $q = \poly(d)$, and non-negligible advantage, \ie{}, $\epsilon$ should be at least $d^{-c}$ for some $c > 0$~\cite{regev2009lattices}. Worst-case hardness reductions show that breaking \LWE{} is at least as hard as solving certain lattice problems in the worst case \cite{regev2009lattices}. \LWE{} has since become a cornerstone of post-quantum cryptography, as its hardness remains robust even against quantum adversaries.

\subsection{Background on lattices and discrete Gaussians}\label{Lattice}

A \emph{lattice} is a discrete additive subgroup of $\mathbb{R}^d$. We also define the discrete Gaussian supported on a lattice $L$ as follows.

\begin{definition}[Discrete Gaussian] For lattice $L \subset \mathbb{R}^d$, vector $y \in \mathbb{R}^d$, and parameter $r > 0$, the \emph{discrete Gaussian distribution} $D_{y+L,r}$ on coset $y+L$ with width $r$ is defined to have support $y+L$ and probability mass function proportional to $\rho_r$, where $\rho_r(y) \deq \exp(-\pi \|y\|^2/r^2)$.
\end{definition}
For $y = 0$, we simply denote the discrete Gaussian distribution on lattice $L$ with width $r$ by $D_{L,r}$.

An important lattice parameter induced by discrete Gaussians is the \emph{smoothing parameter}, defined as follows.

\begin{definition}[Smoothing parameter]
For lattice $L$ and real $\eps > 0$, we define the \emph{smoothing parameter} $\eta_{\eps}(L)$ as
\begin{align*}
    \eta_{\eps}(L) = \inf \{s > 0 \mid \rho_{1/s}(L^* \setminus \{0\}) \leq \eps\}\,,
\end{align*}
where $L^*$ corresponds to the dual lattice of $L$, defined as $L^* \deq \{ y \in \mathbb{R}^d \mid \langle x, y \rangle \in \mathbb{Z} \text{ for all } x \in L\}
    \; .$
\end{definition}
Intuitively, this parameter is the width beyond which the discrete Gaussian distribution behaves like a continuous Gaussian \cite{regev2009lattices}.

\subsubsection{Discrete Gaussian sampling problem}

The hardness for \CLWE{} proved in \citet{bruna2021continuous} is based on the hardness of the following standard problem, which asks for sampling from a discrete Gaussian.

\begin{definition}[\DGS{}]
For a function $\varphi$ that maps lattices to non-negative reals, an instance of $\DGS{}_\varphi$ is given by a lattice $L$ and a parameter $r \geq \varphi(L)$. The goal is to output an independent sample whose distribution is within a negligible statistical distance of $D_{L,r}$.
\end{definition}
This problem becomes interesting when one relates $\varphi$ with the smoothness parameter of the lattice. 
It is well-known that the main computational problems on lattices, namely \GapSVP{} and \textsf{SIVP}, which are believed to be computationally hard (even with quantum computation) for polynomial approximation factor, can be reduced to \DGS{}~\citep[see][Section 3.3]{regev2009lattices} for some choice of $\varphi$. Roughly speaking, if $\eps < \nfrac{1}{10}$ and $\varphi(L) \geq \sqrt{2} \eta_\epsilon(L)$, there is a polynomial-time reduction from (a generalization of) the shortest independent vector problem to $\DGS{}_\varphi$.

\subsection{Continuous LWE}\label{CLWE}

Our cryptographic hardness result for score estimation requires some background on continuous \LWE{}.
The exposition follows the work of \citet{bruna2021continuous}. The Learning with Errors (\LWE{}) problem has served as a foundation for many lattice-based cryptographic schemes. \citet{bruna2021continuous} introduced continuous \LWE{} (\CLWE{}), the continuous analogue of \LWE{}. In what follows, we let
\begin{equation}
    \rho_\sigma (x) = \exp(-\pi \|x\|^2/\sigma^2)\,.
\end{equation}
Following \citet{bruna2021continuous}, we also let $\rho(x) = \exp(-\pi \|x\|^2).$

\subsubsection{CLWE distributions}

We first define the \CLWE{} distribution.

\begin{definition}[\CLWE{} distribution \cite{bruna2021continuous}]
For unit vector $w \in \mathbb{R}^d$ and parameters $\beta, \gamma > 0$, define the \emph{\CLWE{} distribution} $A_{{w}, \beta, \gamma}$ over $\mathbb{R}^{d+1}$ to have density at $(y,z)$ proportional to
\begin{align*}
    \rho(y) \cdot \sum_{i \in \mathbb{Z}} \rho_\beta(z+ i -\gamma\,\langle w,y \rangle)
    \; .
\end{align*}
\end{definition}
The next key distribution is the homogeneous \CLWE{} (\hCLWE{}) distribution, which, roughly speaking, can be obtained from \CLWE{} by essentially conditioning on
$z \approx 0$.

\begin{definition}[Homogeneous \CLWE{} distribution \cite{bruna2021continuous}]\label{def:hclwe}
For unit vector $w \in \mathbb{R}^d$ and parameters $\beta, \gamma > 0$, define the \emph{\hCLWE{} distribution} $H_{w, \beta, \gamma}$ over $\mathbb{R}^d$ to have density at $y$ proportional to
\begin{align}\label{eqn:hclwe-def}
    \rho(y) \cdot \sum_{i \in \mathbb{Z}} \rho_\beta(i-\gamma\,\langle w, y \rangle)
    \; .
\end{align}
\end{definition}
The \hCLWE{} distribution can be equivalently defined as a mixture of Gaussians.
To see this, notice that \Cref{eqn:hclwe-def} is equal to
\begin{align}\label{eqn:hclwe-mixture-def}
    \sum_{i \in \mathbb{Z}} \rho_{\sqrt{\beta^2+\gamma^2}}(i) \cdot 
    \rho(y^{\perp w}) \cdot \rho_{\beta/\sqrt{\beta^2+\gamma^2}}\Bigl(\langle w, y \rangle -\frac{\gamma}{\beta^2+\gamma^2}\, i \Bigr) \; ,
\end{align}
where $y^{\perp w}$ denotes the projection of $y$ on the orthogonal space to $w$.
Hence, $H_{w, \beta, \gamma}$ can be viewed as a mixture of (infinitely many) Gaussian components of width 
$\beta/\sqrt{\beta^2+\gamma^2}$ (which is roughly $\beta/\gamma$ for $\beta \ll \gamma$) in the secret direction, and width $1$ in the orthogonal space. The components are equally spaced, with a separation of $\gamma/(\beta^2+\gamma^2)$ between them (which is roughly $1/\gamma$ for $\beta \ll \gamma$).

\subsubsection{CLWE decision problems}\label{hCLWE-Decision}\label{CLWE-Decision}

Given the above distributions, we can introduce the decision problems which, under standard assumptions and for some regime of $(\beta,\gamma)$, will be the basis of our density estimation hardness results. We state the average-case version of these problems but it can be shown that the average case is as hard as the worst-case decision problem \cite{bruna2021continuous}.

Following \citet{bruna2021continuous}, we will denote by
\begin{itemize}
    \item $D_{\R^d}$ the Gaussian distribution in $d$ dimensions with 0 mean and covariance $\mathrm{Id}/(2\pi)$.
    \item $\mathbb T$ is the quotient group of reals modulo
the integers, \ie{}, $\mathbb T = \reals / \ints = [0,1)$.
    \item $U$ is the uniform distribution on $\mathbb T$.
\end{itemize}
Given the above, we are now ready to define the decision version of \CLWE{}.

\begin{definition}[Decision \CLWE{}~\cite{bruna2021continuous}]
For parameters $\beta, \gamma > 0$, the average-case decision problem $\CLWE_{\beta, \gamma}$ is to distinguish with probability $>1/2$ the following two distributions over $\mathbb{R}^d \times \mathbb{T}$: 
\begin{enumerate}
    \item[$(\mathsf H_0)$] $D_{\mathbb{R}^d} \times U$, or
    \item[$(\mathsf H_1)$] the \CLWE{} distribution $A_{w, \beta, \gamma}$ for some uniformly random unit vector $w \in \mathbb{R}^d$ (which is fixed for all samples).
\end{enumerate}
\end{definition}
Similarly, we define the \hCLWE{} problem as: 

\begin{definition}[Decision \hCLWE{} \cite{bruna2021continuous}]
For parameters $\beta, \gamma > 0$, the average-case decision problem $\hCLWE{}_{\beta, \gamma}$ is to distinguish with probability $>1/2$ the following two distributions over $\mathbb{R}^d$: 
\begin{enumerate}
    \item[$(\mathsf H_0)$] $D_{\mathbb{R}^d}$, or
    \item[$(\mathsf H_1)$] the \hCLWE{} distribution $H_{w, \beta, \gamma}$ for some uniformly random unit vector $w \in \mathbb{R}^d$ (which is fixed for all samples).
\end{enumerate}
\end{definition}

\subsubsection{Hardness of CLWE}\label{CLWE-hard}

Now, we present the main hardness results of \citet{bruna2021continuous} for \CLWE{} and \hCLWE{}. The hardness results apply in the regime where $\gamma \geq 2\sqrt{d}$ and $\beta(d)$ is a small constant (but the ratio $\gamma/\beta$ is polynomially bounded). This allows \citet{bruna2021continuous} to exclude $\poly(d,k)$-time algorithms for density estimation of $d$-dimensional Gaussian mixtures with $k \gtrsim \sqrt{d \log d}$ components, assuming the hardness of standard lattice-based problems.

\begin{theorem}[Hardness of \CLWE{} \cite{bruna2021continuous}]
Let $\beta = \beta(d) \in (0,1)$ and $\gamma = \gamma(d) \geq 2\sqrt{d}$ such that $\gamma/\beta$ is polynomially bounded. Then, there is a polynomial-time quantum reduction from $\DGS{}_{2\sqrt{d}\eta_\epsilon(L)/\beta}$ to $\CLWE{}_{\beta,\gamma}$.
\end{theorem}
\citet{bruna2021continuous} also show the hardness of \hCLWE{} by reducing from \CLWE{}. The idea of the reduction is to transform \CLWE{} samples to \hCLWE{} samples using rejection sampling.

\begin{theorem}[Hardness of \hCLWE{} \cite{bruna2021continuous}]
\label{sec:hardhCLWE}
For any $\beta = \beta(d) \in (0,1)$ and $\gamma = \gamma(d) \geq 2\sqrt{d}$ such that $\gamma/\beta$ is polynomially bounded,
there is a polynomial-time quantum reduction from $\DGS{}_{2\sqrt{2d}\eta_\eps(L)/\beta}$ to $\hCLWE{}_{\beta,\gamma}$.
\end{theorem}
Finally, using standard reductions from \GapSVP{} to \DGS{}, one can obtain reductions from \GapSVP{} to \CLWE{} (and hence \hCLWE{}).

Follow-up work by \citet{gupte2022continuous} has shown improved hardness results for \CLWE{} by giving a direct reduction from \LWE. For instance, for $\gamma = \wt{O}(\sqrt{d})$ and $\beta = O(\sigma \sqrt{d}/q)$, there is a polynomial in $d$ time reduction from \LWE{} in dimension $\ell$, with $d$ samples, modulus $q$ and error parameter $\sigma$ to $\CLWE_{\beta, \gamma}$ in dimension $d$, as long as $d \gg \ell \log_2 q$ and $\sigma \gg 1.$

\section{Auxiliary lemmas}

\subsection{Standard facts about sub-Gaussianity}\label{sec:facts:subGaussian}
    
    In this section, we define sub-Gaussianity and state standard facts about sub-Gaussian random variables that are used in our proofs.
    
    \begin{definition}[Sub-Gaussian random variable]
        Given a constant $\sigma > 0$, a random variable \(X\in \R\) is said to be $\sigma$-sub-Gaussian if for all $\lambda \in\R$,
        \[
            \E e^{\lambda \,(X-\E X)} \le e^{\lambda^2 \sigma^2/2}\,.
        \]
    \end{definition}
    We note that there are several definitions of sub-Gaussian random variables which are all equivalent up to constant factors.
    In particular, this equivalence implies the following fact.
    
    \begin{fact}\label{fact:momentSubGaussian}
        Given a constant $\sigma > 0$ and a \emph{$\sigma$-sub-Gaussian} random variable \(X\in \R\), for any $p\geq 1$,
        \[
            \E[|X|^p]^{1/p}\lesssim \sigma\sqrt{p}\,.
        \]
    \end{fact}
    Next, we define a sub-Gaussian random vector.
    
    \begin{definition}[Sub-Gaussian random vector]
        For constant $\sigma > 0$, a random vector \(X \in \R^d\) is said to be $\sigma$-sub-Gaussian if, for every unit vector \( v \in \mathbb{R}^d\), $\langle v, X \rangle$ is $\sigma$-sub-Gaussian.
    \end{definition}
    Hence, in particular, if $X\in \R^d$ is $\sigma$-sub-Gaussian, then each coordinate of $X$ is $\sigma$-sub-Gaussian.
    Further, the Euclidean norm $\norm{X}$ is also sub-Gaussian.
    
    \begin{fact}[Sub-Gaussianity of the Euclidean norm]\label{fact:normSubGaussian}
        For any $\sigma >0$ and a random vector $X\in \R^d$ which is $\sigma$-sub-Gaussian (with possibly dependent coordinates), the Euclidean norm of $X$, \ie{},  $\norm{X},$ is $O(\sigma\sqrt{d})$-sub-Gaussian.
    \end{fact}

\subsection{Integral estimates}

\begin{lemma}[Integral estimates]\label{lem:integrals}
    Let $L \ge 1$ and $0 < \tau \le 1 \le T$.
    For each $t > 0$, let $L_t$ be as in~\Cref{lem:score_subG}.
    Then, the following estimates hold.
    \begin{enumerate}
        \item $\int_\tau^T L_t \,\d t \lesssim T + \log \nfrac{1}{(\tau+1/L)}$.
        \item $\int_\tau^T L_t^2 \,\d t \lesssim T + \nfrac{1}{(\tau+1/L)}$.
        \item $\int_\tau^T \nfrac{L_t}{(1-e^{-2t})}\,\d t \lesssim T + \log_+(\nfrac{1}{L\tau})/(\tau+\nfrac{1}{L})$.
    \end{enumerate}
\end{lemma}
\begin{proof}
    Let $t_* \deq \log(1+1/L)/2$ and note that $L_t \le 1/(1-e^{-2t})$ for $t\ge t_*$.
    For the first integral,
    \begin{align*}
        \int_\tau^T L_t\,\d t
        &\le 2L\,(t_* - \tau)_+ + \int_{t_* \vee \tau}^T \frac{1}{1-e^{-2t}}\,\d t
        = 2L\,(t_* - \tau)_+ + \frac{1}{2} \log \frac{e^{2T}-1}{e^{2\,(t_* \vee \tau)} - 1} \\
        &\lesssim T + \log \frac{1}{\tau+1/L}\,.
    \end{align*}
    For the second integral,
    \begin{align*}
        \int_\tau^T L_t^2\,\d t
        &\le 4L^2\,(t_* - \tau)_+ + \int_{t_*\vee \tau}^T \frac{1}{{(1-e^{-2t})}^2}\,\d t
        \lesssim T + \frac{1}{\tau + 1/L}\,.
    \end{align*}
    For the third integral,
    \begin{align*}
        \int_\tau^T \frac{L_t}{1-e^{-2t}}\,\d t
        &\lesssim L \log_+ \frac{t_*}{\tau} + \int_{t_* \vee \tau}^T \frac{1}{{(1-e^{-2t})}^2}\,\d t
        \lesssim T+ \frac{\log_+\nfrac{1}{L\tau}}{\tau + 1/L}\,. \qedhere
    \end{align*}
\end{proof}

\subsection{Facts about the H\"older class}

The following fact is classical, but we prove it for completeness.

\begin{lemma}[Lipschitz estimate for the H\"older class]\label{lem:holder_lip}
    Let $P \in \hyH_s(C,L)$ with $s > 1$.
    Then, $P$ is Lipschitz continuous on $[-1, 1]$ with a constant that only depends on $s$, $C$, and $L$.
\end{lemma}
\begin{proof}
    We prove via induction that for all $k \le \lfloor s \rfloor$ and all intervals $[a,a+\ell] \subseteq [-1,1]$, there exists $x \in [a,a+\ell]$ with $|D^k P(x)| \le C'$, where $C'$ denotes a constant which only depends on $s$, $C$, $L$, and $\ell$ and can change from line to line.
    The base case $k=0$ follows from the definition of $C$ in~\Cref{defn:holder}.
    Next, assuming that the statement holds for an integer $k < \lfloor s\rfloor$, we split $[a,a+\ell]$ into three sub-intervals and note that there must exist $x \in [a,a+\ell/3]$, $y \in [a+2\ell/3, a+\ell]$ such that $|D^k P(x)| \vee |D^k P(y)| \le C'$.
    Now, if $|D^{k+1} P|$ is always large on $[a,a+\ell]$, say $D^{k+1} P \ge C''$, this leads to a contradiction for $C'' > 6C'/\ell$, and thus the inductive step holds.

    Now we perform backward induction on $k$ and argue that in fact $|D^k P| \le C'$ on all of $[-1,1]$.
    When $s$ is an integer and $k = s$, this is true by definition.
    Otherwise, for $k=\lfloor s\rfloor$, we know that there exists $x \in [-1,1]$ with $|D^k P(x)| \le C'$, and that $D^k$ is $(s-k)$-H\"older continuous, so $|D^k P| \le C'$ on $[-1,1]$.
    This verifies the base case.
    Similarly, assuming that the statement holds for an integer $k > 1$, we know that there exists $x \in [-1,1]$ such that $|D^{k-1} P(x)| \le C'$ and that $D^{k-1} P$ is $C'$-Lipschitz by the inductive hypothesis; hence, $|D^{k-1} P| \le C'$ on $[-1,1]$.
\end{proof}

\section{Further related work}\label{appendix:relatedwork}

    In this section, we present further related work.

    \paragraph{Statistical guarantees for diffusion models for learning.} Recent works have established rigorous statistical guarantees for diffusion models and related score matching estimators. For example, \citet{oko2023minimalDiffusion} bounded the estimation error of the empirical risk minimizer over a neural network class and demonstrated that diffusion models are nearly minimax-optimal generators in both the total variation and the Wasserstein distance of order one, provided that the target density belongs to the Besov space. 
    \citet{cui2024analysis} employed a two-layer neural network to learn score functions and, in the special case where the target distribution is a mixture of two Gaussians, they established an error guarantee of $\Theta(\nfrac{1}{n})$ for the estimated mean. 
    Further, \citet{wibisono2024optimalscoreestimationempirical} considered sub-Gaussian densities with Lipschitz-continuous score functions and provided optimal rates for estimating the scores in the $L_2$-norm. 
    In a related direction,~\citet{KoeLeeVuo25DataInit} showed that pseudo-likelihood methods can be used to learn low-rank Ising models, which is an example of using score matching for designing learning methods with provable statistical guarantees.
    Also, \citet{mei2023deep} studied the statistical efficiency of neural networks for approximating score functions, focusing on the setting of graphical models and
variational inference algorithms.
    Moreover, \citet{dou2024optimalscorematchingoptimal} studied the score matching (SM) estimator in detail and established the sharp minimax rate of score estimation for smooth, compactly supported densities using sophisticated techniques.

    \paragraph{Computational properties of {the DDPM estimator}.} 
        Beyond the immense practical success of DDPM estimators and the growing interest from statisticians, surprisingly, DDPM estimators are also leading to new \textit{provably} efficient algorithms for sampling and distribution learning; see \eg{}, the works of \citet{shah2023learning,chen2024learninggeneralgaussianmixtures,gatmiry2024learning}. 
    
    \paragraph{Sampling guarantees for diffusion models.}
    Finally, a {rapidly} growing body of work establishes sampling guarantees for SGMs under the assumption that the score functions are accurately estimated.
    Since the literature is vast and orthogonal to the statistical concerns in this paper, we do not survey it here; see, however,~\Cref{ssec:related}.

    \subsection{Importance of density estimation}

        Density estimation is a foundational problem in statistics \cite{pearson1894contributions} and computer science \cite{kearns1994learnability}, {and has been extensively studied for a number of parametric and non-parametric families}. 
        Gaussian mixture models (GMMs) are one of the most well-studied {parametric distribution families for density estimation in particular and, also} in statistics {more broadly}, with a history going back to the work of Pearson \cite{pearson1894contributions} {(also see the survey by \citet{titterington1985statistical} for applications for GMMs in the sciences).}
        The study of {statistically and computationally} efficient algorithms for {estimating} GMMs goes back to the seminal {works} of \citet{redner1984mixtures,dasgupta1999learning,lindsay1995mixtures} and has {since} attracted significant interest from theoretical computer scientists \cite{vempala2002spectral,kannan2005spectral,brubaker2008isotropic,feldman2006pac,kalai2010efficiently,moitra2010settling,hopkins2018mixture,bakshi2020outlier,diakonikolas2020robustly,liu2022clustering,liu2022learning,buhai2023beyond,anderson2024dimension,li2017robust,diakonikolas2020small,bakshi2022robustly}.

    \subsection{History of evaluators and generators for GMMs}\label{sec:GMM-history}

    A special case of our Gaussian location mixtures is the fundamental problem of learning spherical Gaussian mixtures.
    We are interested in learning Gaussian mixture models from samples {without any separability assumptions on the components. This makes parameter recovery impossible and the goal is to} output a hypothesis {(\ie{}, the representation of a distribution)} that is close to the target GMM in total variation distance \cite{feldman2006pac,kalai2010efficiently,moitra2010settling,suresh2014near,daskalakis2014faster,diakonikolas2019robust,li2017robust,ashtiani2018nearly,diakonikolas2020small,bakshi2022robustly, chen2024learninggeneralgaussianmixtures, gatmiry2024learning, KoeLeeVuo25DataInit}.
    {In this setting, the} sample complexity of estimating a mixture of $k$ $d$-dimensional Gaussians {was} {settled by} the work of \citet*{ashtiani2018nearly} {which} showed that $\wt{\Theta}(\nfrac{k d^2}{\epsilon^2})$ samples are both necessary and sufficient. 
    {This} result is information-theoretic and the algorithm {performs a brute-force search over the parameters resulting in an} exponential dependence on $k$ and $d$ {in the running time}.
    
    In contrast to sample complexity, the computational complexity of the problem is still far from well-understood and there is strong evidence that the problem exhibits a statistical-to-computational  trade-off. While the sample complexity of the problem is $\poly(d,k,\nfrac{1}{\eps})$, all known algorithms for the density estimation problem require time $d^{g(k)}(\nfrac{1}{\eps})^{f(k)}$ for $\max\{g(k),f(k)\} \geq k$ {and, in fact, without further assumptions $f(k)\geq \Omega(k^{k^2})$.}
        
    \paragraph{Algorithms for GMM density estimation.}
    In the special case of spherical Gaussians, the best-known density estimation algorithm (in the sense of an evaluator) is quasi-polynomial in the number of components $k$: {this algorithm is by} \citet{diakonikolas2020small} {and its} sample and time complexity scale as $\poly(d)\, (\nfrac{k}{\eps})^{\log^2 k}$.
The runtime of this result is an improvement on the work of
\citet{suresh2014near} who gave an algorithm for learning mixtures of spherical Gaussians with $\poly(d,k,\nfrac{1}{\eps})$ samples {and} $\poly(d)\,(\nfrac{k}{\eps})^k$ runtime. 

When the covariance matrix is arbitrary and the number of components $k$ is constant, the breakthrough results of \citet{kalai2010efficiently}, \citet{moitra2010settling}, and \citet{belkin2010polynomial} provided a polynomial-time density estimation algorithm for Gaussian mixtures. 
The case where $k$ is not a constant, is studied by an important line of works \cite{feldman2006pac,moitra2010settling,kalai2010efficiently,belkin2010polynomial,bakshi2022robustly} culminating in the state-of-the-art $(\nfrac{d}{\eps})^k\, (\nfrac{1}{\eps})^{k^{k^2}}$-time algorithm by \citet*{bakshi2022robustly}.
{This} runtime scales polynomially in $d$ and $\eps$ for constant $k$ but otherwise has a doubly exponential {dependence on} the number of components $k$. 
Their algorithm is a result of a series of works in the coalescence of a {large toolkit of algorithms based on method of moments and sum-of-squares for
(1) learning mixtures models and (2) robust statistics} (see \eg{}, the book by \citet{DiakonikolasKane2023} and the works of \citet{kalai2010efficiently,moitra2010settling,hopkins2018mixture,bakshi2020outlier,diakonikolas2020robustly,liu2022clustering,liu2022learning,buhai2023beyond,anderson2024dimension}) and, as a result, is robust to a small number of outliers.

\paragraph{Lower bounds for GMM density estimation.}
Regarding lower bounds,
\citet{diakonikolas2017statistical} showed that density estimation of Gaussian mixtures incurs a super-polynomial lower bound in the restricted statistical query (SQ) model \cite{kearns1998efficient,feldman2017statistical}. In particular, they show that any SQ algorithm giving density estimates requires ${d}^{\Omega(k)}$ queries to an SQ oracle of precision ${d}^{-O(k)}$; this leads to a super-polynomial runtime as long as $k$ is {not a constant}. The work of \citet{bruna2021continuous} proved the computational hardness of density estimation for Gaussian mixtures based on lattice problems. They showed that a $\poly(d,k)$-time 
density estimation algorithm for mixtures of $k \gtrsim \sqrt{d \log d}$ Gaussians implies a quantum polynomial-time algorithm
for worst-case lattice problems. Follow-up work by \citet{gupte2022continuous} extended this lower bound for values of $k$ between $\Theta_d(1)$~\citep[for which there exist polynomial time algorithms, see][]{moitra2010settling} and $k ~{\asymp}~ \sqrt{d}.$  
They showed that, assuming the polynomial-time hardness of \LWE{}, there is no {$\poly(d)$} density estimation algorithm for $k = d^{{\gamma}}$ components in $d$ dimensions for any ${{\gamma}} > 0.$ This strengthens the bound of \citet{bruna2021continuous} based on the standard quantum reduction from worst-case lattice problems to \LWE{}. Moreover, under the stronger yet standard assumption of sub-exponential hardness of \LWE{}, they showed that density estimation of {mixtures of} $k = (\log d)^{1 + {\gamma}}$ Gaussians cannot be done in polynomial time given a polynomial number of samples
(where ${\gamma} > 0$ is an arbitrarily small constant). Finally, they showed the computational hardness of density estimation for {mixtures of}
$k = (\log d)^{1/2 + {\gamma}}$ Gaussians \emph{given} $\poly(\log d)$ samples (where ${\gamma} > 0$ is an arbitrary constant).\footnote{The hard instance for this lower bound is learnable with sample complexity poly-logarithmic in $d$. This is also the case for the hard SQ instance of \citet{diakonikolas2017statistical}.}

\paragraph{Generators for GMMs.}
{All of the aforementioned works studied the complexity of density estimation for GMMs, \ie{}, computing an \emph{evaluator} \cite{kearns1994learnability} that induces a density which is $\eps$-close in TV-distance to the target GMM (although, possibly very far in parameter distance).
}
{For the problem of generation (\ie{}, learning a \emph{generator} \cite{kearns1994learnability}), the recent} work by \citet{chen2024learninggeneralgaussianmixtures} gave an algorithm for Gaussian mixtures with general covariance matrices under mild condition number bounds. 
The proposed algorithm, which departs from the moment-based and sum-of-squares-based blueprints of prior works and is based on diffusion models, runs in time $d^{\poly(k/\epsilon)}$, and returns a \emph{generator} for the target GMM, {significantly advancing} the state-of-the-art for the task {of} learning to {generate} from general GMMs. {Concurrently and independently}, the work of \citet{gatmiry2024learning} provided a generator for the special case of {mixtures with} identity covariance {matrices} that importantly extends to the more general family of Gaussian location mixtures (we make use of their score estimation result in order to get our PAC density estimator for GLMs); this family and the results of \citet{gatmiry2024learning} are further discussed in \cref{sec:GLM}. Finally, recent progress on high-order tensor estimation by \citet{DiaKan25HighOrder} gives a $\poly(d, k, \nfrac{1}{\epsilon})$ time generator for spherical GMMs under that assumption that the means are bounded by $O(\sqrt{\log k})$.

\end{document}